\documentclass{article}

\PassOptionsToPackage{numbers, compress}{natbib}

\usepackage[preprint]{configuration/neurips_2026}

\usepackage{hyperref}
\usepackage{url}
\usepackage{enumitem}
\usepackage{tcolorbox}

\usepackage{amsmath,amssymb,amsthm}

\usepackage{amsmath,amsfonts,bm}

\def\secref#1{section~\ref{#1}}
\def\Secref#1{Section~\ref{#1}}

\def\eqref#1{equation~\ref{#1}}

\def\algref#1{algorithm~\ref{#1}}

\def\1{\bm{1}}

\DeclareMathAlphabet{\mathsfit}{\encodingdefault}{\sfdefault}{m}{sl}
\SetMathAlphabet{\mathsfit}{bold}{\encodingdefault}{\sfdefault}{bx}{n}

\newcommand{\E}{\mathbb{E}}

\providecommand{\xx}{\mathbf{x}}

\newenvironment{talign*}
{\csname align*\endcsname}
{\endalign}

\usepackage[utf8]{inputenc}         %
\usepackage[T1]{fontenc}            %
\usepackage{url}                    %
\usepackage{booktabs}               %
\usepackage{amsfonts}               %
\usepackage{nicefrac}               %
\usepackage{microtype}              %
\usepackage{xcolor}                 %
\usepackage{algorithm}
\usepackage{algpseudocode}
\usepackage{graphicx}
\usepackage{subcaption}
\usepackage[flushleft]{threeparttable}
\usepackage{float}
\usepackage{multirow}
\usepackage{makecell}
\usepackage{xspace}
\usepackage{enumitem}
\usepackage[font=small]{caption}
\usepackage{autobreak}
\usepackage{sidecap}
\usepackage{wrapfig}
\usepackage{bbding}
\usepackage[toc, page, header]{appendix}
\usepackage{tikz}
\usetikzlibrary{calc,trees,positioning,arrows,chains,shapes.geometric,%
    decorations.pathreplacing,decorations.pathmorphing,shapes,%
    matrix,shapes.symbols}
\usepackage{xcolor}
\usepackage{pifont}
\usepackage{mdframed}
\usepackage{colortbl}

\usepackage{tcolorbox}
\tcbuselibrary{skins, breakable, theorems}
\usepackage{empheq}
\usepackage{arydshln}
\usepackage{bm}
\usepackage[capitalize]{cleveref}
\usepackage{listings}

\hypersetup{
    colorlinks=true,
    linkcolor=blue,
    citecolor=blue,
    urlcolor=blue
}

\definecolor{coral}{RGB}{255,127,80}
\definecolor{darkgreen}{RGB}{0,100,0}
\definecolor{darkyellow}{RGB}{204,153,0}
\definecolor{salmon}{RGB}{250,128,114}
\definecolor{darkred}{RGB}{150,0,0}
\newcommand{\darkredtext}[1]{{\color{darkred}#1}}
\newcommand{\posc}[1]{{\color{darkgreen}#1}}
\newcommand{\negc}[1]{{\color{darkred}#1}}

\definecolor{eqbg}{gray}{0.95}
\definecolor{indomaincolor}{rgb}{0.9, 0.95, 1.0}
\definecolor{oodcolor}{rgb}{1.0, 0.9, 0.9}

\newcommand{\transparentgreen}[1]{%
    \tikz[baseline=(X.base)] \node[fill=green, fill opacity=0.1, text opacity=1, inner sep=2pt, outer sep=0pt] (X) {#1};%
}
\newcommand{\transparentyellow}[1]{%
    \tikz[baseline=(X.base)] \node[fill=yellow, fill opacity=0.1, text opacity=1, inner sep=2pt, outer sep=0pt] (X) {#1};%
}

\renewcommand{\secref}[1]{\hyperref[#1]{\darkredtext{Sec.~\ref*{#1}}}}
\renewcommand{\Secref}[1]{\hyperref[#1]{\darkredtext{Sec.~\ref*{#1}}}}
\providecommand{\thmref}[1]{\hyperref[#1]{\darkredtext{Thm.~\ref*{#1}}}}
\providecommand{\defref}[1]{\hyperref[#1]{\darkredtext{Def.~\ref*{#1}}}}
\providecommand{\propref}[1]{\hyperref[#1]{\darkredtext{Prop.~\ref*{#1}}}}
\providecommand{\assumpref}[1]{\hyperref[#1]{\transparentgreen{Assumption~\ref*{#1}}}}
\providecommand{\remarkref}[1]{\hyperref[#1]{\transparentyellow{Remark~\ref*{#1}}}}
\providecommand{\conjref}[1]{\hyperref[#1]{\darkredtext{Conj.~\ref*{#1}}}}
\providecommand{\lemref}[1]{\hyperref[#1]{\darkredtext{Lem.~\ref*{#1}}}}
\providecommand{\corref}[1]{\hyperref[#1]{\darkredtext{Cor.~\ref*{#1}}}}
\providecommand{\noteref}[1]{\hyperref[#1]{\darkredtext{Nota.~\ref*{#1}}}}
\providecommand{\claimref}[1]{\hyperref[#1]{\darkredtext{Clm.~\ref*{#1}}}}
\providecommand{\algref}[1]{\hyperref[#1]{\darkredtext{Alg.~\ref*{#1}}}}
\providecommand{\algmref}[1]{\hyperref[#1]{\darkredtext{Alg.~\ref*{#1}}}}
\providecommand{\figref}[1]{\hyperref[#1]{\darkredtext{Fig.~\ref*{#1}}}}
\providecommand{\tabref}[1]{\hyperref[#1]{\darkredtext{Tab.~\ref*{#1}}}}
\providecommand{\appref}[1]{\hyperref[#1]{\darkredtext{App.~\ref*{#1}}}}

\newtheoremstyle{professional}
{10pt} 
{10pt} 
{\itshape} 
{} 
{\bfseries} 
{.} 
{.5em} 
{} 

\theoremstyle{professional}
\newtheorem{myth}{Theorem}[section]
\newtheorem{myprop}[myth]{Proposition}
\newtheorem{mylem}[myth]{Lemma}
\newtheorem{mycor}[myth]{Corollary}
\newtheorem{mydef}[myth]{Definition}
\newtheorem{myassump}[myth]{Assumption}
\newtheorem{myrem}[myth]{Remark}
\newtheorem{myhyp}[myth]{Hypothesis}
\newtheorem{myconj}[myth]{Conjecture}
\newtheorem{mynota}[myth]{Notation}
\newtheorem{myclaim}[myth]{Claim}
\newtheorem{myprob}[myth]{Problem}
\newtheorem{myobs}[myth]{Observation}

\tcbset{
    thmbox/.style={
            enhanced,
            breakable,
            sharp corners,
            boxrule=0pt,
            leftrule=3pt,
            top=0pt,
            bottom=0pt,
            left=5pt,
            right=5pt,
            before skip=10pt,
            after skip=10pt,
        }
}

\newenvironment{theorem}{\begin{tcolorbox}[thmbox, colback=red!5!white, colframe=red!75!black]\begin{myth}}{\end{myth}\end{tcolorbox}}
\newenvironment{proposition}{\begin{tcolorbox}[thmbox, colback=blue!5!white, colframe=blue!75!black]\begin{myprop}}{\end{myprop}\end{tcolorbox}}
\newenvironment{lemma}{\begin{tcolorbox}[thmbox, colback=cyan!5!white, colframe=cyan!75!black]\begin{mylem}}{\end{mylem}\end{tcolorbox}}

\newenvironment{definition}{\begin{tcolorbox}[thmbox, colback=gray!10!white, colframe=gray!75!black]\begin{mydef}}{\end{mydef}\end{tcolorbox}}

\newenvironment{remark}{\begin{tcolorbox}[thmbox, colback=yellow!5!white, colframe=yellow!75!black]\begin{myrem}}{\end{myrem}\end{tcolorbox}}

\newtheorem{innernote}{Note}

\newtheorem{innerexercise}{Exercise}

\definecolor{codegreen}{rgb}{0,0.6,0}
\definecolor{codegray}{rgb}{0.5,0.5,0.5}
\definecolor{codepurple}{rgb}{0.58,0,0.82}
\definecolor{backcolour}{rgb}{0.95,0.95,0.92}

\lstdefinestyle{mystyle}{
    backgroundcolor=\color{backcolour},
    commentstyle=\color{codegreen},
    keywordstyle=\color{magenta},
    numberstyle=\tiny\color{codegray},
    stringstyle=\color{codepurple},
    basicstyle=\ttfamily\footnotesize,
    breakatwhitespace=false,
    breaklines=true,
    captionpos=b,
    keepspaces=true,
    numbers=left,
    numbersep=5pt,
    showspaces=false,
    showstringspaces=false,
    showtabs=false,
    tabsize=2
}
\usepackage[textwidth=3.5cm]{todonotes}

\newcommand{\methodname}{\textsc{ThermoDPO}}
\newcommand{\method}{\methodname\xspace}
\newcommand{\methodweightedname}{\textsc{ThermoDPO}-weighted\xspace}
\newcommand{\methodweighted}{\methodweightedname}

\title{Manifold Drift in Flow Preference Optimization: \\
A Root Cause of Reward Hacking}

\author{
Yansen Han$^{1,2,\ast}$ \quad
Shengyi Liao$^{3,\ast}$ \quad
Yuanxing Zhang$^{3}$ \quad 
Pengfei Wan$^{3}$ \quad 
Tao Lin$^{1,\dagger}$ \quad 
\\[0.5em]
$^\ast$Equal contribution \quad
$^\dagger$Corresponding author
\\[0.25em]
$^1$Westlake University \quad
$^2$Zhejiang University \quad
$^3$Kling Team, Kuaishou Technology
}

\begin{document}

\maketitle

\begin{abstract}
    Preference optimization is a standard alignment method for generative models, yet extending it to continuous-time dynamics remains non-trivial.
    In flow matching, reward-driven updates modify transport trajectories without an inherent constraint to the pretrained data manifold and can move terminal samples off the pretrained support.
    We formalize this failure mode as \emph{manifold drift}.
    Theoretically, we show that optimal flow matching recovers the terminal data distribution, whereas a preference update leaves the pretrained manifold whenever its induced terminal displacement has a nonzero normal component.
    As a remedy, we propose \method, a temperature-controlled objective that anchors pairwise preference optimization on preferred samples.
    Across temperature regimes, this objective connects rejection sampling fine-tuning and FlowDPO and controls a pointwise reconstruction-based surrogate for manifold distance.
    To counteract diminished signals at low temperatures, we further introduce a weighted variant, \methodweighted.
    On the main toy benchmark, \methodweighted attains a StrictScore of $0.899$, compared with $0.629$ for FlowDPO and $0.857$ for FlowDPO+RFT.
    On SD3.5-M at CFG $=4.5$, it improves OCR by $47.5\%$ and the average of four metrics by $16.0\%$.
\end{abstract}

\begin{figure}[h]
    \centering
    \centering
    \includegraphics[width=1\linewidth]{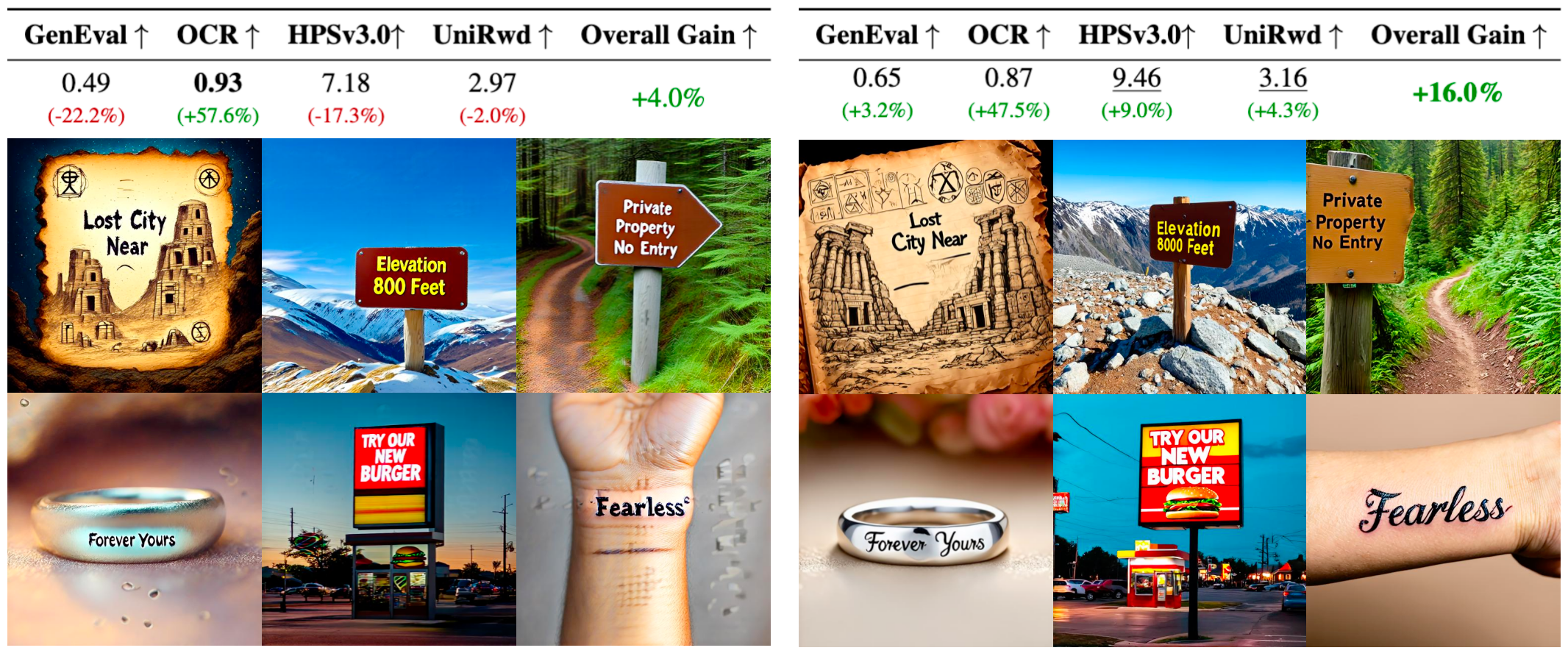}
    \caption{\small
        \textbf{Comparison of FlowDPO (left) and \methodweighted (right).}
        FlowDPO improves preference with manifold drift, whereas \methodweighted improves preference while preserving quality.
        \looseness=-1
    }
    \label{fig:comparison}
\end{figure}

\section{Introduction}

Direct Preference Optimization (DPO) has established itself as a simple and effective paradigm for aligning discrete generative models with human judgments~\citep{rafailov2023direct, ouyang2022training,sun2025solopo,tang2024generalized}. Motivated by this success, recent work has begun extending preference optimization to continuous generative models, including diffusion models~\citep{ho2020denoising, sohl2015deep, song2020score, wallace2024diffusion, domingoadjoint} and flow-based models~\citep{albergo2022building, lipman2022flow, sun2025unified, liu2025flow}.
This direction is crucial as these models now underpin state-of-the-art image and video generation, where preference alignment is essential for controllability.

\begin{figure*}[!t]
    \centering
    \includegraphics[width=0.9\textwidth]{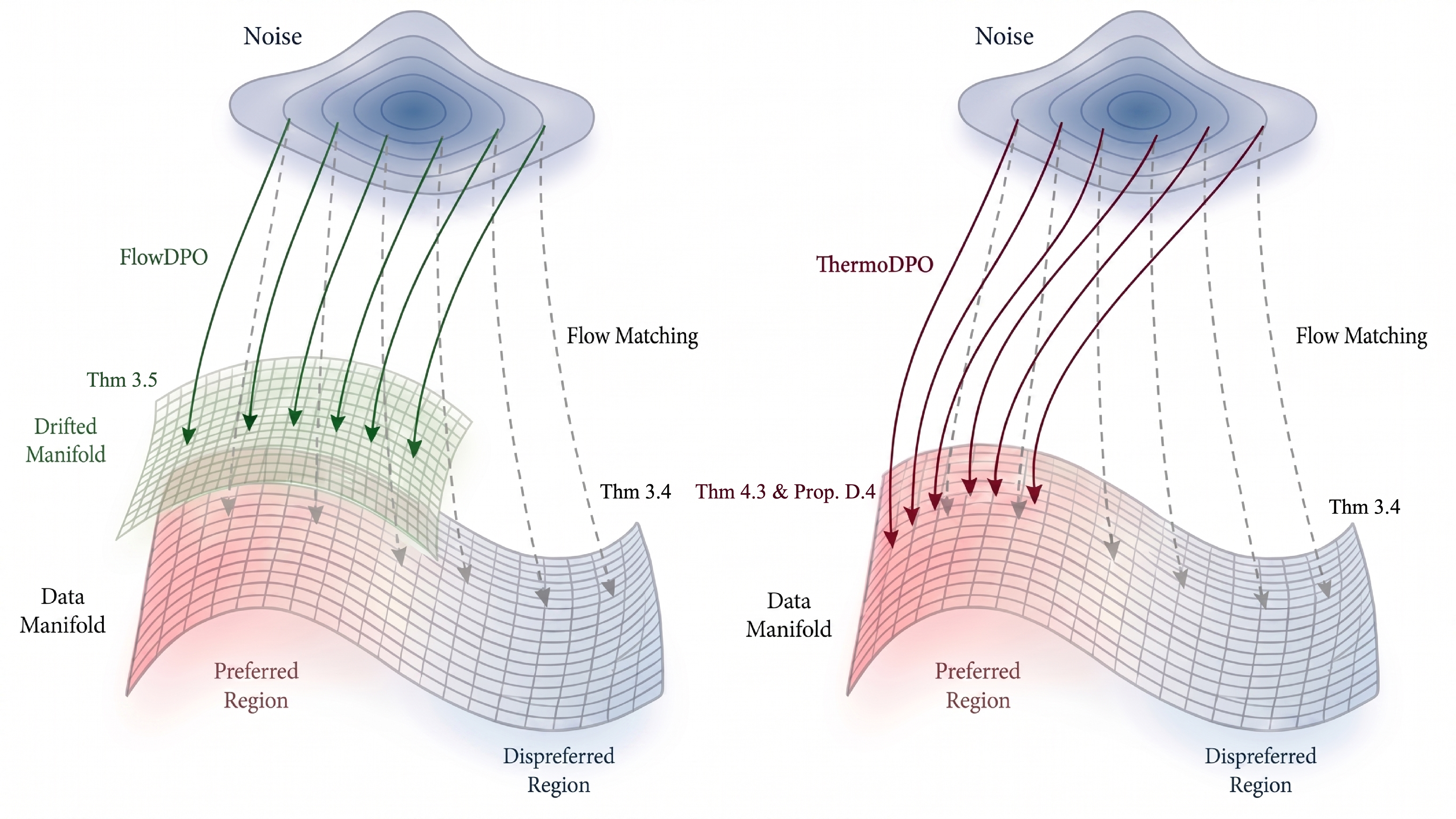}
    \vspace{-1em}
    \caption{\small
        \textbf{Core intuition of \method.}
        Flow Matching (gray) transports noise to the pretrained data manifold.
        FlowDPO (green) may reach preferred regions through an off-manifold displacement, whereas \method (red) adds a winner-side anchor intended to keep the redirected mass closer to the pretrained manifold.
        The formal statements and their assumptions are summarized in \tabref{tab:notation_theory_map}; the behavior is evaluated in the toy study (Fig.~\ref{fig:four_images}) and real-image study (Fig.~\ref{fig:comparison}).
    }
    \vspace{-1em}
    \label{fig:thermodpo_intuition}
\end{figure*}

However, transferring DPO to continuous generative models is not straightforward.
Unlike discrete models that primarily reweight the probabilities of completed outputs, continuous models generate samples by transporting noise along learned trajectories toward the data manifold.
Optimization in this setting therefore modifies not only output selection but also the transport dynamics themselves.

Although prior studies have noted empirical performance drops when applying DPO to continuous models~\citep{domingoadjoint, liu2025improving}, these findings have remained largely experimental observations.
In this paper, we argue that such degradation stems from a fundamental structural cause: \emph{reward-driven updates steer trajectories toward preferred regions only weakly supported by the pretrained data manifold}, a failure mode we formalize as \textbf{Manifold Drift}.
Beyond a simple geometric deviation, this drift damages the pretrained generative prior, leading to visibly degraded sample quality.
We characterize this problem in \secref{sec:manifold_drift} through theoretical analysis and supporting analytical evidence (see \figref{fig:four_images}).
\looseness=-1

To address these challenges, we introduce \textbf{\method}, which constrains trajectory endpoints to the pretrained manifold via a temperature-controlled anchor.
Analytically, \method unifies preference alignment and manifold preservation: it recovers rejection sampling fine-tuning (RFT)~\citep{xiong2025minimalist,chen2026nft} as $\tau \downarrow 0$, while for $\tau > 0$, it decomposes into a temperature-scaled FlowDPO~\citep{liu2025improving} objective and a nonnegative anchoring term.
This decomposition reveals an explicit trade-off between reward maximization and manifold drift.
To resolve practical weighting issues near $t=0$, we also propose a reweighted variant, \textbf{\methodweightedname}.

As shown in \figref{fig:thermodpo_intuition}, \method directs updates toward preferred regions while remaining anchored to the pretrained support, rather than treating alignment and preservation as conflicting goals.
Our experiments on both synthetic and real-world image benchmarks confirm that \methodweighted achieves a superior trade-off, significantly improving target metrics without compromising sample fidelity.
\textbf{Our main contributions are as follows:}


\begin{itemize}[leftmargin=*]
    \item We identify and formalize manifold drift, a failure mode in continuous preference optimization where the terminal of fine-tuned trajectories is away from the pretrained terminal manifold. \looseness=-1

    \item We propose \method, a method that augments pairwise preference optimization with a winner-side manifold anchor, and introduce a reweighted implementation, \methodweightedname, to ensure robust preservation of the terminal manifold in practice.

    \item We provide theoretical guarantees for both the \method and the \methodweighted, showing that it bridges rejection sampling fine-tuning and FlowDPO, while providing an upper bound on a reconstruction-based surrogate for manifold drift.

    \item We empirically demonstrate on a toy manifold and real-image benchmarks that \methodweightedname achieves a superior trade-off between preference alignment and manifold preservation compared to FlowDPO variants, improving target metrics without sacrificing visual quality.
\end{itemize}


\section{Preliminaries}
\label{sec:preliminaries}

\subsection{Flow Matching}
\label{subsec:flow_matching}
Flow Matching~\citep{lipman2022flow,albergo2022building} learns a time-dependent vector field that transports a simple prior distribution $p_1$ to the data distribution $p_0$.
Let $\mathbf x_0 \sim p_0$ denote a data sample and $\mathbf x_1 \sim p_1$ denote a noise sample.
Under the linear interpolation used throughout this paper, $\mathbf x_1$ is equivalently written as $\epsilon$.
For $t \in [0,1]$, define
\[
    \mathbf x_t = (1-t) \cdot \mathbf x_0 + t \cdot \mathbf x_1 \,,
\]
so generation proceeds from the noise endpoint $t=1$ to the data endpoint $t=0$.
The standard flow matching objective fits a vector field $v_\theta$ to the conditional velocity along this path:
\[
    \mathcal L_{\mathrm{FM}}(\theta)
    =
    \mathbb E_{(\mathbf x_0,\mathbf x_1,t)}
    \left[
        \|v_\theta(\mathbf x_t,t)-(\mathbf x_1-\mathbf x_0)\|^2
        \right] \,.
\]
Here $(\mathbf x_0,\mathbf x_1)\sim\gamma$ for a coupling $\gamma$ of $p_0$ and $p_1$, and $t\sim\mathrm{Unif}[0,1]$.
The learned vector field induces a flow map $\Phi^\theta_{1\to t}$ by solving the ODE $\frac{d\mathbf x_s}{ds} = v_\theta(\mathbf x_s,s)$ from $s=1$ to $s=t$.

\subsection{Direct Preference Optimization (DPO) in Continuous-Time Models}
\label{subsec:dpo_continuous}
For a preference dataset $\mathcal{D}=\{(c,\xx_0^w,\xx_0^l)\}$, standard DPO~\citep{rafailov2023direct} compares the log-likelihoods of winner $\xx_0^w$ and loser $\xx_0^l$ under the current model $\pi_\theta$ against a frozen reference model $\pi_{\mathrm{ref}}$:
\begin{equation}
    \label{eq:dpo}
    \mathcal{L}_{\text{DPO}}(\theta)
    =
    -\mathbb{E}_{(c,\xx_0^w,\xx_0^l)\sim\mathcal D}
    \left[
        \log \sigma \left(
        \beta \log \frac{\pi_\theta(\xx^w_0|c)}{\pi_{\mathrm{ref}}(\xx^w_0|c)}
        -
        \beta \log \frac{\pi_\theta(\xx^l_0|c)}{\pi_{\mathrm{ref}}(\xx^l_0|c)}
        \right)
        \right].
\end{equation}
Since evaluating exact log-likelihoods is computationally prohibitive for continuous models during training, prior works~\citep{liu2025improving,liu2025flow,zheng2025diffusionnft,wallace2024diffusion} replace $\log \pi_\theta(\xx_0 \mid c)$ with timestep-wise surrogates $\|v_\theta(\mathbf x_t,t)-(\mathbf x_1-\mathbf x_0)\|^2$.
DiffusionDPO \citep{wallace2024diffusion} uses denoising error, while FlowDPO \citep{liu2025improving} applies DPO to the flow matching regression loss.
To derive FlowDPO, we define the following notations with omitted shared condition $c$:
\begin{align}
    \label{eq:ell_def}
    \ell_\theta^{w/l} & := \|v_\theta(\mathbf x_t^{w/l},t)-(\mathbf x_1^{w/l}-\mathbf x_0^{w/l})\|^2, & \ell_{\mathrm{ref}}^{w/l} & := \|v_{\mathrm{ref}}(\mathbf x_t^{w/l},t)-(\mathbf x_1^{w/l}-\mathbf x_0^{w/l})\|^2, \\
    \label{eq:delta_def}
    \Delta_\theta^w   & := \ell_\theta^w - \ell_{\mathrm{ref}}^w,                                     & \Delta_\theta^l           & := \ell_\theta^l - \ell_{\mathrm{ref}}^l.
\end{align}
With $v_\theta$ and $v_{\mathrm{ref}}$ as the corresponding vector fields, the FlowDPO objective is:
\[
    \mathcal L_{\mathrm{FlowDPO}}(\theta)
    = \mathbb E \left[
        -\log \sigma\!\bigl(-\beta(\Delta_\theta^w-\Delta_\theta^l)\bigr)
        \right] \,.
\]

\begin{table*}[t]
    \centering
    \small
    \setlength{\tabcolsep}{5pt}
    \renewcommand{\arraystretch}{1.10}
    \caption{
        \textbf{Theorems' roadmap.}
        The analysis proceeds from characterizing manifold preservation and drift, through connecting \method to RFT and FlowDPO, to controlling winner-side manifold drift.}
    \label{tab:notation_theory_map}
    \begin{tabular}{@{}p{0.15\textwidth}p{0.82\textwidth}@{}}
        \toprule
        \textbf{Result} & \textbf{Main statement} \\
        \midrule
        \multicolumn{2}{@{}l}{\textbf{I. Manifold preservation and drift}} \\
        \addlinespace[0.15em]
        \thmref{thm:linear_fm_terminal_manifold}
            & FM can exactly recover $p_0$ and its terminal support (gray path in \figref{fig:thermodpo_intuition}). \\
        \thmref{thm:flowdpo_drift_main}
            & A nonzero normal component in the induced endpoint update is sufficient for off-manifold drift (green path in \figref{fig:thermodpo_intuition}). \\
        \addlinespace[0.45em]
        \multicolumn{2}{@{}l}{\textbf{II. Connection to RFT and FlowDPO}} \\
        \addlinespace[0.15em]
        \thmref{thm:thermo_to_rft}
            & As $\tau\downarrow0$, \method conditionally reduces to the RFT objective. \\
        \thmref{thm:thermo_vs_flowdpo}
            & For $\tau>0$, \method decomposes into FlowDPO and a winner-side anchor. \\
        \addlinespace[0.45em]
        \multicolumn{2}{@{}l}{\textbf{III. Manifold drift control}} \\
        \addlinespace[0.15em]
        \thmref{thm:thermo_exp_drift_suppression}
            & The pointwise loss upper-bounds the reconstructed winner's squared distance to $\mathcal M_{\mathrm{data}}$ (red path in \figref{fig:thermodpo_intuition}); it is not a distribution-level guarantee. \\
        \bottomrule
    \end{tabular}
\end{table*}

\section{Manifold Drift in Continuous Preference Optimization}
\label{sec:manifold_drift}

\paragraph{Motivation: manifold support and the drift problem.}
Preference optimization typically starts from a pretrained reference model with vector field $v_{\mathrm{ref}}$ and flow map $\Phi^{\mathrm{ref}}_{1\to t}$.
Following the manifold hypothesis~\citep{fefferman2016testing, lei2020geometric}, we assume natural data concentrate near a lower-dimensional set $\mathcal M_{\mathrm{data}} := \operatorname{supp}(p_0)$.
In practice, we use the \emph{pretrained terminal manifold} as an operational proxy for this set:
\[
    \mathcal M_0 := \operatorname{supp}\!\left((\Phi^{\mathrm{ref}}_{1\to 0})_\# p_1\right) \,.
\]
For any learned model $\theta$, we denote its terminal sample and distribution as $\xx_0 := \Phi^\theta_{1\to 0}(\mathbf x_1)$ and $\mu_\theta := (\Phi^\theta_{1\to 0})_\# p_1$.
The problem of \emph{manifold drift} arises when preference optimization steers $\mu_\theta$ away from $\mathcal M_0$ toward regions that lack generative support.
To rigorously analyze this, we first establish formal characterizations of terminal on-manifold flows (\defref{def:on_mani_flow}) and manifold drift (\defref{def:mani_drift}).

\begin{definition}[Terminal on-manifold flow]
    \label{def:on_mani_flow}
    Let $\mathcal M_0 \subset \mathbb R^d$ denote the pretrained terminal manifold, and let
    \[
        \mu_\theta := (\Phi^\theta_{1\to 0})_\# p_1
    \]
    denote the terminal distribution induced by the fine-tuned flow map $\Phi^\theta_{1\to 0}$.
    We say that $\Phi^\theta_{1\to 0}$ is \emph{terminally on-manifold} if
    \[
        \operatorname{supp}(\mu_\theta) \subseteq \mathcal M_0 \,.
    \]
\end{definition}

\begin{definition}[Manifold drift]
    \label{def:mani_drift}
    Let $\mu_\theta := (\Phi^\theta_{1\to 0})_\# p_1$ denote the terminal distribution induced by the fine-tuned flow map $\Phi^\theta_{1\to 0}$.
    We say that $\Phi^\theta_{1\to 0}$ exhibits \emph{manifold drift} with respect to the pretrained terminal manifold $\mathcal M_0$ if
    \[
        \operatorname{supp}(\mu_\theta)\nsubseteq \mathcal M_0 \,.
    \]
\end{definition}

\begin{remark}[Intuitive Interpretation]
    The manifold $\mathcal{M}_0$ can be considered as the ground-truth image manifold by manifold hypothesis or the terminal manifold learned during pretraining. The intended role of preference optimization is to reweight probability mass toward preferred regions while preserving the semantic and perceptual structure learned during pretraining. Manifold drift refers to the failure of this preservation: the aligned flow may move terminal samples outside the pretrained terminal manifold, potentially causing visual artifacts, semantic distortions, or degradation in sample fidelity, which is typically considered as a result of reward hacking.
\end{remark}

\paragraph{Preferred samples should not be outside the pretrained terminal manifold.}
We agree that moving beyond the pretrained manifold can be beneficial when fine-tuning a weak baseline. Our claim is that the manifold drift is theoretically illegal and can be practically risky.
\begin{itemize}[leftmargin=*]
    \item \textbf{Theoretical interpretation}: DPO~\citep{rafailov2023direct} is derived from KL-regularized reward maximization: $\max_\theta \mathbb{E}_{\pi_\theta} \left[r(x)\right] - \beta D_{KL}(\pi_\theta || \pi_{\mathrm{ref}})$, whose optimal solution is $\pi_\theta \propto \pi_{\mathrm{ref}}\exp(r(x)/\beta)$. This implies $supp(\pi^*) \subseteq supp(\pi_{\mathrm{ref}})$, and thus the original goal of DPO objective is to \textbf{condense the probability within the high-reward region of the support manifold}. However, practical flow-based RL often removes this constraint and directly optimizes the vector field.
    \item \textbf{Practical interpretation}: we can categorize the reward function into two types: manifold-aware and manifold-unaware. (1) Manifold-unaware reward: OCR is in this type because it only considers the correctness of the text in the generated image rather than the validity of the whole image. Therefore, using this kind of reward function, we can easily notice the manifold drift. (2) Manifold-aware reward: Pickscore is in this type because it focuses on the quality of the whole image, and thus fine-tuning with this kind of reward can hardly notice the manifold drift (still can happen). Using this kind of reward, FlowDPO can achieve higher reward due to manifold drift, while the valid visual quality make us unconscious about manifold drift.
\end{itemize}
Even though we only test this problem in image generation, we think manifold drift is more dangerous in robotics, because manifold drift means out-of-distribution behavior, i.e., unexpected behavior.

\paragraph{Optimal flow matching preserves the terminal manifold.}
\thmref{thm:linear_fm_terminal_manifold} states that, under exact optimization and the listed regularity assumptions, the induced FM flow recovers the data distribution at the terminal time.
\looseness=-1

\begin{theorem}[Optimal Flow Matching reaches the data manifold]
    \label{thm:linear_fm_terminal_manifold}
    Let $\mathcal M_{\mathrm{data}}:=\operatorname{supp}(p_0)$. Under linear interpolation and standard regularity assumptions ensuring that $v^\star$ generates a unique flow map and that the associated continuity equation admits a unique weak solution, then an optimal Flow Matching vector field $v^\star$ transports the prior $p_1$ exactly to the data distribution $p_0$:
    \[
        (\Phi^{v^\star}_{1\to 0})_\# p_1 = p_0 \,.
    \]
    Consequently,
    \[
        \operatorname{supp}\!\left((\Phi^{v^\star}_{1\to 0})_\# p_1\right)=\operatorname{supp}(p_0)=\mathcal M_{\mathrm{data}} \,.
    \]
\end{theorem}
See \appref{app:proof_linear_fm_terminal_manifold} for the proof, and \figref{fig:four_images} for the empirical toy example with setup details in \appref{app:toy_setup}.

\paragraph{On the failure of FlowDPO.}
In contrast to the ideal FM benchmark, \thmref{thm:flowdpo_drift_main} gives a sufficient condition under which a preference update leaves the pretrained terminal manifold.

\begin{theorem}[Manifold drift under a nonzero normal component assumption]
    \label{thm:flowdpo_drift_main}
    Let $\mathcal M_0:=\operatorname{supp}\!\left((\Phi^{\mathrm{ref}}_{1\to 0})_\# p_1\right)$ be the pretrained terminal manifold, assume $\mathcal M_0\subset\mathbb R^d$ is a twice continuously differentiable embedded submanifold, and fix $\mathbf x_1\in\mathcal M_1$. Define
    \[
        F(\theta,\mathbf x_1):=\Phi^\theta_{1\to 0}(\mathbf x_1),\qquad
        \mathbf x_0^\star:=F(\theta_0,\mathbf x_1)\in\mathcal M_0 \,,
    \]
    where $\theta_0=\theta_{\mathrm{ref}}$ and $F(\theta,\mathbf x_1)$ is differentiable in $\theta$ at $\theta_0$.
    Let $\Pi_{N_{\mathbf x_0^\star}\mathcal M_0}$ denote the orthogonal projection onto the normal space $N_{\mathbf x_0^\star}\mathcal M_0$ and $\mathcal L$ be the loss function.
    Let $\theta_1=\theta_0-\alpha\nabla_\theta \mathcal L(\theta_0)$ be one gradient step.
    If
    \[
        \Pi_{N_{\mathbf x_0^\star}\mathcal M_0}
        D_\theta F(\theta_0,\mathbf x_1)\big[\nabla_\theta \mathcal L(\theta_0)\big]\neq 0 \,,
    \]
    then there exists $\alpha_0>0$ such that for all $\alpha\in(0,\alpha_0)$, $F(\theta_1,\mathbf x_1)\notin\mathcal M_0$.
\end{theorem}
See \appref{app:proof_flowdpo_terminal_normal_drift} for the proof. The result is loss-agnostic: it does not assert that every FlowDPO update drifts, but identifies a nonzero normal component as sufficient. \figref{fig:four_images} shows this behavior in our controlled toy instance.
\begin{remark}[Why is FlowDPO prone to manifold drift?]
        At its optimum, FM regression $\E[\|v_\theta(\mathbf x_t,t)-(\mathbf x_1-\mathbf x_0)\|^2]$ preserves the target manifold under the assumptions of \thmref{thm:linear_fm_terminal_manifold}. The RFT term $\E[\|v_\theta(\mathbf x^w_t,t)-(\mathbf x_1-\mathbf x_0^w)\|^2]$ preserves the preferred region in the target manifold. FlowDPO can be simplied as $\E[\|v_\theta(\mathbf x^w_t,t)-(\mathbf x_1-\mathbf x_0^w)\|^2] - \E[\|v_\theta(\mathbf x^l_t,t)-(\mathbf x_1-\mathbf x_0^l)\|^2]$, where the substraction of the loser error gives a force to drift away from the manifold.
\end{remark}

\begin{figure*}[t]
    \centering
    \begin{minipage}{0.24\textwidth}
        \centering
        \includegraphics[width=\linewidth]{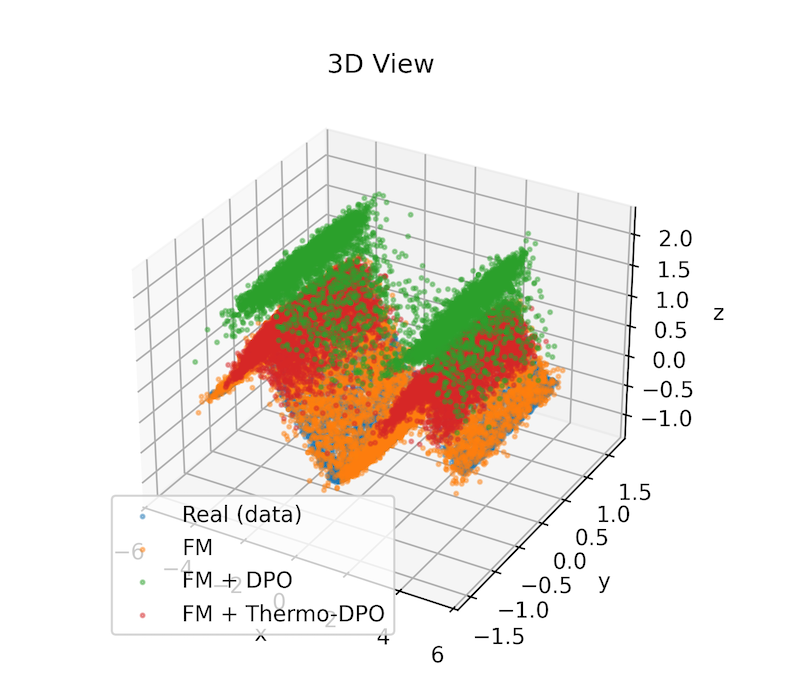}
    \end{minipage}
    \hfill
    \begin{minipage}{0.24\textwidth}
        \centering
        \includegraphics[width=\linewidth]{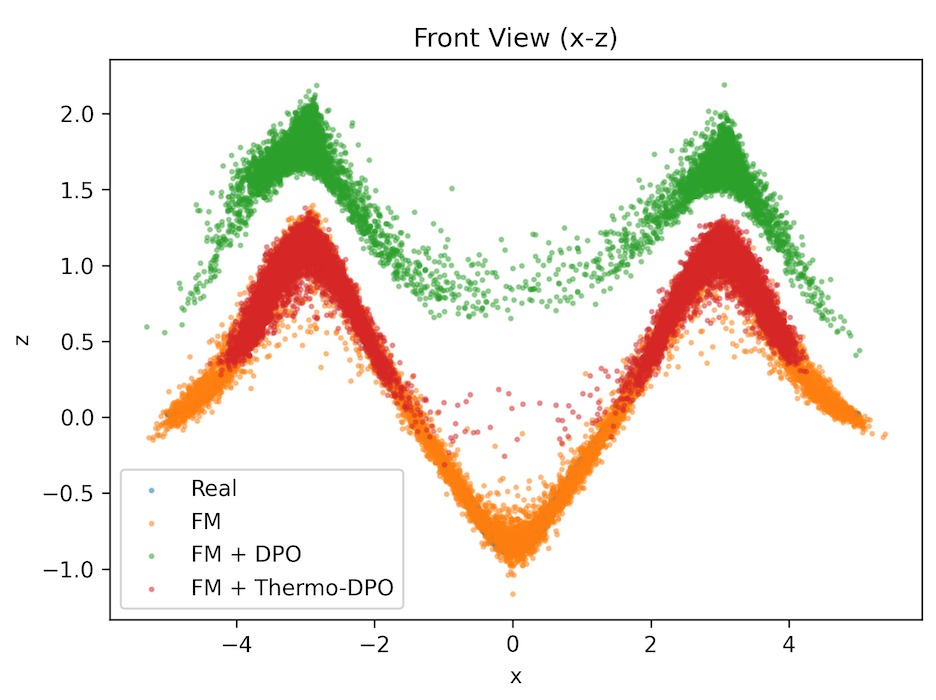}
    \end{minipage}
    \hfill
    \begin{minipage}{0.24\textwidth}
        \centering
        \includegraphics[width=\linewidth]{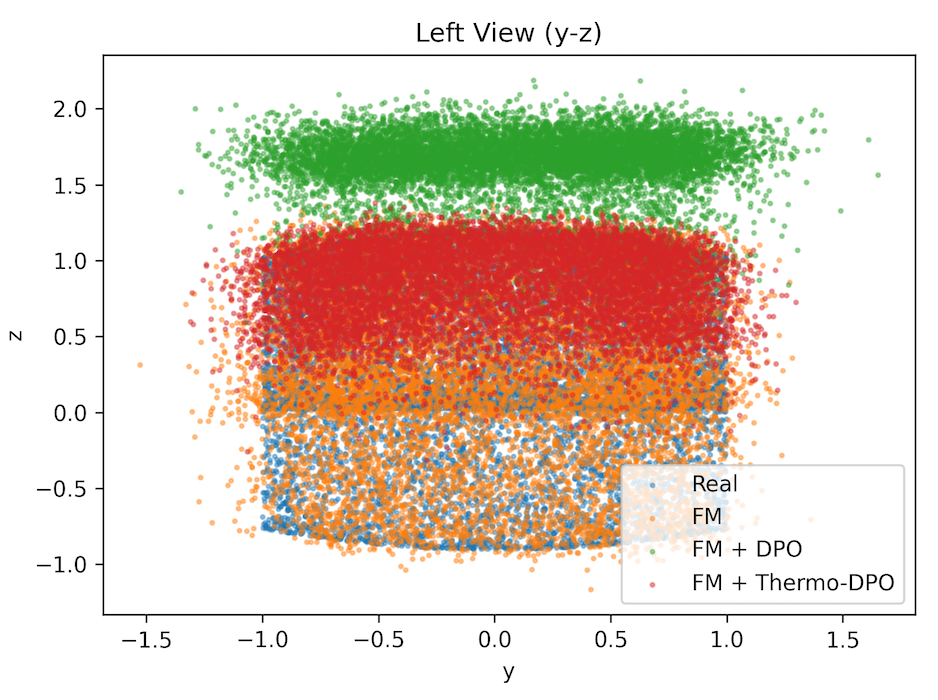}
    \end{minipage}
    \hfill
    \begin{minipage}{0.24\textwidth}
        \centering
        \includegraphics[width=\linewidth]{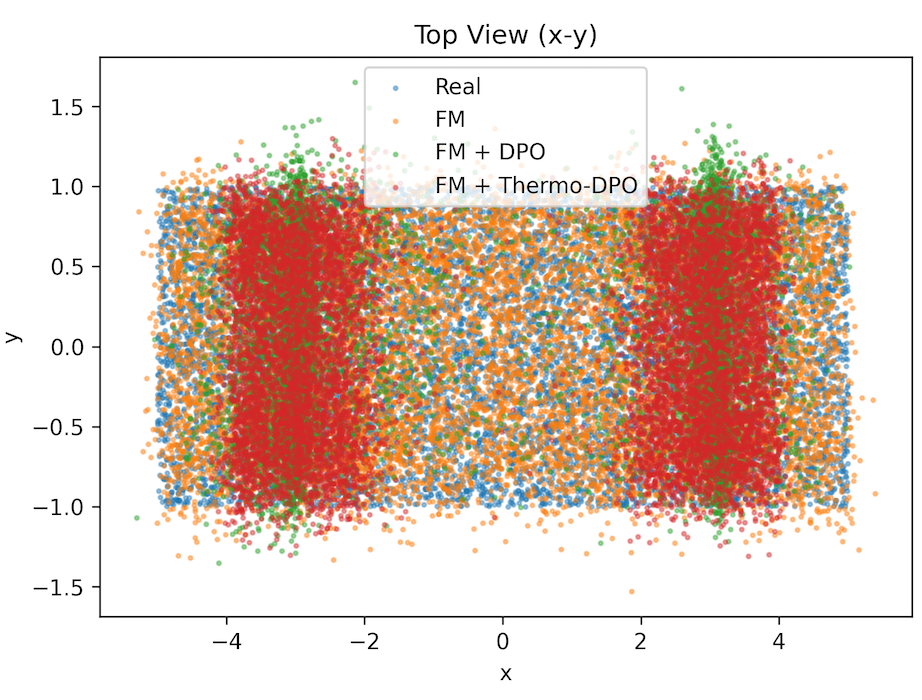}
    \end{minipage}
    \caption{
        \textbf{Toy example of manifold drift under direct preference optimization}.
        Starting from the same pretrained flow matching (FM) reference model, FlowDPO drives terminal samples toward preferred regions but also causes a deviation from the pretrained data manifold.
        In contrast, \method variant preserves the overall manifold structure much better while still improving alignment with the preference signal.
    }
    \label{fig:four_images}
\end{figure*}

\section{\method: Preference Optimization with Terminal Manifold Control}
\label{sec:terminal_manifold_preservation}

The possibility of FlowDPO leaving the terminal manifold (\secref{sec:manifold_drift}) raises a practical question: \emph{can we improve preference alignment while explicitly controlling terminal displacement?}

To resolve this tension, we introduce \textbf{\method}, a method that explicitly bridges preference alignment and manifold preservation.
Our approach is characterized by three key properties: a theoretical reduction to RFT (\secref{subsec:thermo_theory_relation}), an anchored-FlowDPO decomposition (\secref{subsec:thermo_theory_relation}), and a direct geometric bound on manifold drift (\secref{subsec:thermo_suppress_drift}).

\subsection{Temperature-Controlled Preference Optimization for Flow Models}
We formulate \method by viewing alignment as a thermodynamic balancing act, where a time-dependent temperature function $\tau(t)$ governs the trade-off between maximizing reward and anchoring mass to the pretrained manifold.
By varying $\tau(t)$, we can interpolate between the rigid constraints of rejection sampling (low temperature) and the flexible preference signal of FlowDPO (high temperature).
This formulation is inspired by Boltzmann distributions over energy states~\citep{aggarwal2025boltznce}, representing the competing goals of alignment and preservation.

Following the FlowDPO notations from \secref{sec:preliminaries}, we define three energy-like components that represent the preferred, rejected, and reference states:
\begin{align}
    \label{eq:thermo_energy_terms}
    E_w & = - \frac{\Delta_\theta^w}{\tau(t)} \,,
    \qquad
    E_l = - \frac{\Delta_\theta^l}{\tau(t)} \,,
    \qquad
    E_b = \frac{\ell_{\mathrm{ref}}^w}{\tau(t)} \,.
\end{align}
The \method objective then minimizes the negative log-probability of the preferred sample within this three-state system, effectively anchoring preference optimization:
\begin{equation}
    \label{eq:thermodpo_loss}
    \mathcal{L}_{\text{ThermoDPO}}(\theta) = \mathbb{E} \left[-\tau(t) \cdot t^2 \cdot \log \frac{e^{E_w}}{e^{E_w}+e^{E_l}+e^{E_b}}\right] \,.
\end{equation}
For our theoretical analysis, we use this standard formulation.
In practice, as the $t^2$ factor weakens the manifold anchor near the terminal endpoint ($t=0$), we propose and evaluate a reweighted variant, \methodweightedname, as detailed in \secref{subsec:thermo_practical_implementation}.

\subsection{Theoretical Relationship to Rejection Sampling Fine-Tuning and FlowDPO}
\label{subsec:thermo_theory_relation}

For the brevity of theoretical analysis, we rewrite the loss (in \eqref{eq:thermodpo_loss}) in terms of a single-sample integrand as follows:
\begin{equation}
    \label{eq:g_tau}
    \mathcal{L}_{\text{ThermoDPO}}(\theta)
    =
    \mathbb E\big[g_\tau(\theta)\big] \,, \quad
    g_\tau(\theta)
    :=
    t^2 \cdot \tau \cdot \log\!\left(
    1+
    \exp\!\left(\frac{\Delta_\theta^w-\Delta_\theta^l}{\tau}\right)
    +
    \exp\!\left(\frac{\ell_\theta^w}{\tau}\right)
    \right) \,.
\end{equation}
To analyze these properties pointwise, we fix a tuple $(\mathbf x_0^w,\mathbf x_0^l,t,\mathbf x_1)$ with $t\in(0,1]$, let $\tau := \tau(t)$, and recall $\ell_\theta^w := \|v_\theta(\mathbf x_t^w,t)-(\mathbf x_1-\mathbf x_0^w)\|^2$.
We first demonstrate that \method recovers rejection sampling fine-tuning (RFT)~\citep{xiong2025minimalist,chen2026nft} as the temperature vanishes.

\begin{theorem}[\method reduces to RFT]
    \label{thm:thermo_to_rft}
    For the integrand in \eqref{eq:g_tau},
    \begin{equation}
        \lim_{\tau \downarrow 0} g_\tau(\theta)
        = t^2 \cdot \max\!\left\{
        0 \,,
        \Delta_\theta^w-\Delta_\theta^l,\,
        \ell_\theta^w
        \right\} \,.
    \end{equation}
    Furthermore, if $\Delta_\theta^w-\Delta_\theta^l \leq \ell_\theta^w$, the objective reduces to the weighted reconstruction error:
    \[
        \lim_{\tau \downarrow 0} g_\tau(\theta) = \| \tilde{\mathbf x}_0^w - \mathbf x_0^w \|^2 = t^2 \ell_\theta^w \,,
    \]
    where $\tilde{\mathbf x}_0^w := \mathbf x_t^w - t \cdot v_\theta(\mathbf x_t^w, t)$ is the reconstructed preferred sample.
\end{theorem}

See \appref{app:proof_thermo_to_rft} for the proof.
In the low-temperature regime, \method effectively collapses to a time-weighted reconstruction objective whenever the preference signal is dominated by the manifold constraint.
Beyond this limit, we can analytically relate \method to the FlowDPO objective.
\looseness=-1

\begin{theorem}[\method as anchored FlowDPO]
    \label{thm:thermo_vs_flowdpo}
    For every $\tau>0$, the integrand $g_{\tau}(\theta)$ decomposes into a temperature-scaled FlowDPO objective and a nonnegative anchoring term:
    \begin{equation}
        \label{eq:thermo_vs_flowdpo_decomposition}
        g_{\tau}(\theta)
        =
        t^2 \cdot \tau \cdot \log\!\left(
        1+
        \exp\!\left(\frac{\Delta_\theta^w-\Delta_\theta^l}{\tau}\right)
        \right)
        +
        r_\tau(\theta) \,,
    \end{equation}
    where the anchoring term $r_\tau(\theta) \ge 0$ is defined as:
    \begin{equation}
        r_\tau(\theta)
        :=
        t^2 \cdot \tau \cdot \log\!\left(
        \frac{
            1+
            \exp\!\left(\frac{\Delta_\theta^w-\Delta_\theta^l}{\tau}\right)
            +
            \exp\!\left(\frac{\ell_\theta^w}{\tau}\right)
        }{
            1+
            \exp\!\left(\frac{\Delta_\theta^w-\Delta_\theta^l}{\tau}\right)
        }
        \right) \,.
    \end{equation}
\end{theorem}

The proof defers to \appref{app:proof_thm_thermo_vs_flowdpo}.
\thmref{thm:thermo_vs_flowdpo} shows algebraically that \method retains the pairwise FlowDPO term and adds the nonnegative winner-side penalty $r_\tau(\theta)$.
\looseness=-1

\subsection{\method Introduces Geometric Suppression of Manifold Drift}
\label{subsec:thermo_suppress_drift}
Beyond functional decomposition, \method provides direct geometric control over the terminal manifold departure by bounding the winner-side deviation of the reconstructed preferred sample.
This theoretical guarantee, formalized in \thmref{thm:thermo_exp_drift_suppression}, establishes the mathematical foundation for the anchored mass-redirection behavior (illustrated by the red line in \figref{fig:thermodpo_intuition}).
Specifically, under exact pretraining where $\mathcal M_0 = \mathcal M_{\mathrm{data}}$, this bound ensures that aligned mass remains anchored to the generative support even as the preference signal redirects it.

\begin{theorem}[Manifold drift control of \method]
    \label{thm:thermo_exp_drift_suppression}
    Let $\mathcal M_{\mathrm{data}}:=\operatorname{supp}(p_0)$.
    For the integrand in \eqref{eq:g_tau},
    \begin{equation}
        g_\tau(\theta)
        \ge
        \operatorname{dist}\!\bigl(\tilde{\mathbf x}_0^w,\mathcal M_{\mathrm{data}}\bigr)^2 \,.
        \label{eq:thermo_loss_lower_bound_distance}
    \end{equation}
    where $\tilde{\mathbf x}_0^w = \mathbf x_t^w - t \cdot v_\theta(\mathbf x_t^w, t)$ and $\operatorname{dist}(\tilde{\mathbf x}_0^w, \mathcal M_{\mathrm{data}}) = \inf_{y\in\mathcal M_{\mathrm{data}}} \|\tilde{\mathbf x}_0^w - y\|$.
\end{theorem}

See \appref{app:proof_thermo_exp_drift_suppression} for the proof.
Under exact FM pretraining, the same pointwise bound is relative to the pretrained terminal manifold $\mathcal M_0$. 

\subsection{From Theory to Practice: The Reweighted Variant of \method}
\label{subsec:thermo_practical_implementation}

While the theoretical objective \eqref{eq:thermodpo_loss} provides strong guarantees, its global $t^2$ coefficient causes the manifold anchor to vanish precisely near the terminal endpoint ($t=0$), where geometric preservation is most critical.
To resolve this weighting deficiency, we introduce \textbf{\methodweightedname}, which removes the global $t^2$ factor and instead activates the manifold anchor dynamically through a $(1-t)^2$ term:
\begin{equation}
    \label{eq:thermodpo_weighted_loss}
    \mathcal{L}_{\text{\methodweighted}}(\theta)
    =
    \mathbb E\left[
        \tau(t) \cdot \log\!\left(
        1+
        \exp\!\left(\frac{\Delta_\theta^w-\Delta_\theta^l}{\tau(t)}\right)
        +
        \exp\!\left(\frac{(1-t)^2 \cdot \ell_\theta^w}{\tau(t)}\right)
        \right)
        \right] \,.
\end{equation}
We evaluate \methodweighted as the definitive practical realization of the core objective across all experiments.
By substituting the vanishing $t^2$ weight with an endpoint-focused activation, this variant maintains a robust manifold anchor while inheriting all analytical guarantees of the core \method objective
(see \appref{app:thermo_weighted_analysis} for detailed discussion).
\looseness=-1

\section{Experiments}
\label{sec:experiments}
We evaluate \methodweighted on both synthetic and real-world image benchmarks.
Our synthetic experiments validate the core intuition (\figref{fig:thermodpo_intuition}) and analyze the trade-offs between preference alignment and manifold preservation.
On real-world tasks, we assess the performance of \methodweighted against RFT and FlowDPO variants using comprehensive automated and human metrics.

\subsection{Toy Experiments}
\label{subsec:toy_experiments}

\begin{table*}[t]
    \centering
    \small
    \setlength{\tabcolsep}{4.5pt}
    \renewcommand{\arraystretch}{1.04}
    \caption{\small
        \textbf{Toy results comparing RFT, FlowDPO variants, Diffusion-SDPO, Linear-DPO, $\chi$PO, and \methodweighted}.
        All methods start from the same pretrained flow-matching reference model and are fine-tuned for 10K steps.
        \emph{Win} and \emph{Loss} measure occupancy of the preferred and dispreferred regions regardless of manifold validity; \emph{StrictWin} requires samples to be both preferred and on-manifold; \emph{OnManifold} measures geometric validity; \emph{WinQuality} is the fraction of preferred samples that remain on-manifold; and $\mathrm{StrictScore}:=0.5\cdot\mathrm{StrictWin}+0.5\cdot\mathrm{OnManifold}$ summarizes the alignment-preservation trade-off.
        Best results are highlighted in \textbf{bold}; second-best results are \underline{underlined}.
        \looseness=-1
    }
    \label{tab:toy_main_results}
    \resizebox{\linewidth}{!}{
        \begin{tabular}{lcccccc}
            \toprule
            Method                                               & Win (\%) $\uparrow$ & Loss (\%) $\downarrow$ & StrictWin (\%) $\uparrow$ & OnManifold (\%) $\uparrow$ & WinQuality $\uparrow$ & StrictScore $\uparrow$ \\
            \midrule
            RFT                                                  & \textbf{93.6}       & \underline{0.3}        & 83.4                      & 88.3                       & 0.891                 & 0.858                  \\
            \midrule
            FlowDPO ($\beta=1$)                                  & 92                  & 0.8                    & 0                         & 0                          & 0                     & 0                      \\
            FlowDPO ($\beta=10$)                                 & 76.7                & 2.7                    & 3.2                       & 3.2                        & 0.042                 & 0.033                  \\
            FlowDPO ($\beta=100$)                                & 43.4                & 17.5                   & 37.7                      & 88.2                       & 0.868                 & 0.629                  \\
            FlowDPO ($\beta=500$)                                & 41.8                & 18.2                   & 35.6                      & 88.3                       & 0.852                 & 0.620                  \\
            \addlinespace[1pt]
            FlowDPO ($\beta=1$) + RFT                            & 91.7                & \underline{0.3}        & 82.9                      & 88.5                       & 0.903                 & 0.857                  \\
            FlowDPO ($\beta=10$) + RFT                           & 87.1                & 0.9                    & 77.7                      & 85.6                       & 0.893                 & 0.817                  \\
            FlowDPO ($\beta=100$) + RFT                          & 53.9                & 11.7                   & 47.6                      & 87.5                       & 0.884                 & 0.676                  \\
            FlowDPO ($\beta=500$) + RFT                          & 47.9                & 14.5                   & 41.5                      & 87.9                       & 0.866                 & 0.647                  \\
            \addlinespace[1pt]
            FlowDPO ($\beta=1$) + KL                             & 51.7                & 11.7                   & 45.7                      & 87.4                       & 0.885                 & 0.666                  \\
            FlowDPO ($\beta=10$) + KL                            & 43.8                & 18                     & 37.8                      & 84.3                       & 0.864                 & 0.611                  \\
            FlowDPO ($\beta=100$) + KL                           & 41.9                & 18.6                   & 36.4                      & 89.1                       & 0.868                 & 0.627                  \\
            FlowDPO ($\beta=500$) + KL                           & 40.7                & 19.2                   & 34.7                      & 88.4                       & 0.854                 & 0.616                  \\
            \addlinespace[1pt]
            Diffusion-SDPO ($\beta=1$, $\mu=0.99$)              & 66.1                & 5.5                    & 21.5                      & 21.6                       & 0.325                 & 0.215                  \\
            Diffusion-SDPO ($\beta=10$, $\mu=0.99$)             & 47.8                & 13.4                   & 30.7                      & 53.6                       & 0.642                 & 0.421                  \\
            Diffusion-SDPO ($\beta=100$, $\mu=0.99$)            & 42.2                & 17.5                   & 37.5                      & 91.3                       & 0.888                 & 0.644                  \\
            Diffusion-SDPO ($\beta=500$, $\mu=0.99$)            & 40.9                & 18.2                   & 36.3                      & 91.3                       & 0.887                 & 0.638                  \\
            \addlinespace[1pt]
            Linear-DPO ($\beta=1$)                               & 78.7                & \textbf{0.2}           & 26.9                      & 41.4                       & 0.341                 & 0.341                  \\
            Linear-DPO ($\beta=10$)                              & 85.8                & 0.5                    & 74.3                      & 78.0                       & 0.866                 & 0.761                  \\
            Linear-DPO ($\beta=100$)                             & 89.3                & 0.7                    & 56.9                      & 56.9                       & 0.638                 & 0.569                  \\
            Linear-DPO ($\beta=500$)                             & 89.6                & 0.8                    & 44.6                      & 44.6                       & 0.498                 & 0.446                  \\
            \addlinespace[1pt]
            $\chi$PO ($\beta=1$)                                 & \underline{92.7}    & 1                      & 0                         & 0                          & 0                     & 0                      \\
            $\chi$PO ($\beta=10$)                                & 80.3                & 3.3                    & 0.8                       & 0.8                        & 0.01                  & 0.008                  \\
            $\chi$PO ($\beta=100$)                               & 43                  & 17.8                   & 37.2                      & 88.4                       & 0.865                 & 0.628                  \\
            $\chi$PO ($\beta=500$)                               & 41.1                & 17.5                   & 35                        & 88.2                       & 0.851                 & 0.616                  \\
            \midrule
            \rowcolor{gray!15}\multicolumn{7}{l}{\textit{\textup{\methodweightedname} with $\tau(t)=\frac{t}{\beta}$}}                                                                                                    \\
            \rowcolor{gray!15}\methodweighted ($t$, $\beta=1$)   & \underline{92.7}    & 0.5                    & \textbf{87.6}             & \textbf{92.2}              & \textbf{0.945}        & \textbf{0.899}         \\
            \rowcolor{gray!15}\methodweighted ($t$, $\beta=10$)  & 92.0                & 0.4                    & 86.1                      & 91.5                       & 0.935                 & 0.888                  \\
            \rowcolor{gray!15}\methodweighted ($t$, $\beta=100$) & 91.2                & 0.4                    & 85.9                      & \underline{92.1}           & \underline{0.941}     & 0.89                   \\
            \rowcolor{gray!15}\methodweighted ($t$, $\beta=500$) & 91.9                & 0.4                    & \underline{86.2}          & \textbf{92.2}              & 0.939                 & \underline{0.892}      \\
            \bottomrule
        \end{tabular}
    }
\end{table*}

\paragraph{Experimental setup.}
We use the analytic surface $z=f(x,y)$ in $\mathbb R^3$ shown in \figref{fig:four_images}: the two bumps near $x=\pm3$ are preferred, the central dip near $x=0$ is dispreferred, and the remaining surface is neutral. All methods start from the same three-layer flow-matching MLP, use the same winner--loser pairs and $10{,}000$-step fine-tuning budget, and are evaluated on $10{,}000$ generated samples. A point is on-manifold when it lies in the surface domain and satisfies $|z-f(x,y)|\le0.15$. We compare RFT~\cite{xiong2025minimalist,chen2026nft}, FlowDPO~\cite{liu2025improving}, $\chi$PO~\cite{huang2024correcting}, FlowDPO+KL, FlowDPO+RFT, Diffusion-SDPO~\citep{fu2025diffusionsdpo}, Linear-DPO~\citep{li2026lineardpo}, and \methodweighted; full architecture and optimizer details are in \appref{app:toy_setup}.

\paragraph{Results.}

\tabref{tab:toy_main_results} shows a sharp trade-off between preference optimization and manifold preservation.
\begin{itemize}[leftmargin=*]
    \item \textbf{Vanilla FlowDPO and $\chi$PO fails to preserve terminal manifold}: Vanilla FlowDPO and $\chi$PO can achieve high Winner Ratio, but often does so by leaving the manifold: for example, at $\beta=1$ it reaches 92\% Win while OnManifold drops to 0\%, causing StrictWin and WinQuality to collapse. This is exactly the failure mode we call manifold drift.
    \item \textbf{Adding explicit regularization term helps to preserve terminal manifold, while \methodweighted performs best}: FlowDPO+RFT and \methodweighted both preserve high OnManifold scores while recovering strong preference performance. RFT is a strong baseline, but among pairwise preference objectives \methodweighted achieves the best balance: with $\tau(t)=t$, it attains 92.7\% Win, 87.6\% StrictWin, and the best StrictScore of 0.899. KL regularization preserves the manifold more than vanilla FlowDPO, but improves preference less.
\end{itemize}

\newsavebox{\toyTauAblationTableBox}
\begin{wraptable}{r}{0.45\columnwidth}
    \vspace{-0.7\baselineskip}
    \centering
    \sbox{\toyTauAblationTableBox}{%
        \scriptsize
        \setlength{\tabcolsep}{2pt}%
        \renewcommand{\arraystretch}{1.04}%
        \begin{tabular}{@{}lccc@{}}
            \toprule
            $\beta\tau(t)$ & Win $\uparrow$ & OnM. $\uparrow$ & Strict $\uparrow$ \\
            \midrule
            $t$ & 92.7 & \textbf{92.2} & 0.899 \\
            $\frac{t^2}{t^2+(1-t)^2}$ & 92.5 & 91.9 & 0.895 \\
            $t^{10}$ & 91.8 & \textbf{92.2} & 0.893 \\
            $t^{0.1}$ & \textbf{93.1} & \textbf{92.2} & \textbf{0.900} \\
            \bottomrule
        \end{tabular}%
    }
    \captionsetup{font=footnotesize,skip=3pt,width=0.9\linewidth}
    \caption{
        \textbf{Temperature ablation} at $\beta=1$.
        The full sweep is in \tabref{tab:thermo_dpo_results}.}
    \label{tab:toy_tau_ablation_main}
    \usebox{\toyTauAblationTableBox}
    \vspace{-0.5\baselineskip}
\end{wraptable}

\paragraph{Sensitivity analysis.}
Across the four schedules in \tabref{tab:toy_tau_ablation_main}, StrictScore ranges from $0.893$ to $0.900$ at $\beta=1$. For the linear schedule in \tabref{tab:toy_main_results}, the score ranges from $0.888$ to $0.899$ over $\beta\in\{1,10,100,500\}$. These results support stability only within this toy grid; the complete sweep, including the more sensitive unweighted objective, appears in \appref{app:toy_exp_results_thermo_dpo_variants}.

\subsection{Real-Image Generation Experiments}
\label{subsec:real_image}

In this section, we aim to test the practical consequence on real-image generation: can preference optimization improve reward without sacrificing prompt fidelity, perceptual quality, or agreement with human judgment?

\paragraph{Experimental setup.}
All real-image runs start from the Stable Diffusion 3.5-M checkpoint and use the same OCR preference-pair dataset. We evaluate the optimized OCR metric together with GenEval~\cite{ghosh2023geneval}, HPSv3.0~\cite{ma2025hpsv3widespectrumhumanpreference}, and UniReward~\cite{unifiedreward}; full training and sampling details are in \appref{app:real_image_setup}.

\paragraph{Reward-model evaluation.}
We will report automatic scores from the training reward and held-out evaluators. The key question is whether a method improves OCR while retaining gains on metrics it was not directly optimized for. Improvements restricted to the optimized reward are more suggestive of reward hacking.

\begin{table*}[t]
    \centering
    \small
    \setlength{\tabcolsep}{4.5pt}
    \renewcommand{\arraystretch}{1.12}
    \caption{
        \textbf{Quantitative comparison on \texttt{SD3.5-M}.}
        All compared methods are trained using the OCR preference pair dataset, while evaluation is conducted across GenEval, OCR, HPSv3.0, and UniRwd.
        For each metric, we report the absolute score at both CFG settings, and report the relative change (\%) with respect to the \texttt{SD3.5-M} baseline at the CFG$=4.5$.
        The overall score is defined as the macro-average relative gain across these four metrics.
        Best results are highlighted in \textbf{bold} and the second-best results are \underline{underlined}.}
    \label{tab:main_results_relative_cfg45}
    \resizebox{1\textwidth}{!}{
        \begin{tabular}{l c c >{\columncolor{gray!15}}c c c c}
            \toprule
            \textbf{Model} & \textbf{CFG}
                           & \textbf{GenEval~\cite{ghosh2023geneval} $\uparrow$}
                           & \textbf{OCR $\uparrow$}
                           & \textbf{HPSv3.0~\cite{ma2025hpsv3widespectrumhumanpreference} $\uparrow$}
                           & \textbf{UniRwd~\cite{unifiedreward} $\uparrow$}
                           & \textbf{Overall Gain $\uparrow$}                                          \\
            \midrule

            \multirow{2}{*}{SD3.5-M (Baseline)}
                           & $2.0$
                           & \makecell[c]{0.53}
                           & \makecell[c]{$\ \ $0.36$\ \ \ $}
                           & \makecell[c]{5.38}
                           & \makecell[c]{2.78}
                           & \makecell[c]{--}                                                          \\
                           & $4.5$
                           & \makecell[c]{0.63}
                           & \makecell[c]{$\ \ $0.59$\ \ \ $}
                           & \makecell[c]{8.68}
                           & \makecell[c]{3.03}
                           & \makecell[c]{--}                                                          \\
            \midrule

            \multirow{2}{*}{RFT~\cite{xiong2025minimalist,chen2026nft}}
                           & $2.0$
                           & \makecell[c]{\underline{0.68} \\ {\scriptsize \posc{(+7.9\%)}}}
                           & \makecell[c]{0.67 \\ {\scriptsize \posc{(+13.6\%)}}}
                           & \makecell[c]{8.98 \\ {\scriptsize \posc{(+3.5\%)}}}
                           & \makecell[c]{3.11 \\ {\scriptsize \posc{(+2.6\%)}}}
                           & \makecell[c]{\posc{+6.9\%}}                                               \\
                           & $4.5$
                           & \makecell[c]{\textbf{0.70} \\ {\scriptsize \posc{(+11.1\%)}}}
                           & \makecell[c]{0.74 \\ {\scriptsize \posc{(+25.4\%)}}}
                           & \makecell[c]{\textbf{9.59} \\ {\scriptsize \posc{(+10.5\%)}}}
                           & \makecell[c]{\textbf{3.19} \\ {\scriptsize \posc{(+5.3\%)}}}
                           & \makecell[c]{\posc{+13.1\%}}                                              \\
            \cmidrule(l){2-7}

            \multirow{2}{*}{FlowDPO~\cite{liu2025improving} ($\beta=100$)}
                           & $2.0$
                           & \makecell[c]{0.49 \\ {\scriptsize \negc{(-22.2\%)}}}
                           & \makecell[c]{\textbf{0.93} \\ {\scriptsize \posc{(+57.6\%)}}}
                           & \makecell[c]{7.18 \\ {\scriptsize \negc{(-17.3\%)}}}
                           & \makecell[c]{2.97 \\ {\scriptsize \negc{(-2.0\%)}}}
                           & \makecell[c]{\posc{+4.0\%}}                                               \\
                           & $4.5$
                           & \makecell[c]{0.46 \\ {\scriptsize \negc{(-27.0\%)}}}
                           & \makecell[c]{0.70 \\ {\scriptsize \posc{(+18.6\%)}}}
                           & \makecell[c]{6.83 \\ {\scriptsize \negc{(-21.3\%)}}}
                           & \makecell[c]{2.89 \\ {\scriptsize \negc{(-4.6\%)}}}
                           & \makecell[c]{\negc{-8.6\%}}                                               \\
            \cmidrule(l){2-7}

            \multirow{2}{*}{FlowDPO + RFT ($\beta=100$)}
                           & $2.0$
                           & \makecell[c]{0.58 \\ {\scriptsize \negc{(-7.9\%)}}}
                           & \makecell[c]{\underline{0.92} \\ {\scriptsize \posc{(+55.9\%)}}}
                           & \makecell[c]{7.68 \\ {\scriptsize \negc{(-11.5\%)}}}
                           & \makecell[c]{3.05 \\ {\scriptsize \posc{(+0.7\%)}}}
                           & \makecell[c]{\posc{+9.3\%}}                                               \\
                           & $4.5$
                           & \makecell[c]{0.64 \\ {\scriptsize \posc{(+1.6\%)}}}
                           & \makecell[c]{0.88 \\ {\scriptsize \posc{(+49.2\%)}}}
                           & \makecell[c]{8.85 \\ {\scriptsize \posc{(+2.0\%)}}}
                           & \makecell[c]{3.10 \\ {\scriptsize \posc{(+2.3\%)}}}
                           & \makecell[c]{\posc{\underline{+13.8\%}}}                                  \\
            \cmidrule(l){2-7}

            \multirow{2}{*}{FlowDPO + KL ($\beta=100$)}
                           & $2.0$
                           & \makecell[c]{0.56 \\ {\scriptsize \negc{(-11.1\%)}}}
                           & \makecell[c]{\underline{0.92} \\ {\scriptsize \posc{(+55.9\%)}}}
                           & \makecell[c]{6.58 \\ {\scriptsize \negc{(-24.2\%)}}}
                           & \makecell[c]{2.99 \\ {\scriptsize \negc{(-1.3\%)}}}
                           & \makecell[c]{\posc{+4.8\%}}                                               \\
                           & $4.5$
                           & \makecell[c]{0.62 \\ {\scriptsize \negc{(-1.6\%)}}}
                           & \makecell[c]{\underline{0.92} \\ {\scriptsize \posc{(+55.9\%)}}}
                           & \makecell[c]{8.58 \\ {\scriptsize \negc{(-1.2\%)}}}
                           & \makecell[c]{3.09 \\ {\scriptsize \posc{(+2.0\%)}}}
                           & \makecell[c]{\posc{\underline{+13.8\%}}}                                  \\
            \cmidrule(l){2-7}

            \multirow{2}{*}{\methodweighted ($t$, $\beta=100$)}
                           & $2.0$
                           & \makecell[c]{0.59 \\ {\scriptsize \negc{(-6.3\%)}}}
                           & \makecell[c]{0.84 \\ {\scriptsize \posc{(+42.4\%)}}}
                           & \makecell[c]{7.29 \\ {\scriptsize \negc{(-16.0\%)}}}
                           & \makecell[c]{2.99 \\ {\scriptsize \negc{(-1.3\%)}}}
                           & \makecell[c]{\posc{+4.7\%}}                                               \\
                           & $4.5$
                           & \makecell[c]{0.65 \\ {\scriptsize \posc{(+3.2\%)}}}
                           & \makecell[c]{0.87 \\ {\scriptsize \posc{(+47.5\%)}}}
                           & \makecell[c]{\underline{9.46} \\ {\scriptsize \posc{(+9.0\%)}}}
                           & \makecell[c]{\underline{3.16} \\ {\scriptsize \posc{(+4.3\%)}}}
                           & \makecell[c]{\textbf{\posc{+16.0\%}}}                                     \\
            \bottomrule
        \end{tabular}
    }
\end{table*}

Quantitavely, we report experimental results in \tabref{tab:main_results_relative_cfg45}, and \textbf{\methodweightedname achieves the strongest overall gain} within the FlowDPO family and improves Geneval, OCR, HPSv3.0, and UniReward over the pretrained baseline. RFT remains a strong baseline on several held-out metrics.

Qualitatively, we report the generated images of all the compared methods in \appref{app:real_image_qualitative}. The qualitative results show that \textbf{\methodweightedname and RFT retain visual quality closer to the pretrained model, whereas FlowDPO, FlowDPO+RFT, and FlowDPO+KL exhibit noticeable quality degradation.} Together, these results support our claim that \methodweighted improves target metrics without visible quality degradation.

\paragraph{Human evaluation.}
Reward models cannot fully determine whether improved scores correspond to genuinely better images, so we also run pairwise human evaluation on 30 prompts following \appref{app:real_image_human}. \methodweighted remains competitive on text accuracy and is generally preferred on visual quality. The result of human evaluation is shown in \figref{fig:human_evaluation}.

\begin{figure*}[h]
    \centering
    \includegraphics[width=0.9\textwidth]{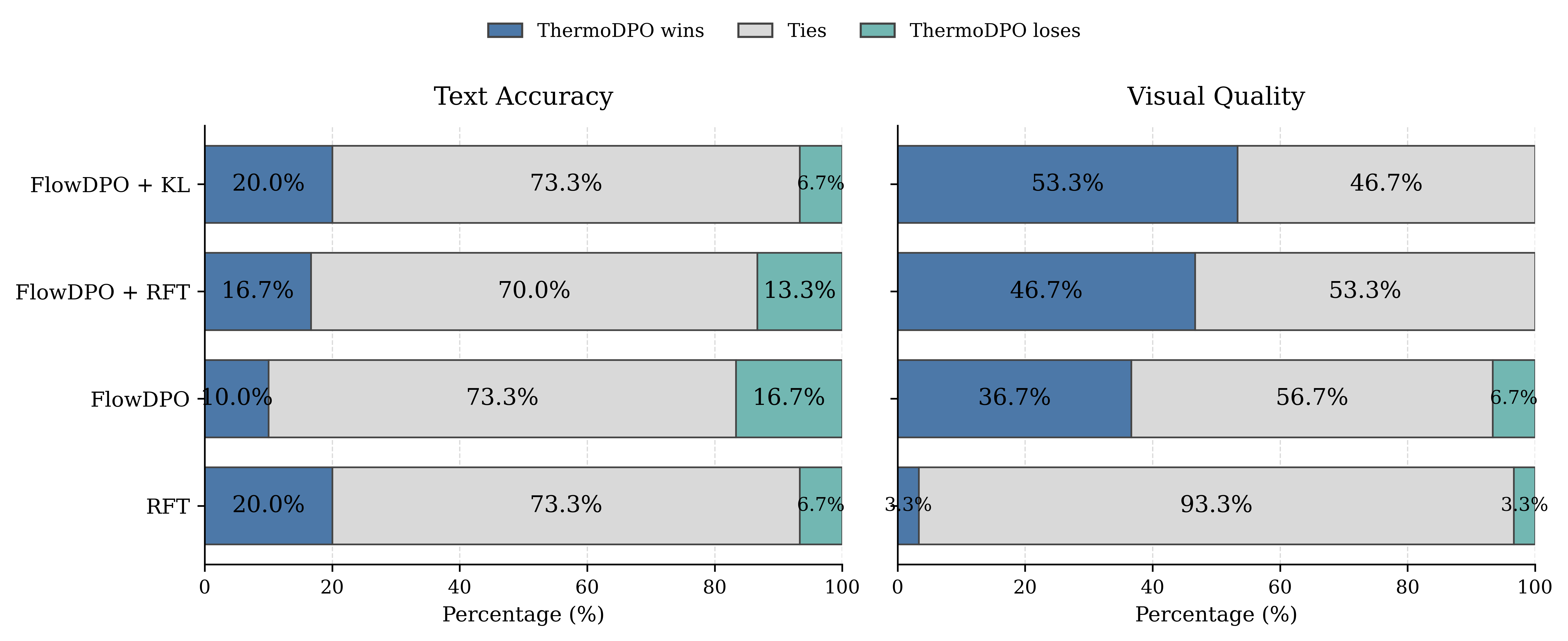}
    \caption{
        \textbf{Pairwise human evaluation of \methodweighted against different baselines} on text accuracy and visual quality over 30 prompts.
        Each stacked bar reports the percentage of prompts for which \methodweighted is preferred, tied, or dispreferred relative to the corresponding baseline.
        \methodweighted shows consistently stronger performance on visual quality while remaining competitive on text accuracy.
    }
    \label{fig:human_evaluation}
\end{figure*}

\section{Limitations}
\label{sec:limitations}

This paper studies flow-based preference optimization in the offline setting. All methods are trained from a fixed winner--loser dataset, so we do not address the additional exploration, reward-updating, and stability issues that appear in online RLHF. We also do not test whether the same idea transfers cleanly to other continuous-time or diffusion-based alignment algorithms. In addition, although \method is motivated by thermodynamic energy functions and Boltzmann distributions, our analysis only establishes its optimization and manifold-control properties rather than a fully principled physical derivation; a better physically grounded objective may therefore exist. Besides, extending manifold-drift control to online RL algorithms for flow preference optimization is a natural next step.
\looseness=-1

\section{Conclusion}

This paper studies preference optimization for continuous-time generative models through the lens of \emph{manifold drift}. We argue that, in flow-based models, preference optimization does not only change which outputs are favored, it also changes the transport dynamics that produce them. This creates a failure mode in which reward-based metrics improve while terminal samples move away from regions supported by the pretrained model. To make this issue explicit, we formalize manifold drift, show that optimal Flow Matching recovers the terminal data distribution, and give a first-order result showing that FlowDPO can admit off-manifold updates.

Motivated by this analysis, we introduce \method, a temperature-controlled objective that adds a winner-side anchor to pairwise preference optimization. Our theory shows that the \method objective reduces to RFT in the low-temperature regime under a mild condition, decomposes into a temperature-scaled FlowDPO term plus a nonnegative anchoring term, and upper bounds a reconstruction-based manifold-distance surrogate on preferred samples. Because the \method objective weakens the anchor near the terminal endpoint in practice, we evaluate a reweighted implementation, \methodweighted, in all experiments.

Empirically, the toy experiment shows a clear trade-off between preference optimization and manifold preservation, and demonstrates that \methodweighted achieves a substantially better balance than vanilla FlowDPO and its regularized variants. On real-image generation, \methodweighted improves OCR-oriented alignment while remaining competitive on held-out automatic metrics and human evaluation. Taken together, these results suggest that \methodweighted is a practical and effective method for preference optimization in continuous-time generative models.


\bibliography{resources/reference}

\begin{thebibliography}{72}
\providecommand{\natexlab}[1]{#1}
\providecommand{\url}[1]{\texttt{#1}}
\expandafter\ifx\csname urlstyle\endcsname\relax
  \providecommand{\doi}[1]{doi: #1}\else
  \providecommand{\doi}{doi: \begingroup \urlstyle{rm}\Url}\fi

\bibitem[Aggarwal et~al.(2025)Aggarwal, Chen, Boffi, and Koes]{aggarwal2025boltznce}
Rishal Aggarwal, Jacky Chen, Nicholas~M Boffi, and David~Ryan Koes.
\newblock Boltznce: Learning likelihoods for boltzmann generation with stochastic interpolants and noise contrastive estimation.
\newblock \emph{arXiv preprint arXiv:2507.00846}, 2025.

\bibitem[Albergo and Vanden-Eijnden(2022)]{albergo2022building}
Michael~S Albergo and Eric Vanden-Eijnden.
\newblock Building normalizing flows with stochastic interpolants.
\newblock \emph{arXiv preprint arXiv:2209.15571}, 2022.

\bibitem[Albergo et~al.(2025)Albergo, Boffi, and Vanden-Eijnden]{albergo2025stochastic}
Michael~S. Albergo, Nicholas~M. Boffi, and Eric Vanden-Eijnden.
\newblock Stochastic interpolants: A unifying framework for flows and diffusions.
\newblock \emph{Journal of Machine Learning Research}, 26\penalty0 (209):\penalty0 1--80, 2025.
\newblock URL \url{https://arxiv.org/abs/2303.08797}.

\bibitem[Black et~al.(2024)Black, Janner, Du, Kostrikov, and Levine]{black2024ddpo}
Kevin Black, Michael Janner, Yilun Du, Ilya Kostrikov, and Sergey Levine.
\newblock Training diffusion models with reinforcement learning.
\newblock In \emph{The Twelfth International Conference on Learning Representations}, 2024.
\newblock URL \url{https://arxiv.org/abs/2305.13301}.

\bibitem[Chen et~al.(2026)Chen, Zheng, Zhang, Cui, Cui, Ye, Lin, Liu, Zhu, and Wang]{chen2026nft}
Huayu Chen, Kaiwen Zheng, Qinsheng Zhang, Ganqu Cui, Yin Cui, Haotian Ye, Tsung-Yi Lin, Ming-Yu Liu, Jun Zhu, and Haoxiang Wang.
\newblock Nft: Bridging supervised learning and reinforcement learning in math reasoning.
\newblock In \emph{International Conference on Learning Representations}, volume 2026, pages 124025--124042, 2026.

\bibitem[Chen et~al.(2018)Chen, Rubanova, Bettencourt, and Duvenaud]{chen2018neural}
Ricky~TQ Chen, Yulia Rubanova, Jesse Bettencourt, and David~K Duvenaud.
\newblock Neural ordinary differential equations.
\newblock \emph{Advances in neural information processing systems}, 31, 2018.

\bibitem[Clark et~al.(2024)Clark, Vicol, Swersky, and Fleet]{clark2024draft}
Kevin Clark, Paul Vicol, Kevin Swersky, and David~J. Fleet.
\newblock Directly fine-tuning diffusion models on differentiable rewards.
\newblock In \emph{The Twelfth International Conference on Learning Representations}, 2024.
\newblock URL \url{https://arxiv.org/abs/2309.17400}.

\bibitem[Ding et~al.(2026)Ding, Chen, Lyu, Yuan, Zhu, Tian, Zhu, Wang, Deng, Mi, Shang, and Lin]{ding2026rethinking}
Bowen Ding, Yuhan Chen, Jiayang Lyu, Jiyao Yuan, Qi~Zhu, Shuangshuang Tian, Dantong Zhu, Futing Wang, Heyuan Deng, Fei Mi, Lifeng Shang, and Tao Lin.
\newblock Rethinking expert trajectory utilization in {LLM} post-training for mathematical reasoning.
\newblock In \emph{Proceedings of the 64th Annual Meeting of the Association for Computational Linguistics (Volume 1: Long Papers)}, pages 33081--33106. Association for Computational Linguistics, 2026.
\newblock \doi{10.18653/v1/2026.acl-long.1528}.
\newblock URL \url{https://aclanthology.org/2026.acl-long.1528/}.

\bibitem[Domingo-Enrich et~al.()Domingo-Enrich, Drozdzal, Karrer, and Chen]{domingoadjoint}
Carles Domingo-Enrich, Michal Drozdzal, Brian Karrer, and Ricky~TQ Chen.
\newblock Adjoint matching: Fine-tuning flow and diffusion generative models with memoryless stochastic optimal control.
\newblock In \emph{The Thirteenth International Conference on Learning Representations}.

\bibitem[Dong et~al.(2023)Dong, Xiong, Goyal, Zhang, Chow, Pan, Diao, Zhang, Shum, and Zhang]{dong2023raft}
Hanze Dong, Wei Xiong, Deepanshu Goyal, Yihan Zhang, Winnie Chow, Rui Pan, Shizhe Diao, Jipeng Zhang, Kashun Shum, and Tong Zhang.
\newblock {RAFT}: Reward ranked finetuning for generative foundation model alignment.
\newblock \emph{Transactions on Machine Learning Research}, 2023.
\newblock URL \url{https://arxiv.org/abs/2304.06767}.

\bibitem[Esser et~al.(2024)Esser, Kulal, Blattmann, Entezari, M{\"u}ller, Saini, Levi, Lorenz, Sauer, Boesel, et~al.]{esser2024scaling}
Patrick Esser, Sumith Kulal, Andreas Blattmann, Rahim Entezari, Jonas M{\"u}ller, Harry Saini, Yam Levi, Dominik Lorenz, Axel Sauer, Frederic Boesel, et~al.
\newblock Scaling rectified flow transformers for high-resolution image synthesis.
\newblock In \emph{Forty-first international conference on machine learning}, 2024.

\bibitem[Fan et~al.(2023)Fan, Watkins, Du, Liu, Ryu, Boutilier, Abbeel, Ghavamzadeh, Lee, and Lee]{fan2023dpok}
Ying Fan, Olivia Watkins, Yuqing Du, Hao Liu, Moonkyung Ryu, Craig Boutilier, Pieter Abbeel, Mohammad Ghavamzadeh, Kangwook Lee, and Kimin Lee.
\newblock {DPOK}: Reinforcement learning for fine-tuning text-to-image diffusion models.
\newblock In \emph{Advances in Neural Information Processing Systems}, volume~36, 2023.
\newblock URL \url{https://arxiv.org/abs/2305.16381}.

\bibitem[Farghly et~al.(2025)Farghly, Potaptchik, Howard, Deligiannidis, and Pidstrigach]{farghly2025diffusion}
Tyler Farghly, Peter Potaptchik, Samuel Howard, George Deligiannidis, and Jakiw Pidstrigach.
\newblock Diffusion models and the manifold hypothesis: Log-domain smoothing is geometry adaptive.
\newblock \emph{arXiv preprint arXiv:2510.02305}, 2025.

\bibitem[Fefferman et~al.(2016)Fefferman, Mitter, and Narayanan]{fefferman2016testing}
Charles Fefferman, Sanjoy Mitter, and Hariharan Narayanan.
\newblock Testing the manifold hypothesis.
\newblock \emph{Journal of the American Mathematical Society}, 29\penalty0 (4):\penalty0 983--1049, 2016.

\bibitem[Fu et~al.(2025)Fu, Wang, Cui, Chen, Xu, Luo, and Zhang]{fu2025diffusionsdpo}
Minghao Fu, Guo-Hua Wang, Tianyu Cui, Qing-Guo Chen, Zhao Xu, Weihua Luo, and Kaifu Zhang.
\newblock Diffusion-{SDPO}: Safeguarded direct preference optimization for diffusion models.
\newblock \emph{arXiv preprint arXiv:2511.03317}, 2025.
\newblock URL \url{https://arxiv.org/abs/2511.03317}.

\bibitem[Furuta et~al.(2024)Furuta, Zen, Schuurmans, Faust, Matsuo, Liang, and Yang]{furuta2024dynamic}
Hiroki Furuta, Heiga Zen, Dale Schuurmans, Aleksandra Faust, Yutaka Matsuo, Percy Liang, and Sherry Yang.
\newblock Improving dynamic object interactions in text-to-video generation with {AI} feedback.
\newblock \emph{arXiv preprint arXiv:2412.02617}, 2024.
\newblock \doi{10.48550/arXiv.2412.02617}.
\newblock URL \url{https://arxiv.org/abs/2412.02617}.

\bibitem[Ghosh et~al.(2023)Ghosh, Hajishirzi, and Schmidt]{ghosh2023geneval}
Dhruba Ghosh, Hannaneh Hajishirzi, and Ludwig Schmidt.
\newblock Geneval: An object-focused framework for evaluating text-to-image alignment.
\newblock \emph{Advances in Neural Information Processing Systems}, 36:\penalty0 52132--52152, 2023.

\bibitem[Guo et~al.(2025)Guo, Cui, Bo, and Huang]{guo2025shortft}
Xiefan Guo, Miaomiao Cui, Liefeng Bo, and Di~Huang.
\newblock {ShortFT}: Diffusion model alignment via shortcut-based fine-tuning.
\newblock In \emph{Proceedings of the IEEE/CVF International Conference on Computer Vision}, pages 678--687, 2025.
\newblock \doi{10.1109/ICCV51701.2025.00071}.
\newblock URL \url{https://arxiv.org/abs/2507.22604}.

\bibitem[He et~al.(2025{\natexlab{a}})He, Liang, Wang, Wan, Zhang, Gai, and Pan]{he2025evosearch}
Haoran He, Jiajun Liang, Xintao Wang, Pengfei Wan, Di~Zhang, Kun Gai, and Ling Pan.
\newblock Scaling image and video generation via test-time evolutionary search.
\newblock \emph{arXiv preprint arXiv:2505.17618}, 2025{\natexlab{a}}.
\newblock \doi{10.48550/arXiv.2505.17618}.
\newblock URL \url{https://arxiv.org/abs/2505.17618}.

\bibitem[He et~al.(2025{\natexlab{b}})He, Fu, Zhao, Li, Yang, Yin, Rao, and Zhang]{he2025tempflow}
Xiaoxuan He, Siming Fu, Yuke Zhao, Wanli Li, Jian Yang, Dacheng Yin, Fengyun Rao, and Bo~Zhang.
\newblock Tempflow-grpo: When timing matters for grpo in flow models.
\newblock \emph{arXiv preprint arXiv:2508.04324}, 2025{\natexlab{b}}.

\bibitem[He et~al.(2024)He, Jiang, Zhang, Ku, Soni, Siu, Chen, Chandra, Jiang, Arulraj, Wang, Do, Ni, Lyu, Narsupalli, Fan, Lyu, Lin, and Chen]{he2024videoscore}
Xuan He, Dongfu Jiang, Ge~Zhang, Max Ku, Achint Soni, Sherman Siu, Haonan Chen, Abhranil Chandra, Ziyan Jiang, Aaran Arulraj, Kai Wang, Quy~Duc Do, Yuansheng Ni, Bohan Lyu, Yaswanth Narsupalli, Rongqi Fan, Zhiheng Lyu, Yuchen Lin, and Wenhu Chen.
\newblock {VideoScore}: Building automatic metrics to simulate fine-grained human feedback for video generation.
\newblock \emph{arXiv preprint arXiv:2406.15252}, 2024.
\newblock \doi{10.48550/arXiv.2406.15252}.
\newblock URL \url{https://arxiv.org/abs/2406.15252}.

\bibitem[He et~al.(2023)He, Murata, Lai, Takida, Uesaka, Kim, Liao, Mitsufuji, Kolter, Salakhutdinov, and Ermon]{he2023mpgd}
Yutong He, Naoki Murata, Chieh-Hsin Lai, Yuhta Takida, Toshimitsu Uesaka, Dongjun Kim, Wei-Hsiang Liao, Yuki Mitsufuji, J.~Zico Kolter, Ruslan Salakhutdinov, and Stefano Ermon.
\newblock Manifold preserving guided diffusion.
\newblock \emph{arXiv preprint arXiv:2311.16424}, 2023.
\newblock URL \url{https://arxiv.org/abs/2311.16424}.

\bibitem[Ho et~al.(2020)Ho, Jain, and Abbeel]{ho2020denoising}
Jonathan Ho, Ajay Jain, and Pieter Abbeel.
\newblock Denoising diffusion probabilistic models.
\newblock \emph{Advances in neural information processing systems}, 33:\penalty0 6840--6851, 2020.

\bibitem[Hong et~al.(2026)Hong, Paul, Lee, Rasul, Thorne, and Jeong]{hong2026mapo}
Jiwoo Hong, Sayak Paul, Noah Lee, Kashif Rasul, James Thorne, and Jongheon Jeong.
\newblock Margin-aware preference optimization for aligning diffusion models without reference.
\newblock \emph{Proceedings of the AAAI Conference on Artificial Intelligence}, 40\penalty0 (6):\penalty0 4744--4752, 2026.
\newblock \doi{10.1609/aaai.v40i6.42476}.
\newblock URL \url{https://arxiv.org/abs/2406.06424}.

\bibitem[Huang et~al.(2024)Huang, Zhan, Xie, Lee, Sun, Krishnamurthy, and Foster]{huang2024correcting}
Audrey Huang, Wenhao Zhan, Tengyang Xie, Jason~D Lee, Wen Sun, Akshay Krishnamurthy, and Dylan~J Foster.
\newblock Correcting the mythos of kl-regularization: Direct alignment without overoptimization via chi-squared preference optimization.
\newblock \emph{arXiv preprint arXiv:2407.13399}, 2024.

\bibitem[Kang et~al.(2025)Kang, Lim, Baek, and Shim]{kang2025rethinkingdpo}
Junyong Kang, Seohyun Lim, Kyungjune Baek, and Hyunjung Shim.
\newblock Rethinking direct preference optimization in diffusion models.
\newblock \emph{arXiv preprint arXiv:2505.18736}, 2025.
\newblock URL \url{https://arxiv.org/abs/2505.18736}.

\bibitem[Labs(2025)]{flux-2-2025}
Black~Forest Labs.
\newblock {FLUX.2: Frontier Visual Intelligence}.
\newblock \url{https://bfl.ai/blog/flux-2}, 2025.

\bibitem[Lee et~al.(2023)Lee, Liu, Ryu, Watkins, Du, Boutilier, Abbeel, Ghavamzadeh, and Gu]{lee2023aligning}
Kimin Lee, Hao Liu, Moonkyung Ryu, Olivia Watkins, Yuqing Du, Craig Boutilier, Pieter Abbeel, Mohammad Ghavamzadeh, and Shixiang~Shane Gu.
\newblock Aligning text-to-image models using human feedback.
\newblock \emph{arXiv preprint arXiv:2302.12192}, 2023.
\newblock \doi{10.48550/arXiv.2302.12192}.
\newblock URL \url{https://arxiv.org/abs/2302.12192}.

\bibitem[Lei et~al.(2020)Lei, An, Guo, Su, Liu, Luo, Yau, and Gu]{lei2020geometric}
Na~Lei, Dongsheng An, Yang Guo, Kehua Su, Shixia Liu, Zhongxuan Luo, Shing-Tung Yau, and Xianfeng Gu.
\newblock A geometric understanding of deep learning.
\newblock \emph{Engineering}, 6\penalty0 (3):\penalty0 361--374, 2020.

\bibitem[Li et~al.(2025{\natexlab{a}})Li, Xu, Han, Dang, and Ermon]{li2025divergence}
Binxu Li, Minkai Xu, Jiaqi Han, Meihua Dang, and Stefano Ermon.
\newblock Divergence minimization preference optimization for diffusion model alignment.
\newblock \emph{arXiv preprint arXiv:2507.07510}, 2025{\natexlab{a}}.

\bibitem[Li et~al.(2024{\natexlab{a}})Li, Feng, Fu, Wang, Basu, Chen, and Wang]{li2024t2vturbo}
Jiachen Li, Weixi Feng, Tsu-Jui Fu, Xinyi Wang, Sugato Basu, Wenhu Chen, and William~Yang Wang.
\newblock {T2V-Turbo}: Breaking the quality bottleneck of video consistency model with mixed reward feedback.
\newblock \emph{arXiv preprint arXiv:2405.18750}, 2024{\natexlab{a}}.
\newblock \doi{10.48550/arXiv.2405.18750}.
\newblock URL \url{https://arxiv.org/abs/2405.18750}.

\bibitem[Li et~al.(2025{\natexlab{b}})Li, Long, Zheng, Gao, Piramuthu, Chen, and Wang]{li2025t2vturbov2}
Jiachen Li, Qian Long, Jian Zheng, Xiaofeng Gao, Robinson Piramuthu, Wenhu Chen, and William~Yang Wang.
\newblock {T2V-Turbo-v2}: Enhancing video model post-training through data, reward, and conditional guidance design.
\newblock In \emph{The Thirteenth International Conference on Learning Representations}, 2025{\natexlab{b}}.
\newblock URL \url{https://arxiv.org/abs/2410.05677}.

\bibitem[Li et~al.(2025{\natexlab{c}})Li, Cui, Huang, Ma, Fan, Yang, and Zhong]{li2025mixgrpo}
Junzhe Li, Yutao Cui, Tao Huang, Yinping Ma, Chun Fan, Miles Yang, and Zhao Zhong.
\newblock Mixgrpo: Unlocking flow-based grpo efficiency with mixed ode-sde.
\newblock \emph{arXiv preprint arXiv:2507.21802}, 2025{\natexlab{c}}.

\bibitem[Li et~al.(2026)Li, Xu, Tseng, Lu, Liu, and Lan]{li2026lineardpo}
Kesong Li, Yixuan Xu, Kuo-kun Tseng, Weiyi Lu, Kan Liu, and Tao Lan.
\newblock Linear-{DPO}: Linear direct preference optimization for diffusion and flow-matching generative models.
\newblock \emph{arXiv preprint arXiv:2605.21123}, 2026.
\newblock URL \url{https://arxiv.org/abs/2605.21123}.

\bibitem[Li et~al.(2024{\natexlab{b}})Li, Kallidromitis, Gokul, Kato, and Kozuka]{li2024diffusionkto}
Shufan Li, Konstantinos Kallidromitis, Akash Gokul, Yusuke Kato, and Kazuki Kozuka.
\newblock Aligning diffusion models by optimizing human utility.
\newblock In \emph{Advances in Neural Information Processing Systems}, volume~37, pages 24897--24925, 2024{\natexlab{b}}.
\newblock \doi{10.52202/079017-0785}.
\newblock URL \url{https://arxiv.org/abs/2404.04465}.

\bibitem[Liang et~al.(2025)Liang, Yuan, Gu, Chen, Hang, Cheng, Li, and Zheng]{liang2025spo}
Zhanhao Liang, Yuhui Yuan, Shuyang Gu, Bohan Chen, Tiankai Hang, Mingxi Cheng, Ji~Li, and Liang Zheng.
\newblock Aesthetic post-training diffusion models from generic preferences with step-by-step preference optimization.
\newblock In \emph{Proceedings of the IEEE/CVF Conference on Computer Vision and Pattern Recognition}, pages 13199--13208, 2025.
\newblock URL \url{https://arxiv.org/abs/2406.04314}.

\bibitem[Lipman et~al.(2022)Lipman, Chen, Ben-Hamu, Nickel, and Le]{lipman2022flow}
Yaron Lipman, Ricky~TQ Chen, Heli Ben-Hamu, Maximilian Nickel, and Matt Le.
\newblock Flow matching for generative modeling.
\newblock \emph{arXiv preprint arXiv:2210.02747}, 2022.

\bibitem[Liu et~al.(2025{\natexlab{a}})Liu, Liu, Liang, Li, Liu, Wang, Wan, Zhang, and Ouyang]{liu2025flow}
Jie Liu, Gongye Liu, Jiajun Liang, Yangguang Li, Jiaheng Liu, Xintao Wang, Pengfei Wan, Di~Zhang, and Wanli Ouyang.
\newblock Flow-grpo: Training flow matching models via online rl.
\newblock \emph{arXiv preprint arXiv:2505.05470}, 2025{\natexlab{a}}.

\bibitem[Liu et~al.(2025{\natexlab{b}})Liu, Liu, Liang, Yuan, Liu, Zheng, Wu, Wang, Xia, Wang, et~al.]{liu2025improving}
Jie Liu, Gongye Liu, Jiajun Liang, Ziyang Yuan, Xiaokun Liu, Mingwu Zheng, Xiele Wu, Qiulin Wang, Menghan Xia, Xintao Wang, et~al.
\newblock Improving video generation with human feedback.
\newblock \emph{arXiv preprint arXiv:2501.13918}, 2025{\natexlab{b}}.

\bibitem[Liu et~al.(2025{\natexlab{c}})Liu, Wu, Zheng, Wei, He, Pi, and Chen]{liu2025videodpo}
Runtao Liu, Haoyu Wu, Ziqiang Zheng, Chen Wei, Yingqing He, Renjie Pi, and Qifeng Chen.
\newblock {VideoDPO}: Omni-preference alignment for video diffusion generation.
\newblock In \emph{Proceedings of the IEEE/CVF Conference on Computer Vision and Pattern Recognition}, pages 8009--8019, 2025{\natexlab{c}}.
\newblock URL \url{https://arxiv.org/abs/2412.14167}.

\bibitem[Liu et~al.(2022)Liu, Gong, and Liu]{liu2022flow}
Xingchao Liu, Chengyue Gong, and Qiang Liu.
\newblock Flow straight and fast: Learning to generate and transfer data with rectified flow.
\newblock \emph{arXiv preprint arXiv:2209.03003}, 2022.

\bibitem[Lu et~al.(2025)Lu, Wang, Cao, Xu, and Zhang]{lu2025smpo}
Yunhong Lu, Qichao Wang, Hengyuan Cao, Xiaoyin Xu, and Min Zhang.
\newblock Smoothed preference optimization via {R}e{N}oise inversion for aligning diffusion models with varied human preferences.
\newblock In \emph{Proceedings of the 42nd International Conference on Machine Learning}, volume 267 of \emph{Proceedings of Machine Learning Research}, pages 40709--40725. PMLR, 2025.
\newblock URL \url{https://proceedings.mlr.press/v267/lu25l.html}.

\bibitem[Luo et~al.(2025)Luo, Hu, and Tang]{luo2025dgpo}
Yihong Luo, Tianyang Hu, and Jing Tang.
\newblock Reinforcing diffusion models by direct group preference optimization.
\newblock \emph{arXiv preprint arXiv:2510.08425}, 2025.
\newblock \doi{10.48550/arXiv.2510.08425}.
\newblock URL \url{https://arxiv.org/abs/2510.08425}.

\bibitem[Ma et~al.(2024)Ma, Goldstein, Albergo, Boffi, Vanden-Eijnden, and Xie]{ma2024sit}
Nanye Ma, Mark Goldstein, Michael~S. Albergo, Nicholas~M. Boffi, Eric Vanden-Eijnden, and Saining Xie.
\newblock {SiT}: Exploring flow and diffusion-based generative models with scalable interpolant transformers.
\newblock In \emph{European Conference on Computer Vision}, 2024.
\newblock URL \url{https://arxiv.org/abs/2401.08740}.

\bibitem[Ma et~al.(2025)Ma, Wu, Sun, and Li]{ma2025hpsv3widespectrumhumanpreference}
Yuhang Ma, Xiaoshi Wu, Keqiang Sun, and Hongsheng Li.
\newblock Hpsv3: Towards wide-spectrum human preference score, 2025.
\newblock URL \url{https://arxiv.org/abs/2508.03789}.

\bibitem[Oshima et~al.(2025)Oshima, Suzuki, Matsuo, and Furuta]{oshima2025beamsearch}
Yuta Oshima, Masahiro Suzuki, Yutaka Matsuo, and Hiroki Furuta.
\newblock Inference-time text-to-video alignment with diffusion latent beam search.
\newblock \emph{arXiv preprint arXiv:2501.19252}, 2025.
\newblock \doi{10.48550/arXiv.2501.19252}.
\newblock URL \url{https://arxiv.org/abs/2501.19252}.

\bibitem[Ouyang et~al.(2022)Ouyang, Wu, Jiang, Almeida, Wainwright, Mishkin, Zhang, Agarwal, Slama, Ray, et~al.]{ouyang2022training}
Long Ouyang, Jeffrey Wu, Xu~Jiang, Diogo Almeida, Carroll Wainwright, Pamela Mishkin, Chong Zhang, Sandhini Agarwal, Katarina Slama, Alex Ray, et~al.
\newblock Training language models to follow instructions with human feedback.
\newblock \emph{Advances in neural information processing systems}, 35:\penalty0 27730--27744, 2022.

\bibitem[Ping et~al.(2026)Ping, Jia, Luo, Qian, and Tsang]{ping2026flowfactory}
Bowen Ping, Chengyou Jia, Minnan Luo, Hangwei Qian, and Ivor Tsang.
\newblock {Flow-Factory}: A unified framework for reinforcement learning in flow-matching models.
\newblock \emph{arXiv preprint arXiv:2602.12529}, 2026.
\newblock URL \url{https://arxiv.org/abs/2602.12529}.

\bibitem[Prabhudesai et~al.(2023)Prabhudesai, Goyal, Pathak, and Fragkiadaki]{prabhudesai2023alignprop}
Mihir Prabhudesai, Anirudh Goyal, Deepak Pathak, and Katerina Fragkiadaki.
\newblock Aligning text-to-image diffusion models with reward backpropagation.
\newblock \emph{arXiv preprint arXiv:2310.03739}, 2023.
\newblock \doi{10.48550/arXiv.2310.03739}.
\newblock URL \url{https://arxiv.org/abs/2310.03739}.

\bibitem[Prabhudesai et~al.(2024)Prabhudesai, Mendonca, Qin, Fragkiadaki, and Pathak]{prabhudesai2024vader}
Mihir Prabhudesai, Russell Mendonca, Zheyang Qin, Katerina Fragkiadaki, and Deepak Pathak.
\newblock Video diffusion alignment via reward gradients.
\newblock \emph{arXiv preprint arXiv:2407.08737}, 2024.
\newblock \doi{10.48550/arXiv.2407.08737}.
\newblock URL \url{https://arxiv.org/abs/2407.08737}.

\bibitem[Rafailov et~al.(2023)Rafailov, Sharma, Mitchell, Manning, Ermon, and Finn]{rafailov2023direct}
Rafael Rafailov, Archit Sharma, Eric Mitchell, Christopher~D Manning, Stefano Ermon, and Chelsea Finn.
\newblock Direct preference optimization: Your language model is secretly a reward model.
\newblock \emph{Advances in neural information processing systems}, 36:\penalty0 53728--53741, 2023.

\bibitem[Saharia et~al.(2022)Saharia, Chan, Saxena, Li, Whang, Denton, Ghasemipour, Gontijo~Lopes, Karagol~Ayan, Salimans, et~al.]{saharia2022photorealistic}
Chitwan Saharia, William Chan, Saurabh Saxena, Lala Li, Jay Whang, Emily~L Denton, Kamyar Ghasemipour, Raphael Gontijo~Lopes, Burcu Karagol~Ayan, Tim Salimans, et~al.
\newblock Photorealistic text-to-image diffusion models with deep language understanding.
\newblock \emph{Advances in neural information processing systems}, 35:\penalty0 36479--36494, 2022.

\bibitem[Shao et~al.(2025)Shao, Xiao, Zhu, Liu, Zhai, Cao, and Zha]{shao2025vgpo}
Yawen Shao, Jie Xiao, Kai Zhu, Yu~Liu, Wei Zhai, Yang Cao, and Zheng-Jun Zha.
\newblock Anchoring values in temporal and group dimensions for flow matching model alignment.
\newblock \emph{arXiv preprint arXiv:2512.12387}, 2025.
\newblock \doi{10.48550/arXiv.2512.12387}.
\newblock URL \url{https://arxiv.org/abs/2512.12387}.

\bibitem[Sohl-Dickstein et~al.(2015)Sohl-Dickstein, Weiss, Maheswaranathan, and Ganguli]{sohl2015deep}
Jascha Sohl-Dickstein, Eric Weiss, Niru Maheswaranathan, and Surya Ganguli.
\newblock Deep unsupervised learning using nonequilibrium thermodynamics.
\newblock In \emph{International conference on machine learning}, pages 2256--2265. pmlr, 2015.

\bibitem[Song et~al.(2020)Song, Sohl-Dickstein, Kingma, Kumar, Ermon, and Poole]{song2020score}
Yang Song, Jascha Sohl-Dickstein, Diederik~P Kingma, Abhishek Kumar, Stefano Ermon, and Ben Poole.
\newblock Score-based generative modeling through stochastic differential equations.
\newblock \emph{arXiv preprint arXiv:2011.13456}, 2020.

\bibitem[Sun et~al.(2025{\natexlab{a}})Sun, Liao, Han, Bai, Gao, Fu, Shen, Wan, Yan, Zhang, et~al.]{sun2025solopo}
Huashan Sun, Shengyi Liao, Yansen Han, Yu~Bai, Yang Gao, Cheng Fu, Weizhou Shen, Fanqi Wan, Ming Yan, Ji~Zhang, et~al.
\newblock Solopo: Unlocking long-context capabilities in llms via short-to-long preference optimization.
\newblock \emph{arXiv preprint arXiv:2505.11166}, 2025{\natexlab{a}}.

\bibitem[Sun et~al.(2025{\natexlab{b}})Sun, Jiang, and Lin]{sun2025unified}
Peng Sun, Yi~Jiang, and Tao Lin.
\newblock Unified continuous generative models.
\newblock \emph{arXiv preprint arXiv:2505.07447}, 2025{\natexlab{b}}.

\bibitem[Tang et~al.(2024{\natexlab{a}})Tang, Guo, Zheng, Calandriello, Munos, Rowland, Richemond, Valko, Pires, and Piot]{tang2024generalized}
Yunhao Tang, Zhaohan~Daniel Guo, Zeyu Zheng, Daniele Calandriello, R{\'e}mi Munos, Mark Rowland, Pierre~Harvey Richemond, Michal Valko, Bernardo~{\'A}vila Pires, and Bilal Piot.
\newblock Generalized preference optimization: A unified approach to offline alignment.
\newblock \emph{arXiv preprint arXiv:2402.05749}, 2024{\natexlab{a}}.

\bibitem[Tang et~al.(2024{\natexlab{b}})Tang, Peng, Tang, Hong, Wang, and Chang]{tang2024dno}
Zhiwei Tang, Jiangweizhi Peng, Jiasheng Tang, Mingyi Hong, Fan Wang, and Tsung-Hui Chang.
\newblock Inference-time alignment of diffusion models with direct noise optimization.
\newblock \emph{arXiv preprint arXiv:2405.18881}, 2024{\natexlab{b}}.
\newblock \doi{10.48550/arXiv.2405.18881}.
\newblock URL \url{https://arxiv.org/abs/2405.18881}.

\bibitem[Wallace et~al.(2024)Wallace, Dang, Rafailov, Zhou, Lou, Purushwalkam, Ermon, Xiong, Joty, and Naik]{wallace2024diffusion}
Bram Wallace, Meihua Dang, Rafael Rafailov, Linqi Zhou, Aaron Lou, Senthil Purushwalkam, Stefano Ermon, Caiming Xiong, Shafiq Joty, and Nikhil Naik.
\newblock Diffusion model alignment using direct preference optimization.
\newblock In \emph{Proceedings of the IEEE/CVF Conference on Computer Vision and Pattern Recognition}, pages 8228--8238, 2024.

\bibitem[Wang et~al.(2024)Wang, Tan, Wang, Yang, Jin, and Li]{wang2024lift}
Yibin Wang, Zhiyu Tan, Junyan Wang, Xiaomeng Yang, Cheng Jin, and Hao Li.
\newblock {LiFT}: Leveraging human feedback for text-to-video model alignment.
\newblock \emph{arXiv preprint arXiv:2412.04814}, 2024.
\newblock \doi{10.48550/arXiv.2412.04814}.
\newblock URL \url{https://arxiv.org/abs/2412.04814}.

\bibitem[Wang et~al.(2025)Wang, Zang, Li, Jin, and Wang]{unifiedreward}
Yibin Wang, Yuhang Zang, Hao Li, Cheng Jin, and Jiaqi Wang.
\newblock Unified reward model for multimodal understanding and generation.
\newblock \emph{arXiv preprint arXiv:2503.05236}, 2025.

\bibitem[Wang et~al.(2026)Wang, Li, Qian, Tulyakov, Fu, and Kag]{wang2026diffusiondrf}
Yifan Wang, Yanyu Li, Gordon~Guocheng Qian, Sergey Tulyakov, Yun Fu, and Anil Kag.
\newblock {Diffusion-DRF}: Free, rich, and differentiable reward for video diffusion fine-tuning.
\newblock \emph{arXiv preprint arXiv:2601.04153}, 2026.
\newblock \doi{10.48550/arXiv.2601.04153}.
\newblock URL \url{https://arxiv.org/abs/2601.04153}.

\bibitem[Wu et~al.(2025)Wu, Kag, Skorokhodov, Menapace, Mirzaei, Gilitschenski, Tulyakov, and Siarohin]{wu2025densedpo}
Ziyi Wu, Anil Kag, Ivan Skorokhodov, Willi Menapace, Ashkan Mirzaei, Igor Gilitschenski, Sergey Tulyakov, and Aliaksandr Siarohin.
\newblock {DenseDPO}: Fine-grained temporal preference optimization for video diffusion models.
\newblock In \emph{Advances in Neural Information Processing Systems}, volume~38, 2025.
\newblock URL \url{https://arxiv.org/abs/2506.03517}.

\bibitem[Xiong et~al.(2025)Xiong, Yao, Xu, Pang, Wang, Sahoo, Li, Jiang, Zhang, Xiong, et~al.]{xiong2025minimalist}
Wei Xiong, Jiarui Yao, Yuhui Xu, Bo~Pang, Lei Wang, Doyen Sahoo, Junnan Li, Nan Jiang, Tong Zhang, Caiming Xiong, et~al.
\newblock A minimalist approach to llm reasoning: from rejection sampling to reinforce.
\newblock \emph{arXiv preprint arXiv:2504.11343}, 2025.

\bibitem[Xu et~al.(2024)Xu, Huang, Cheng, Yang, Xu, Wang, Duan, Yang, Jin, Li, Teng, Yang, Zheng, Liu, Zhang, Ding, Zhang, Gu, Huang, Huang, Tang, and Dong]{xu2024visionreward}
Jiazheng Xu, Yu~Huang, Jiale Cheng, Yuanming Yang, Jiajun Xu, Yuan Wang, Wenbo Duan, Shen Yang, Qunlin Jin, Shurun Li, Jiayan Teng, Zhuoyi Yang, Wendi Zheng, Xiao Liu, Dan Zhang, Ming Ding, Xiaohan Zhang, Xiaotao Gu, Shiyu Huang, Minlie Huang, Jie Tang, and Yuxiao Dong.
\newblock {VisionReward}: Fine-grained multi-dimensional human preference learning for image and video generation.
\newblock \emph{arXiv preprint arXiv:2412.21059}, 2024.
\newblock \doi{10.48550/arXiv.2412.21059}.
\newblock URL \url{https://arxiv.org/abs/2412.21059}.

\bibitem[Xue et~al.(2025)Xue, Wu, Gao, Kong, Zhu, Chen, Liu, Liu, Guo, Huang, and Luo]{xue2025dancegrpo}
Zeyue Xue, Jie Wu, Yu~Gao, Fangyuan Kong, Lingting Zhu, Mengzhao Chen, Zhiheng Liu, Wei Liu, Qiushan Guo, Weilin Huang, and Ping Luo.
\newblock {DanceGRPO}: Unleashing {GRPO} on visual generation.
\newblock \emph{arXiv preprint arXiv:2505.07818}, 2025.
\newblock \doi{10.48550/arXiv.2505.07818}.
\newblock URL \url{https://arxiv.org/abs/2505.07818}.

\bibitem[Yang et~al.(2026)Yang, Yang, Gong, Qin, Tan, and Li]{yang2026dualipo}
Xiaomeng Yang, Mengping Yang, Jia Gong, Luozheng Qin, Zhiyu Tan, and Hao Li.
\newblock {Dual-IPO}: Dual-iterative preference optimization for text-to-video generation.
\newblock In \emph{The Fourteenth International Conference on Learning Representations}, 2026.
\newblock URL \url{https://arxiv.org/abs/2502.02088}.

\bibitem[Zhang et~al.(2026)Zhang, Wu, Chen, Ji, Xiao, Huang, and Han]{zhang2026onlinevpo}
Jiacheng Zhang, Jie Wu, Weifeng Chen, Yatai Ji, Xuefeng Xiao, Weilin Huang, and Kai Han.
\newblock Align video diffusion model with online video-centric preference optimization.
\newblock In \emph{Proceedings of the IEEE/CVF Winter Conference on Applications of Computer Vision}, pages 6142--6152, 2026.
\newblock \doi{10.1109/WACV61042.2026.00594}.
\newblock URL \url{https://arxiv.org/abs/2412.15159}.

\bibitem[Zhang et~al.(2025)Zhang, Da, Ding, Yang, Jin, Li, Gao, Zhang, Xiang, and Pan]{zhang2025lpo}
Tao Zhang, Cheng Da, Kun Ding, Huan Yang, Kun Jin, Yan Li, Tingting Gao, Di~Zhang, Shiming Xiang, and Chunhong Pan.
\newblock Diffusion model as a noise-aware latent reward model for step-level preference optimization.
\newblock In \emph{Advances in Neural Information Processing Systems}, volume~38, 2025.
\newblock URL \url{https://arxiv.org/abs/2502.01051}.

\bibitem[Zheng et~al.(2025)Zheng, Chen, Ye, Wang, Zhang, Jiang, Su, Ermon, Zhu, and Liu]{zheng2025diffusionnft}
Kaiwen Zheng, Huayu Chen, Haotian Ye, Haoxiang Wang, Qinsheng Zhang, Kai Jiang, Hang Su, Stefano Ermon, Jun Zhu, and Ming-Yu Liu.
\newblock Diffusionnft: Online diffusion reinforcement with forward process.
\newblock \emph{arXiv preprint arXiv:2509.16117}, 2025.

\bibitem[Zhu et~al.(2025)Zhu, Xiao, and Honavar]{zhu2025dspo}
Huaisheng Zhu, Teng Xiao, and Vasant Honavar.
\newblock {DSPO}: Direct score preference optimization for diffusion model alignment.
\newblock In \emph{The Thirteenth International Conference on Learning Representations}, 2025.
\newblock URL \url{https://openreview.net/forum?id=xyfb9HHvMe}.

\end{thebibliography}
\bibliographystyle{plainnat}

\appendix

\newpage

\begingroup
\setlength{\parskip}{0pt}
\hypersetup{linkcolor=black}
\tableofcontents
\endgroup

\newpage

\section{Broader Impacts}
\label{sec:broader_impacts}

This work studies how to reduce reward hacking in preference optimization for continuous generative models, which could have positive impact by making aligned image generators more reliable and by encouraging evaluation beyond a single optimized reward. In particular, methods that better preserve the pretrained data manifold may reduce some forms of quality degradation.

At the same time, improving preference optimization for image generation can also strengthen systems that may be misused to produce deceptive or harmful synthetic media. Better alignment to OCR-oriented or human-preference rewards does not by itself guarantee fairness, safety, or robustness to adversarial prompts, and it could be used to improve misuse-oriented generation quality as well as benign applications. For this reason, we view the method as a technical contribution for controlled offline research settings rather than a claim that preference-tuned image generators are safe for unrestricted deployment.

\section{Related Work}
\label{sec:related}

\paragraph{Continuous-Time Generative Modeling.}
Diffusion and score models learn iterative stochastic denoising~\citep{sohl2015deep,ho2020denoising,song2020score}, whereas Flow Matching (FM) and related transport formulations regress continuous velocity fields~\citep{lipman2022flow,albergo2022building,chen2018neural,sun2025unified}. Stochastic interpolants unify deterministic flows and diffusions, SiT scales this view, and Rectified Flow emphasizes straighter, efficient paths~\citep{albergo2025stochastic,ma2024sit,liu2022flow,esser2024scaling}. We study not a new transport model, but whether preference updates preserve its pretrained terminal support.

\paragraph{Offline Preference Optimization.}
Offline methods learn from fixed preference data rather than collecting rewards during training. DPO provides the canonical alternative to RLHF~\citep{rafailov2023direct,ouyang2022training}, while DiffusionDPO and FlowDPO replace likelihood ratios with denoising- or flow-matching-error surrogates~\citep{wallace2024diffusion,liu2025improving}. Diffusion-KTO, MaPO, SPO, and latent preference optimization relax requirements on paired labels, reference models, or uniform timestep supervision~\citep{li2024diffusionkto,hong2026mapo,liang2025spo,zhang2025lpo}; DSPO, SmPO-Diffusion, and Linear-DPO instead address score mismatch, heterogeneous preferences, or utility saturation~\citep{zhu2025dspo,lu2025smpo,li2026lineardpo}. VideoDPO extends sequence-level preference learning to clips, whereas DenseDPO introduces temporally aligned segment-level pairs~\citep{liu2025videodpo,wu2025densedpo}. These methods improve offline objectives or supervision granularity, but do not directly characterize terminal-support preservation after the update.

\paragraph{Reward Models and Reward-Based Fine-Tuning.}
Reward-based methods first construct a scalar feedback signal and then optimize against it. VideoScore and VisionReward learn fine-grained, multidimensional video rewards~\citep{he2024videoscore,xu2024visionreward}; LiFT incorporates rationale-annotated human feedback, while AI feedback can target dynamic object interactions~\citep{wang2024lift,furuta2024dynamic}. Given such feedback, reward-weighted or reward-ranked fine-tuning and policy gradients optimize generated samples directly~\citep{lee2023aligning,dong2023raft,black2024ddpo,fan2023dpok}. DRaFT, AlignProp, and ShortFT instead backpropagate differentiable rewards through full, truncated, or shortened sampling chains~\citep{clark2024draft,prabhudesai2023alignprop,guo2025shortft}; T2V-Turbo variants inject rewards into consistency distillation~\citep{li2024t2vturbo,li2025t2vturbov2}, while VADER and Diffusion-DRF propagate dense video-reward feedback~\citep{prabhudesai2024vader,wang2026diffusiondrf}. This line focuses on reward construction and credit propagation rather than fixed-pair preference objectives.

\paragraph{Online and Inference-Time Alignment.}
Online alignment refreshes preferences, rewards, or samples during optimization. Dual-IPO alternates reward-model and generator updates, whereas OnlineVPO constructs video-centric preferences online~\citep{yang2026dualipo,zhang2026onlinevpo}. Divergence objectives modify the alignment geometry~\citep{li2025divergence}, and Adjoint Matching, Flow-GRPO, TempFlow-GRPO, MixGRPO, DanceGRPO, and DiffusionNFT provide online updates for diffusion or flow models~\citep{domingoadjoint,liu2025flow,he2025tempflow,li2025mixgrpo,xue2025dancegrpo,zheng2025diffusionnft}. DGPO learns from group preferences with deterministic ODE sampling, while VGPO addresses temporal credit and vanishing group-relative rewards~\citep{luo2025dgpo,shao2025vgpo}. In contrast, Direct Noise Optimization, diffusion latent beam search, and EvoSearch steer trajectories at inference time without parameter updates~\citep{tang2024dno,oshima2025beamsearch,he2025evosearch}. Flow-Factory supplies modular infrastructure for these training regimes~\citep{ping2026flowfactory}. Both online adaptation and inference-time search are orthogonal to our fixed-dataset support-preservation question.

\paragraph{Manifold Preservation and Stabilization.}
Stabilization mechanisms address fine-tuning drift more directly. Reference regularization and timestep-aware training constrain deviations, while Diffusion-SDPO protects preferred reconstruction~\citep{kang2025rethinkingdpo,fu2025diffusionsdpo}. Complementary evidence from LLM post-training suggests that a sufficiently trained SFT foundation can improve subsequent RL, whereas severe SFT overfitting reduces optimization plasticity~\citep{ding2026rethinking}. MPGD instead imposes an autoencoder manifold during training-free guidance~\citep{he2023mpgd}. Motivated by the manifold hypothesis~\citep{fefferman2016testing,lei2020geometric} and geometry-adaptive diffusion smoothing~\citep{farghly2025diffusion}, \method anchors the winner in an offline flow-preference objective and provides a pointwise, rather than distribution-level, guarantee.

\section{Theoretical Results and Proof}

\subsection{Proof of Theorem~\ref{thm:linear_fm_terminal_manifold}}
\label{app:proof_linear_fm_terminal_manifold}

\begin{proof}
Let $\mathcal M_{\mathrm{data}}:=\operatorname{supp}(p_0)$.
Fix a coupling $\gamma$ of $p_0$ and $p_1$, which is the joint law used to sample $(\mathbf{x}_0,\mathbf{x}_1)$ in the Flow Matching objective.
For the linear interpolation
\[
\psi_t(\mathbf{x}_0, \mathbf{x}_1)=(1-t) \cdot \mathbf{x}_0 + t \cdot \mathbf{x}_1,
\]
the path velocity is
\[
\partial_t \psi_t(\mathbf{x}_0, \mathbf{x}_1) = \mathbf{x}_1 - \mathbf{x}_0.
\]
Define
\[
X_t:=\psi_t(\mathbf{x}_0, \mathbf{x}_1),\quad U_t := \mathbf{x}_1-\mathbf{x}_0,
\quad (\mathbf{x}_0, \mathbf{x}_1) \sim \gamma.
\]
Then the Flow Matching objective can be written as
\[
\mathcal L_{\mathrm{FM}}(v)
=
\int_0^1 \mathbb E\!\left[\|v(X_t,t)-U_t\|^2\right]dt.
\]

By the standard $L^2$ projection argument, any global minimizer $v^\star$ satisfies
\[
v^\star(x,t)=\mathbb E[U_t\mid X_t=x],
\quad \rho_t\text{-a.e. }x,
\]
that is,
\[
v^\star(\mathbf{x},t)=\mathbb E[\mathbf{x}_1 - \mathbf{x}_0\mid (1-t) \cdot \mathbf{x}_0 + t \cdot \mathbf{x}_1 = \mathbf{x}].
\]
This is exactly the conditional mean velocity field of the interpolating family, where $\rho_t$ denotes the law of $X_t$.
Since $X_t$ is generated by the linear interpolation, the family $\{\rho_t\}_{t\in[0,1]}$ satisfies the continuity equation
\[
\partial_t \rho_t + \nabla\cdot(\rho_t v^\star_t)=0
\]
in the weak sense, where $v^\star_t(\cdot):=v^\star(\cdot,t)$.
This equality is proved in the following. For any smooth compactly supported test function $\varphi$,
\[
\frac{d}{dt}\mathbb E[\varphi(X_t)]
=
\mathbb E[\nabla \varphi(X_t)\cdot \partial_t X_t]
=
\mathbb E[\nabla \varphi(X_t)\cdot U_t].
\]
Using conditional expectation with respect to $X_t$,
\[
\mathbb E[\nabla \varphi(X_t)\cdot U_t]
=
\mathbb E\!\left[
\nabla \varphi(X_t)\cdot \mathbb E[U_t\mid X_t]
\right]
=
\mathbb E[\nabla \varphi(X_t)\cdot v^\star(X_t,t)].
\]
Hence
\[
\frac{d}{dt}\int_{\mathbb R^d}\varphi(x)\,\rho_t(dx)
=
\int_{\mathbb R^d}\nabla\varphi(x)\cdot v^\star(x,t)\,\rho_t(dx),
\]
which is the weak form of
\[
\partial_t \rho_t + \nabla\cdot(\rho_t v^\star_t)=0.
\]

On the other hand, if we let
\[
\mu_t := (\Phi^{v^\star}_{1\to t})_\# p_1,
\]
then the curve $\{\mu_t\}_{t\in[0,1]}$ also satisfies
\[
\partial_t \mu_t + \nabla\cdot(\mu_t v^\star_t)=0
\]
with initial condition
\[
\mu_1 = p_1.
\]
But from the definition of the interpolation,
\[
\rho_1 = (\psi_1)_\#\gamma = (x_1)_\#\gamma = p_1.
\]
Therefore both $\{\rho_t\}$ and $\{\mu_t\}$ solve the same continuity equation with the same initial condition at $t=1$.
By uniqueness of weak solutions to this continuity equation under the stated regularity assumptions, the transported marginal must coincide with the interpolation marginal:
\[
\mu_t=\rho_t,\qquad \forall t\in[0,1].
\]
That is,
\[
(\Phi^{v^\star}_{1\to t})_\# p_1 = \rho_t.
\]

Evaluating at $t=0$ gives
\[
(\Phi^{v^\star}_{1\to 0})_\# p_1 = \rho_0.
\]
Since
\[
\rho_0 = (\psi_0)_\#\gamma = (x_0)_\#\gamma = p_0,
\]
we conclude
\[
(\Phi^{v^\star}_{1\to 0})_\# p_1 = p_0.
\]
Therefore,
\[
\operatorname{supp}\!\left((\Phi^{v^\star}_{1\to 0})_\# p_1\right)=\operatorname{supp}(p_0)=\mathcal M_{\mathrm{data}}.
\]
\end{proof}

\subsection{Proof of Theorem~\ref{thm:flowdpo_drift_main}}
\label{app:proof_flowdpo_terminal_normal_drift}

\begin{lemma}[On-manifold displacement has only second-order normal component]
\label{lem:normal_second_order}
Let $\mathcal M \subset \mathbb R^d$ be a $C^2$ embedded submanifold, and let $x \in \mathcal M$.
Then there exist a neighborhood $U$ of $x$ and a constant $C>0$ such that for every $y\in \mathcal M\cap U$,
\[
\bigl\|\Pi_{N_x\mathcal M}(y-x)\bigr\|
\le
C\|y-x\|^2.
\]
Equivalently,
\[
\bigl\|\Pi_{N_x\mathcal M}(y-x)\bigr\|
=
O(\|y-x\|^2)
\qquad\text{as } y\to x,\ y\in\mathcal M.
\]
\end{lemma}

\begin{proof}
Since $\mathcal M$ is a $C^2$ embedded submanifold, after a translation and an orthogonal change of coordinates, we may assume
\[
x=0,\qquad
T_x\mathcal M=\mathbb R^m\times\{0\}\subset \mathbb R^m\times\mathbb R^{d-m}.
\]
Then, in a neighborhood of $x$, the manifold can be written as the graph
\[
\mathcal M\cap U
=
\{(u,g(u)): u\in V\},
\]
where $g:V\subset\mathbb R^m\to\mathbb R^{d-m}$ is $C^2$ and satisfies
\[
g(0)=0,\qquad Dg(0)=0.
\]
Hence, by Taylor's theorem,
\[
\|g(u)\|\le C\|u\|^2
\]
for all $u$ sufficiently close to $0$.

Now let $y=(u,g(u))\in \mathcal M\cap U$. Since $N_x\mathcal M=\{0\}\times\mathbb R^{d-m}$ in these coordinates, we have
\[
\Pi_{N_x\mathcal M}(y-x)=\Pi_{N_x\mathcal M}(u,g(u))=(0,g(u)),
\]
and therefore
\[
\bigl\|\Pi_{N_x\mathcal M}(y-x)\bigr\|=\|g(u)\|\le C\|u\|^2.
\]
Since
\[
\|y-x\|=\|(u,g(u))\|\ge \|u\|,
\]
it follows that
\[
\bigl\|\Pi_{N_x\mathcal M}(y-x)\bigr\|
\le
C\|y-x\|^2.
\]
This proves the claim.
\end{proof}


In the following, we prove the \thmref{thm:flowdpo_drift_main}.


\begin{proof}
Let
\[
g:=\nabla_\theta \mathcal L(\theta_0),
\qquad
\theta_1=\theta_0-\alpha g,
\qquad
x_0^\star:=F(\theta_0,\mathbf x_1)\in\mathcal M_0.
\]
Since $F(\theta,\mathbf x_1)$ is differentiable with respect to $\theta$ at $\theta_0$,
\[
F(\theta_1,\mathbf x_1)-x_0^\star
=
-\alpha\,D_\theta F(\theta_0,\mathbf x_1)[g]
+
o(\alpha).
\]
Projecting onto the normal space $N_{x_0^\star}\mathcal M_0$ yields
\[
\Pi_{N_{x_0^\star}\mathcal M_0}\bigl(F(\theta_1,\mathbf x_1)-x_0^\star\bigr)
=
-\alpha\,
\Pi_{N_{x_0^\star}\mathcal M_0}
D_\theta F(\theta_0,\mathbf x_1)[g]
+
o(\alpha).
\]
By assumption,
\[
\Pi_{N_{x_0^\star}\mathcal M_0}
D_\theta F(\theta_0,\mathbf x_1)[g]\neq 0,
\]
so there exists $c>0$ such that for all sufficiently small $\alpha>0$,
\[
\Bigl\|
\Pi_{N_{x_0^\star}\mathcal M_0}\bigl(F(\theta_1,\mathbf x_1)-x_0^\star\bigr)
\Bigr\|
\ge c\alpha.
\]
Suppose, for contradiction, that there exists a sequence $\alpha_n\downarrow 0$ such that
\[
F(\theta_0-\alpha_n g,\mathbf x_1)\in\mathcal M_0.
\]
Since $F(\theta_0-\alpha_n g,\mathbf x_1)\to x_0^\star$, \lemref{lem:normal_second_order} implies
\[
\Bigl\|
\Pi_{N_{x_0^\star}\mathcal M_0}
\bigl(F(\theta_0-\alpha_n g,\mathbf x_1)-x_0^\star\bigr)
\Bigr\|
\le
C\bigl\|F(\theta_0-\alpha_n g,\mathbf x_1)-x_0^\star\bigr\|^2.
\]
But the differentiability of $F$ also gives
\[
\bigl\|F(\theta_0-\alpha_n g,\mathbf x_1)-x_0^\star\bigr\|=O(\alpha_n),
\]
hence
\[
\Bigl\|
\Pi_{N_{x_0^\star}\mathcal M_0}
\bigl(F(\theta_0-\alpha_n g,\mathbf x_1)-x_0^\star\bigr)
\Bigr\|
=
O(\alpha_n^2).
\]
This contradicts the lower bound
\[
\Bigl\|
\Pi_{N_{x_0^\star}\mathcal M_0}\bigl(F(\theta_0-\alpha_n g,\mathbf x_1)-x_0^\star\bigr)
\Bigr\|
\ge c\alpha_n
\]
for all sufficiently large $n$.
Therefore, there exists $\alpha_0>0$ such that for all $\alpha\in(0,\alpha_0)$,
\[
F(\theta_1,\mathbf x_1)\notin\mathcal M_0.
\]
\end{proof}

\subsection{Proof of Theorem~\ref{thm:thermo_to_rft}}
\label{app:proof_thermo_to_rft}

\begin{proof}
Because
\begin{equation}
    \lim_{\tau \downarrow 0} \tau \cdot \log\left(1 + \exp(\frac{a}{\tau}) + \exp(\frac{b}{\tau})\right) = \max\{0, a, b\}
\end{equation},
we have
\begin{equation}
\lim_{\tau \downarrow 0} g_\tau(\theta)
= t^2 \cdot \max\!\left\{
0,\,
\Delta_\theta^w-\Delta_\theta^l,\,
\ell_\theta^w
\right\}.
\end{equation}
If $\Delta_\theta^w-\Delta_\theta^l \leq \ell_\theta^w$, then
\begin{align}
\lim_{\tau \downarrow 0} g_\tau(\theta)
&= t^2 \cdot \ell_\theta^w\\
&= t^2 \cdot \| v_\theta(\mathbf x_t^w, t) - (\epsilon - \mathbf x_0^w)\|^2 \\
&= \| t \cdot v_\theta(\mathbf x_t^w, t) - t \cdot (\epsilon - \mathbf x_0^w)\|^2 \\
&= \| ((1-t)\cdot \mathbf x_0^w + t \cdot \epsilon) - t \cdot v_\theta(\mathbf x_t^w, t) - \mathbf x_0^w \|^2 \\
&= \|\mathbf x_t^w - t \cdot v_\theta(\mathbf x_t^w, t) - \mathbf x_0^w \|^2 \\
&= \|\tilde{\mathbf x}_0^w - \mathbf x_0^w\|^2.
\end{align}
\end{proof}

\subsection{Proof of Theorem~\ref{thm:thermo_vs_flowdpo}}
\label{app:proof_thm_thermo_vs_flowdpo}
\begin{proof}
    \begin{align}
        g_\tau(\theta)
        &=
        t^2 \cdot \tau \cdot \log\!\left(
        {
        1+
        \exp\!\left(\frac{\Delta_\theta^w-\Delta_\theta^l}{\tau}\right)
        +
        \exp\!\left(\frac{\ell_\theta^w}{\tau}\right)
        }
        \right)\\
        &=
        t^2 \cdot \tau \cdot \log\!\left(
        {
        1+
        \exp\!\left(\frac{\Delta_\theta^w-\Delta_\theta^l}{\tau}\right)
        }
        \right)\nonumber\\
        &\quad\ +
        t^2 \cdot \tau \cdot \log\!\left(
        \frac{
        1+
        \exp\!\left(\frac{\Delta_\theta^w-\Delta_\theta^l}{\tau}\right)
        +
        \exp\!\left(\frac{\ell_\theta^w}{\tau}\right)
        }{
        1+
        \exp\!\left(\frac{\Delta_\theta^w-\Delta_\theta^l}{\tau}\right)
        }
        \right).
    \end{align}
\end{proof}

\subsection{Proof of Theorem~\ref{thm:thermo_exp_drift_suppression}}
\label{app:proof_thermo_exp_drift_suppression}
\begin{proof}
Let $\mathcal M_{\mathrm{data}}:=\operatorname{supp}(p_0)$.
Now, we prove the inequality:
\begin{align*}
    g_\tau(\theta) 
    &=
    t^2 \cdot \tau \cdot \log\!\left(
    1+
    \exp\!\left(\frac{\Delta_\theta^w-\Delta_\theta^l}{\tau}\right)
    +
    \exp\!\left(\frac{\ell_\theta^w}{\tau}\right)
    \right)\\
    &\geq t^2 \cdot \tau \cdot \frac{\ell_\theta^w}{\tau} \\
    &=  t^2 \cdot \ell_\theta^w \\
    &= t^2 \cdot \| v_\theta(\mathbf x_t^w, t) - (\epsilon - \mathbf x_0^w)\|^2 \\
    &= \| (\mathbf x_t^w - t \cdot v_\theta(\mathbf x_t^w, t)) - \mathbf x_0^w\|^2 \qquad \text{(where $\mathbf x_t^w = (1 - t) \cdot \mathbf x_0^w + t \cdot \epsilon$)}\\
    &= \|\tilde{\mathbf x}_0^w - \mathbf x_0^w\|^2 \\
    &\geq \left(\inf_{y\in\mathcal M_{\mathrm{data}}} \|\tilde{\mathbf x}_0^w - y\|\right)^2
\end{align*}
By the definition of $\operatorname{dist}(\tilde{\mathbf x}_0^w, \mathcal M_{\mathrm{data}}) = \inf_{y\in\mathcal M_{\mathrm{data}}} \|\tilde{\mathbf x}_0^w - y\|$, we prove the inequality.
\end{proof}

\subsection{Extension of Theorem~\ref{thm:thermo_exp_drift_suppression}
to the ODE Endpoint}
\label{subsec:integrated-ode-extension}

The one-step reconstruction in \thmref{thm:thermo_exp_drift_suppression}
approximates the endpoint obtained by integrating the learned velocity field.
The following result transfers its manifold-distance guarantee to that endpoint.

\begin{theorem}[ODE endpoint drift control]
\label{thm:thermodpo_ode_control}
Fix the pointwise setting of
\thmref{thm:thermo_exp_drift_suppression}, and write
\(\mathcal M_{\mathrm{data}}:=\operatorname{supp}(p_0)\).
Let \(\{\mathbf x_s^{\theta,w}\}_{s\in[0,t]}\) solve
\[
    \frac{\mathrm d}{\mathrm ds}\mathbf x_s^{\theta,w}
    =v_\theta(\mathbf x_s^{\theta,w},s),
    \qquad
    \mathbf x_t^{\theta,w}=\mathbf x_t^w,
\]
and define
\(\Phi_{t\to0}^{\theta}(\mathbf x_t^w):=\mathbf x_0^{\theta,w}\).
Suppose \(v_\theta\) is continuously differentiable near this trajectory
and, for all \(s\in[0,t]\),
\begin{equation}
\label{eq:velocity_regularities}
    \bigl\|D_{\mathbf x}v_\theta(\mathbf x_s^{\theta,w},s)\bigr\|_{\mathrm{op}}
    \le L_x,\qquad
    \bigl\|\partial_s v_\theta(\mathbf x_s^{\theta,w},s)\bigr\|
    \le L_t,\qquad
    \bigl\|v_\theta(\mathbf x_s^{\theta,w},s)\bigr\|
    \le V.
\end{equation}
Then
\begin{equation}
\label{eq:integrated_ode_manifold_control}
    \operatorname{dist}\!\left(
        \Phi_{t\to0}^{\theta}(\mathbf x_t^w),
        \mathcal M_{\mathrm{data}}
    \right)
    \le
    \sqrt{g_\tau(\theta)}
    +\frac{L_t+L_xV}{2}\,t^2.
\end{equation}
\end{theorem}

Thus, the pointwise guarantee of
\thmref{thm:thermo_exp_drift_suppression} extends to the integrated ODE
endpoint up to the \(O(t^2)\) local error of a one-step Euler estimate.

\begin{proof}
For brevity, write
\(\mathbf x_s:=\mathbf x_s^{\theta,w}\) and
\(M_v:=L_t+L_xV\).
By the chain rule and \eqref{eq:velocity_regularities},
\[
    \left\|
        \frac{\mathrm d}{\mathrm ds}v_\theta(\mathbf x_s,s)
    \right\|
    =
    \left\|
        \partial_s v_\theta(\mathbf x_s,s)
        +D_{\mathbf x}v_\theta(\mathbf x_s,s)
         v_\theta(\mathbf x_s,s)
    \right\|
    \le M_v.
\]
Let
\(\widetilde{\mathbf x}_0^w
:=\mathbf x_t^w-t\,v_\theta(\mathbf x_t^w,t)\).
Since
\(\Phi_{t\to0}^{\theta}(\mathbf x_t^w)
=\mathbf x_t^w-\int_0^t v_\theta(\mathbf x_s,s)\,\mathrm ds\),
\begin{align}
    \left\|
        \Phi_{t\to0}^{\theta}(\mathbf x_t^w)
        -\widetilde{\mathbf x}_0^w
    \right\|
    &\le
    \int_0^t
        \left\|
            v_\theta(\mathbf x_t,t)-v_\theta(\mathbf x_s,s)
        \right\|
    \,\mathrm ds \notag\\
    &\le
    \int_0^t M_v(t-s)\,\mathrm ds
    =\frac{M_v}{2}\,t^2.
    \label{eq:euler_local_error}
\end{align}
Moreover, \eqref{eq:thermo_loss_lower_bound_distance} gives
\[
    \operatorname{dist}\!\left(
        \widetilde{\mathbf x}_0^w,\mathcal M_{\mathrm{data}}
    \right)
    \le \sqrt{g_\tau(\theta)}.
\]
The distance to a nonempty set is \(1\)-Lipschitz. Combining this fact
with \eqref{eq:euler_local_error} yields
\[
    \operatorname{dist}\!\left(
        \Phi_{t\to0}^{\theta}(\mathbf x_t^w),
        \mathcal M_{\mathrm{data}}
    \right)
    \le
    \left\|
        \Phi_{t\to0}^{\theta}(\mathbf x_t^w)
        -\widetilde{\mathbf x}_0^w
    \right\|
    +\operatorname{dist}\!\left(
        \widetilde{\mathbf x}_0^w,\mathcal M_{\mathrm{data}}
    \right),
\]
which proves \eqref{eq:integrated_ode_manifold_control}.
\end{proof}

\section{Additional Analysis of \methodweighted}
\label{app:thermo_weighted_analysis}

In this section, we provide additional analysis of \methodweighted:
\begin{itemize}[leftmargin=*]
    \item \propref{prop:thermo_weighted_low_temp} parallels \thmref{thm:thermo_to_rft} by establishing the connections to RFT for \methodweighted.
    \item \propref{prop:thermo_weighted_vs_flowdpo} parallels \thmref{thm:thermo_vs_flowdpo} by giving the anchored-FlowDPO decomposition.
    \item \propref{prop:thermo_weighted_drift} parallels \thmref{thm:thermo_exp_drift_suppression} by providing winner-side manifold drift control.
\end{itemize}
The \method objective in the main text is chosen because it yields the cleanest winner-side manifold distance surrogate statement. Equivalently, the practical loss used in both toy and real-image experiments can be written as
\begin{align}
\mathcal{L}_{\text{ThermoDPO-weighted}}(\theta)
&=
\mathbb E\big[g_{\tau}^{\mathrm{wt}}(\theta)\big], \\
g_{\tau}^{\mathrm{wt}}(\theta)
&:=
\tau\cdot\log\!\left(
1+
\exp\!\left(\frac{\Delta_\theta^w-\Delta_\theta^l}{\tau}\right)
+
\exp\!\left(\frac{(1-t)^2\ell_\theta^w}{\tau}\right)
\right), \nonumber
\end{align}
where, for the pointwise analysis below, we fix $(\mathbf x_0^w,\mathbf x_0^l,t,\epsilon)$ with $t\in(0,1)$ and write $\tau:=\tau(t)>0$. Relative to \method, removing the global $t^2$ prefactor and replacing $\ell_\theta^w$ by $(1-t)^2\ell_\theta^w$ shifts the winner-side anchor toward the terminal window $t\approx 0$.

\begin{proposition}[Low-temperature limit of \textup{\methodweightedname}]
\label{prop:thermo_weighted_low_temp}
Fix $(\mathbf x_0^w,\mathbf x_0^l,t,\epsilon)$ with $t\in(0,1)$. Then
\begin{equation}
\lim_{\tau \downarrow 0} g_{\tau}^{\mathrm{wt}}(\theta)
=
\max\!\left\{
0,\,
\Delta_\theta^w-\Delta_\theta^l,\,
(1-t)^2\ell_\theta^w
\right\}.
\end{equation}
Moreover, if $\Delta_\theta^w-\Delta_\theta^l \le (1-t)^2\ell_\theta^w$, then
\begin{equation}
\lim_{\tau \downarrow 0} g_{\tau}^{\mathrm{wt}}(\theta)
=
(1-t)^2\ell_\theta^w
=
\frac{(1-t)^2}{t^2}\|\tilde{\mathbf x}_0^w-\mathbf x_0^w\|^2,
\end{equation}
where $\tilde{\mathbf x}_0^w = \mathbf x_t^w - t\cdot v_\theta(\mathbf x_t^w,t)$.
\end{proposition}

\begin{proof}
Because
\begin{equation}
\lim_{\tau \downarrow 0}\tau\cdot\log\!\left(
1+\exp\!\left(\frac{a}{\tau}\right)+\exp\!\left(\frac{b}{\tau}\right)
\right)
=
\max\{0,a,b\},
\end{equation}
we obtain
\begin{equation}
\lim_{\tau \downarrow 0} g_{\tau}^{\mathrm{wt}}(\theta)
=
\max\!\left\{
0,\,
\Delta_\theta^w-\Delta_\theta^l,\,
(1-t)^2\ell_\theta^w
\right\}.
\end{equation}
If $\Delta_\theta^w-\Delta_\theta^l \le (1-t)^2\ell_\theta^w$, then
\begin{align}
\lim_{\tau \downarrow 0} g_{\tau}^{\mathrm{wt}}(\theta)
&=
(1-t)^2\ell_\theta^w \\
&=
\frac{(1-t)^2}{t^2}\cdot t^2\ell_\theta^w \\
&=
\frac{(1-t)^2}{t^2}\|\tilde{\mathbf x}_0^w-\mathbf x_0^w\|^2,
\end{align}
where the last identity follows from $t^2\ell_\theta^w=\|\tilde{\mathbf x}_0^w-\mathbf x_0^w\|^2$ as in the proof of \thmref{thm:thermo_to_rft}.
\end{proof}

\begin{remark}[Interpretation]
Compared with \thmref{thm:thermo_to_rft}, the practical variant no longer reduces exactly to the terminal reconstruction error. Instead, it reduces to a reweighted winner-side anchor with factor $(1-t)^2/t^2$, which is largest near the terminal endpoint $t=0$. This is precisely the regime where the practical implementation is intended to strengthen terminal manifold preservation.
\end{remark}

\begin{proposition}[\textup{\methodweightedname} as anchored FlowDPO]
\label{prop:thermo_weighted_vs_flowdpo}
For every $\tau>0$, the single-sample \textup{\methodweightedname} integrand admits the decomposition
\begin{equation}
g_{\tau}^{\mathrm{wt}}(\theta)
=
\tau\cdot\log\!\left(
1+
\exp\!\left(\frac{\Delta_\theta^w-\Delta_\theta^l}{\tau}\right)
\right)
+
r_{\tau}^{\mathrm{wt}}(\theta),
\end{equation}
where
\begin{equation}
r_{\tau}^{\mathrm{wt}}(\theta)
:=
\tau\cdot\log\!\left(
\frac{
1+
\exp\!\left(\frac{\Delta_\theta^w-\Delta_\theta^l}{\tau}\right)
+
\exp\!\left(\frac{(1-t)^2\ell_\theta^w}{\tau}\right)
}{
1+
\exp\!\left(\frac{\Delta_\theta^w-\Delta_\theta^l}{\tau}\right)
}
\right)
\ge 0.
\end{equation}
Therefore, \textup{\methodweightedname} remains a strict upper envelope of the corresponding temperature-scaled FlowDPO objective, with the excess term acting as a reweighted winner-side anchoring penalty.
\end{proposition}

\begin{proof}
\begin{align}
g_{\tau}^{\mathrm{wt}}(\theta)
&=
\tau\cdot\log\!\left(
1+
\exp\!\left(\frac{\Delta_\theta^w-\Delta_\theta^l}{\tau}\right)
+
\exp\!\left(\frac{(1-t)^2\ell_\theta^w}{\tau}\right)
\right) \\
&=
\tau\cdot\log\!\left(
1+
\exp\!\left(\frac{\Delta_\theta^w-\Delta_\theta^l}{\tau}\right)
\right) \nonumber\\
&\quad+
\tau\cdot\log\!\left(
\frac{
1+
\exp\!\left(\frac{\Delta_\theta^w-\Delta_\theta^l}{\tau}\right)
+
\exp\!\left(\frac{(1-t)^2\ell_\theta^w}{\tau}\right)
}{
1+
\exp\!\left(\frac{\Delta_\theta^w-\Delta_\theta^l}{\tau}\right)
}
\right). \nonumber
\end{align}
The numerator in the second logarithm is no smaller than the denominator, so $r_{\tau}^{\mathrm{wt}}(\theta)\ge 0$.
\end{proof}

\begin{proposition}[Winner-side manifold drift control of \textup{\methodweightedname}]
\label{prop:thermo_weighted_drift}
Let $\mathcal M_{\mathrm{data}}:=\operatorname{supp}(p_0)$.
For every $(\mathbf x_0^w,\mathbf x_0^l,t,\epsilon)$ with $t\in(0,1)$ and $\tau>0$,
\begin{equation}
\operatorname{dist}\!\bigl(\tilde{\mathbf x}_0^w,\mathcal M_{\mathrm{data}}\bigr)^2
\le
\frac{t^2}{(1-t)^2}\,g_{\tau}^{\mathrm{wt}}(\theta),
\end{equation}
where $\tilde{\mathbf x}_0^w = \mathbf x_t^w - t\cdot v_\theta(\mathbf x_t^w,t)$. 
In particular, if $t\le \frac12$, then
\begin{equation}
\operatorname{dist}\!\bigl(\tilde{\mathbf x}_0^w,\mathcal M_{\mathrm{data}}\bigr)^2
\le
g_{\tau}^{\mathrm{wt}}(\theta).
\end{equation}
\end{proposition}

\begin{proof}
\begin{align}
g_{\tau}^{\mathrm{wt}}(\theta)
&=
\tau\cdot\log\!\left(
1+
\exp\!\left(\frac{\Delta_\theta^w-\Delta_\theta^l}{\tau}\right)
+
\exp\!\left(\frac{(1-t)^2\ell_\theta^w}{\tau}\right)
\right) \\
&\ge
\tau\cdot\frac{(1-t)^2\ell_\theta^w}{\tau} \nonumber\\
&=
(1-t)^2\ell_\theta^w \nonumber\\
&=
\frac{(1-t)^2}{t^2}\|\tilde{\mathbf x}_0^w-\mathbf x_0^w\|^2 \nonumber\\
&\ge
\frac{(1-t)^2}{t^2}\operatorname{dist}\!\bigl(\tilde{\mathbf x}_0^w,\mathcal M_{\mathrm{data}}\bigr)^2, \nonumber
\end{align}
because $\mathbf x_0^w\in\mathcal M_{\mathrm{data}}$. Rearranging gives
\begin{equation}
\operatorname{dist}\!\bigl(\tilde{\mathbf x}_0^w,\mathcal M_{\mathrm{data}}\bigr)^2
\le
\frac{t^2}{(1-t)^2}\,g_{\tau}^{\mathrm{wt}}(\theta).
\end{equation}
If $t\le \bar t$, then
\begin{equation}
\frac{t^2}{(1-t)^2}
\le
\frac{\bar t^2}{(1-\bar t)^2},
\end{equation}
which proves the truncated-window bound. The case $t\le \frac12$ follows because $\frac{t^2}{(1-t)^2}\le 1$.
\end{proof}

\newpage

\section{Experiments Details}

\subsection{Toy Experiment Setup Details}
\label{app:toy_setup}

\paragraph{Toy manifold.}
We construct a toy 3D data manifold to study how preference optimization affects terminal manifold preservation. Specifically, we define a curved surface in $\mathbb{R}^3$ by
\begin{equation}
\mathcal M
=
\left\{
(x,y,z)\in\mathbb{R}^3
:\;
x\in[-5,5],\;
y\in[-1,1],\;
z=f(x,y)
\right\},
\end{equation}
where
\begin{equation}
f(x,y)
=
\Bigl(
1.2 \times \exp\!\bigl(-(x-3)^2\bigr)
+
1.2 \times \exp\!\bigl(-(x+3)^2\bigr)
-
0.9 \times \exp\!\bigl(-x^2\bigr)
\Bigr)
\cdot
\Bigl(1 - 0.15 y^2\Bigr).
\end{equation}
This surface consists of two elevated bump regions centered around $x=\pm 3$ and one depressed dip region centered around $x=0$, with a mild modulation along the $y$-direction.

\paragraph{Preferred and dispreferred regions.}
To define synthetic preference labels, we partition the manifold according to the $x$-coordinate. Samples in the two bump regions are treated as preferred,
\begin{equation}
\mathcal M_{\mathrm{win}}
=
\{(x,y,z)\in\mathcal M:\ |x-3|<1 \ \text{or}\ |x+3|<1\},
\end{equation}
while samples in the central dip region are treated as dispreferred,
\begin{equation}
\mathcal M_{\mathrm{lose}}
=
\{(x,y,z)\in\mathcal M:\ |x|<1\}.
\end{equation}
The remaining manifold points are regarded as neutral and are not assigned preference labels. This construction produces a simple but geometrically meaningful preference task: the model is encouraged to shift probability mass from the dip region toward the two bump regions.

\paragraph{Training pipeline.}
We first pretrain a reference model on samples from the toy surface using standard flow matching, so that the model learns to generate points lying on the underlying data manifold. In the current implementation, the toy vector field is a three-layer MLP with hidden size $128$ and SiLU activations. We pretrain this model for $10{,}000$ steps with batch size $512$, Adam optimizer, and learning rate $10^{-3}$. Starting from this pretrained model, we then construct synthetic preference pairs by treating samples from the two bump regions as preferred and samples from the central dip region as dispreferred. Unless specified otherwise below, each second-stage method is fine-tuned for $10{,}000$ update steps with batch size $512$, Adam optimizer, learning rate $10^{-4}$, and gradient clipping at norm $1.0$. We evaluate every trained model on $10{,}000$ generated samples. All toy experiments are run on CPU only, and each toy configuration completes within roughly five minutes in our implementation. This controlled pipeline allows us to compare how different preference optimization methods improve preference alignment while affecting terminal manifold preservation.

\paragraph{Evaluation metrics.}
We evaluate both preference alignment and manifold preservation using the following manifold-aware metrics on generated samples:
\begin{itemize}[leftmargin=*]
\item \textsc{Winner ratio:} The fraction of generated samples that lie in the preferred region. This measures how many samples are in the preferred region without considering the manifold drift.
\item \textsc{Loser ratio:} The fraction of generated samples that lie in the dispreferred region. This measures how many samples remain in undesirable regions without considering the manifold drift.
\item \textsc{Strict winner ratio:} The fraction of generated samples that are both on-manifold and lie in the preferred region. This measures how much valid probability mass is assigned to preferred samples.
\item \textsc{On-manifold ratio:} The fraction of generated points that fall inside the valid $(x,y)$ domain and inside an $\varepsilon$-tube around the true toy surface, i.e., $|z-f(x,y)|\le \varepsilon$ with $\varepsilon=0.15$. 
\item \textsc{Winner quality:} The fraction of preferred-region samples that are also on-manifold. This measures whether samples attracted toward the preferred region remain geometrically valid.
\item \textsc{Strict preference score:} The equally weighted combination
\[
\mathrm{StrictScore}
=
0.5 \cdot \mathrm{StrictWin}
+
0.5 \cdot \mathrm{OnManifold},
\]
which summarizes the trade-off between preference satisfaction and manifold preservation.
\end{itemize}

\paragraph{Compared methods.}
We compare the following methods in the toy experiment:
\begin{itemize}[leftmargin=*]
\item \textbf{RFT}~\cite{xiong2025minimalist,chen2026nft}. Starting from the pretrained FM model, we fine-tune it with the standard RFT objective only on synthetic data constructed from the preferred bump regions.
\item \textbf{FlowDPO}~\cite{liu2025improving}. Starting from the pretrained FM model, we fine-tune it with the standard FlowDPO objective on synthetic preference pairs constructed from the preferred bump regions and the dispreferred dip region.
\item \textbf{Diffusion-SDPO}~\cite{fu2025diffusionsdpo}. We adapt its safeguarded DPO update to the squared flow-matching residuals. The forward preference margin is identical to that of FlowDPO, while the backward gradient through the loser residual is multiplied by a detached safeguard factor. Specifically, for the winner- and loser-side output gradients $g_w$ and $g_l$, respectively, the factor is $s_\mu=\operatorname{clip}((1-\mu)\|g_w\|^2/\langle g_w,g_l\rangle,0,1)$ when $\langle g_w,g_l\rangle>0$, and $s_\mu=1$ otherwise. We use a frozen pretrained reference model, set $\mu=0.99$, and sweep $\beta\in\{1,10,100,500\}$.
\item \textbf{Linear-DPO}~\cite{li2026lineardpo}. We adapt its linear preference objective to flow-matching residuals. With $\Delta_d:=\Delta_\theta^w-\Delta_\theta^l$, the implementation forms the detached utility weight $u=\operatorname{clip}(0.2\beta\Delta_d+0.5,0.01,1)$ and minimizes $\E[u(\ell_\theta^w-\ell_\theta^l)]$. The reference model is updated after every optimizer step using an exponential moving average with decay $0.9999$. Following the configuration in our implementation, we use learning rate $10^{-5}$ and sweep $\beta\in\{1,10,100,500\}$.
\item \textbf{$\chi$PO}~\cite{huang2024correcting}. Starting from the pretrained FM model, we fine-tune it with the standard $\chi$PO objective $\E[-\log \sigma(\beta\cdot [(\exp(-\Delta_\theta^w) - \Delta_\theta^w) - (\exp(-\Delta_\theta^l) - \Delta_\theta^l)])]$ on synthetic preference pairs constructed from the preferred bump regions and the dispreferred dip region.
\item \textbf{FlowDPO + KL.} We augment FlowDPO with an additional KL-style regularization term $\|v_\theta- v_{\text{ref}}\|$ that keeps the fine-tuned model close to the pretrained reference model. This baseline tests whether staying closer to the reference flow is sufficient to mitigate manifold drift.
\item \textbf{FlowDPO + RFT.} We combine the FlowDPO objective with an additional flow-matching term on preferred samples. This baseline tests whether explicitly anchoring optimization toward preferred data can improve the trade-off between preference alignment and manifold preservation.
\item \textbf{\methodname / \methodweightedname.} Starting from the same pretrained FM model, we fine-tune it with the prototype \method objective and with the practical \methodweighted variant on the same preference pairs. The main-text tables focus on \methodweighted because it is the practical implementation used throughout the empirical study. We also test different $\tau$ scheduler (see \figref{fig:tau_schedules}) which means different weights of preference optimization and terminal manifold perservation.
\end{itemize}

\begin{figure*}[ht]
    \centering
    \includegraphics[width=0.45\textwidth]{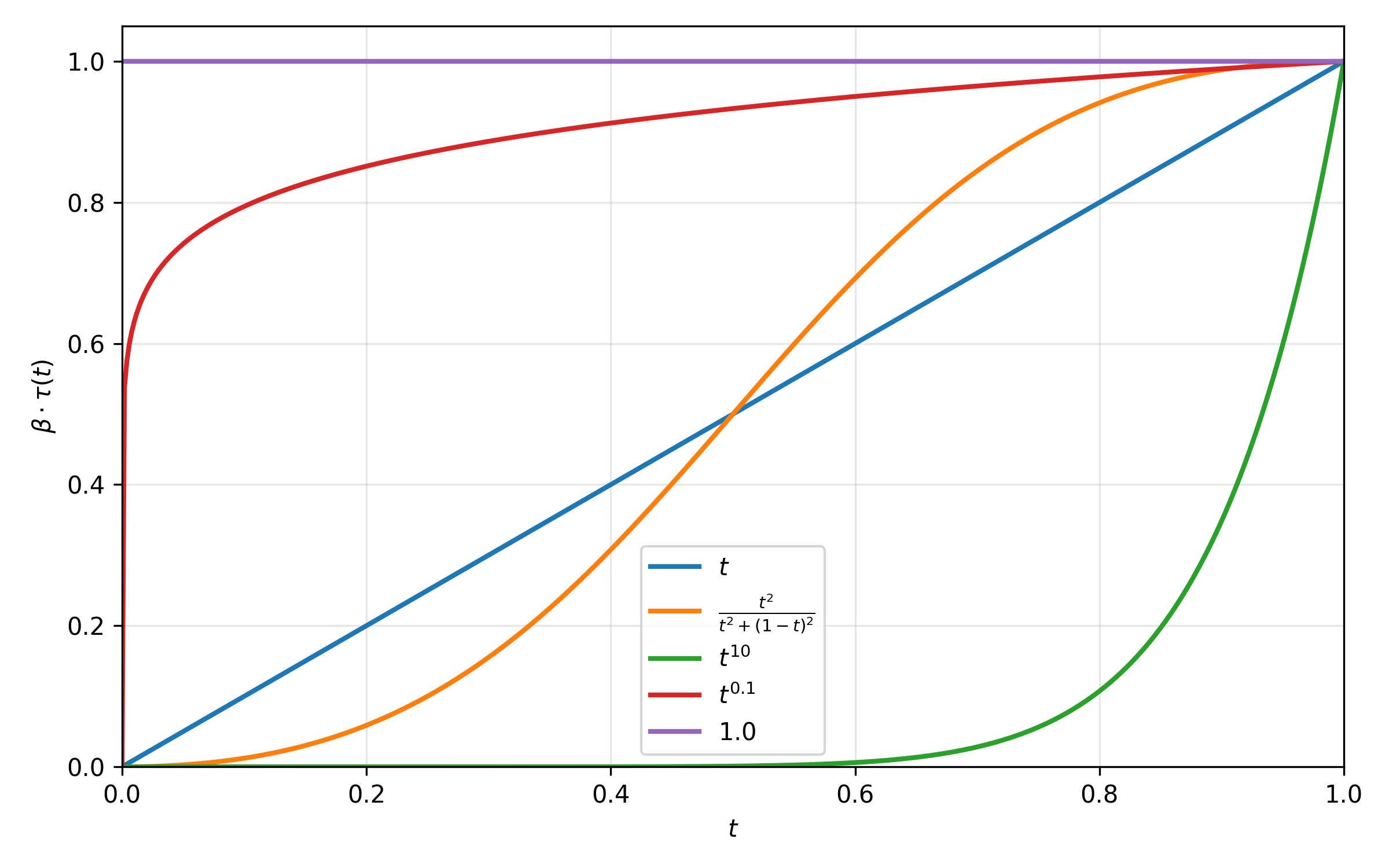}
    \caption{
    Illustration of different temperature schedules $\tau(t)$ used in \method, including the linear schedule $\tau(t)=t$, the SNR-style schedule $\beta \cdot \tau(t)=\frac{t^2}{t^2+(1-t)^2}$, and the power schedules $\beta \cdot \tau(t)=t^{10}$ and $\beta \cdot \tau(t)=t^{0.1}$.
    These schedules control the trade-off between terminal manifold preservation and preference-driven trajectory deformation.
    }
    \label{fig:tau_schedules}
\end{figure*}

\paragraph{Comparison to $\chi$PO.}
We compare against $\chi$PO~\cite{huang2024correcting}, a DPO variant designed to mitigate overoptimization by replacing the logarithmic link function in the standard DPO objective. On our toy benchmark, $\chi$PO performs comparably to DPO and does not yield a consistent advantage. We emphasize that this should be interpreted as a setting-specific observation rather than a contradiction of prior work: $\chi$PO is formulated for the conventional DPO objective over log-probability ratios, whereas our setting is for flow matching, so its benefits may not transfer directly.

\subsection{Additional Toy Experiments Results}
\label{app:toy_exp_results}

\subsubsection{Toy Results of \method Variants}
\label{app:toy_exp_results_thermo_dpo_variants}
We summarizes the additional sweep results of \method variants in \tabref{tab:thermo_dpo_results}.

\begin{table*}[ht]
    \centering
    \small
    \setlength{\tabcolsep}{5pt}
    \renewcommand{\arraystretch}{1.08}
    \caption{Toy experiment results comparing \method variants and \methodweighted variants. 
    Best results are in \textbf{bold}, and second-best results are \underline{underlined}. 
    For Loss, lower is better; for all other metrics, higher is better.}
    \label{tab:thermo_dpo_results}
    \resizebox{\linewidth}{!}{
    \begin{tabular}{lcccccc}
    \toprule
    Method & Win (\%) $\uparrow$ & Loss (\%) $\downarrow$ & StrictWin (\%) $\uparrow$ & OnManifold (\%) $\uparrow$ & WinQuality $\uparrow$ & StrictScore $\uparrow$ \\
    \midrule 
    \multicolumn{7}{l}{\textit{\textup{\methodname} with $\tau(t)=\frac{t}{\beta}$}} \\ 
    \method ($t$, $\beta=1$) & 81.7 & 1.0 & 63.5 & 72.4 & 0.778 & 0.68 \\ 
    \method ($t$, $\beta=10$) & 50.7 & 12.8 & 44 & 85.9 & 0.87 & 0.65 \\ 
    \method ($t$, $\beta=100$) & 42.5 & 17.1 & 36.4 & 88 & 0.856 & 0.622 \\ 
    \method ($t$, $\beta=500$) & 41.2 & 17.9 & 35 & 88 & 0.849 & 0.615 \\ 
    \midrule 
    \multicolumn{7}{l}{\textit{\textup{\methodname} with $\tau(t)=\frac{1}{\beta}\frac{t^2}{t^2+(1-t)^2}$}} \\
    \method ($\frac{t^2}{t^2 + (1-t)^2}$, $\beta=1$) & 79.4 & 1.1 & 61.6 & 72.1 & 0.777 & 0.669 \\ 
    \method ($\frac{t^2}{t^2 + (1-t)^2}$, $\beta=10$) & 50.7 & 12.8 & 44 & 86.3 & 0.868 & 0.651 \\ 
    \method ($\frac{t^2}{t^2 + (1-t)^2}$, $\beta=100$) & 43.1 & 17.2 & 37 & 87.9 & 0.857 & 0.625 \\ 
    \method ($\frac{t^2}{t^2 + (1-t)^2}$, $\beta=500$) & 41.4 & 17.6 & 35.3 & 88 & 0.853 & 0.616 \\ 
    \midrule
    \multicolumn{7}{l}{\textit{\textup{\methodname} with $\tau(t)=\frac{t^{10}}{\beta}$}} \\ 
    \method ($t^{10}$, $\beta=1$) & 47.4 & 14.2 & 40.9 & 86.9 & 0.863 & 0.639 \\ 
    \method ($t^{10}$, $\beta=10$) & 43.5 & 17.2 & 37.6 & 87.6 & 0.864 & 0.626 \\ 
    \method ($t^{10}$, $\beta=100$) & 41.6 & 17.6 & 35.2 & 87.8 & 0.847 & 0.615 \\ 
    \method ($t^{10}$, $\beta=500$) & 41.1 & 17.9 & 35.2 & 88.3 & 0.856 & 0.617 \\ 
    \midrule 
    \multicolumn{7}{l}{\textit{\textup{\methodname} with $\tau(t)=\frac{t^{0.1}}{\beta}$}} \\ 
    \method ($t^{0.1}$, $\beta=1$) & 92.0 & 0.4 & 65.5 & 69 & 0.712 & 0.673 \\ 
    \method ($t^{0.1}$, $\beta=10$) & 54 & 10.3 & 46.6 & 81.3 & 0.863 & 0.64 \\ 
    \method ($t^{0.1}$, $\beta=100$) & 43.8 & 16.9 & 37.9 & 88.3 & 0.866 & 0.631 \\ 
    \method ($t^{0.1}$, $\beta=500$) & 41.2 & 18.3 & 35.2 & 87.8 & 0.854 & 0.615 \\ 
    \midrule 
    \multicolumn{7}{l}{\textit{\textup{\methodweightedname} with $\tau(t)=\frac{t}{\beta}$}} \\ 
    \methodweighted ($t$, $\beta=1$) & 92.7 & 0.5 & \underline{87.6} & \textbf{92.2} & \textbf{0.945} & \underline{0.899} \\ 
    \methodweighted ($t$, $\beta=10$) & 92.0 & 0.4 & 86.1 & 91.5 & 0.935 & 0.888 \\ 
    \methodweighted ($t$, $\beta=100$) & 91.2 & 0.4 & 85.9 & 92.1 & 0.941 & 0.89 \\ 
    \methodweighted ($t$, $\beta=500$) & 91.9 & 0.4 & 86.2 & 92.2 & 0.939 & 0.892 \\ 
    \midrule 
    \multicolumn{7}{l}{\textit{\textup{\methodweightedname} with $\tau(t)=\frac{1}{\beta}\frac{t^2}{t^2+(1-t)^2}$}} \\ 
    \methodweighted ($\frac{t^2}{t^2 + (1-t)^2}$, $\beta=1$) & 92.5 & 0.6 & 87.2 & 91.9 & 0.943 & 0.895 \\ 
    \methodweighted ($\frac{t^2}{t^2 + (1-t)^2}$, $\beta=10$) & 91.6 & 0.4 & 85.7 & 91.8 & 0.936 & 0.887 \\ 
    \methodweighted ($\frac{t^2}{t^2 + (1-t)^2}$, $\beta=100$) & 91.4 & 0.3 & 85.6 & 91.6 & 0.936 & 0.886 \\ 
    \methodweighted ($\frac{t^2}{t^2 + (1-t)^2}$, $\beta=500$) & 91.2 & 0.4 & 85.7 & 92 & 0.94 & 0.888 \\ 
    \midrule 
    \multicolumn{7}{l}{\textit{\textup{\methodweightedname} with $\tau(t)=\frac{t^{10}}{\beta}$}} \\ 
    \methodweighted ($t^{10}$, $\beta=1$) & 91.8 & 0.4 & 86.4 & 92.2 & 0.941 & 0.893 \\ 
    \methodweighted ($t^{10}$, $\beta=10$) & 91.6 & 0.4 & 86.3 & 92.2 & 0.941 & 0.892 \\ 
    \methodweighted ($t^{10}$, $\beta=100$) & 91.4 & 0.5 & 85.6 & 91.9 & 0.936 & 0.887 \\ 
    \methodweighted ($t^{10}$, $\beta=500$) & 91.4 & 0.4 & 85.5 & 91.8 & 0.936 & 0.887 \\ 
    \midrule 
    \multicolumn{7}{l}{\textit{\textup{\methodweightedname} with $\tau(t)=\frac{t^{0.1}}{\beta}$}} \\
    \methodweighted ($t^{0.1}$, $\beta=1$) & 93.1 & 0.5 & \textbf{87.9} & \textbf{92.2} & \underline{0.944} & \textbf{0.9} \\ 
    \methodweighted ($t^{0.1}$, $\beta=10$) & 91.9 & 0.4 & 86 & 92 & 0.936 & 0.89 \\ 
    \methodweighted ($t^{0.1}$, $\beta=100$) & 92.1 & 0.4 & 86.4 & 92 & 0.939 & 0.892 \\ 
    \methodweighted ($t^{0.1}$, $\beta=500$) & 91.4 & 0.3 & 86.1 & 92.2 & 0.942 & 0.891 \\
    \bottomrule
    \end{tabular}
    }
\end{table*}

\subsubsection{Additional Results of RFT}
\label{app:toy_exp_results_rft}
\tabref{tab:rft_results} isolates the effect of longer RFT training on the toy surface. Increasing the budget from $10$K to $100$K steps raises OnManifold from $88.3\%$ to $96.5\%$ and StrictScore from $0.858$ to $0.954$, consistent with winner-only reconstruction fitting this synthetic preferred region. The result should not be transferred directly to real images: in \tabref{tab:main_results_relative_cfg45}, RFT is strong on held-out metrics but does not attain the largest four-metric macro-average gain.

\begin{table*}[ht]
    \centering
    \small
    \setlength{\tabcolsep}{5pt}
    \renewcommand{\arraystretch}{1.08}
    \caption{Toy experiment results of RFT. }
    \label{tab:rft_results}
    \resizebox{\linewidth}{!}{
    \begin{tabular}{lccccccc}
    \toprule
    Method & Itr. & Win (\%) $\uparrow$ & Loss (\%) $\downarrow$ & StrictWin (\%) $\uparrow$ & OnManifold (\%) $\uparrow$ & WinQuality $\uparrow$ & StrictScore $\uparrow$ \\
    \midrule
    RFT & 1K & 89.2 & 1.5 & 78.1 & 85.8 & 0.876 & 0.819 \\ 
    RFT & 10K & 93.6 & 0.3 & 83.4 & 88.3 & 0.891 & 0.858 \\ 
    RFT & 50K & 96.8 & 0.1 & 92.4 & 94.9 & 0.954 & 0.937 \\ 
    RFT & 100K & 97.5 & 0.0 & 94.3 & 96.5 & 0.967 & 0.954 \\ 
    \bottomrule
    \end{tabular}
    }
\end{table*}
    
\subsubsection{Toy Results of Manifolds with Different Curvatures}
\label{app:toy_exp_results_curvature}
\tabref{tab:curvature_results} repeats the controlled comparison on a plane, cylinder, sphere, and saddle. Within the reported $\beta$ grid, \methodweighted attains the highest StrictScore on each geometry ($0.947$, $0.976$, $0.986$, and $0.961$, respectively). The degree of FlowDPO drift varies with geometry and $\beta$, so these results support robustness across the tested surfaces rather than universal manifold preservation.

\begin{table*}[ht]
    \centering
    \small
    \setlength{\tabcolsep}{5pt}
    \renewcommand{\arraystretch}{1.08}
    \caption{Toy experiment results comparing FlowDPO and \methodweighted variants across manifolds with different curvature. 
    Best results are in \textbf{bold} for each geometry.
    $\mathrm{II}$ means the second fundamental form, and $K$ means the Gaussian curvature.
    }
    \label{tab:curvature_results}
    \resizebox{\linewidth}{!}{
    \begin{tabular}{lcccccc}
    \toprule
    Method & Win (\%) $\uparrow$ & Loss (\%) $\downarrow$ & StrictWin (\%) $\uparrow$ & OnManifold (\%) $\uparrow$ & WinQuality $\uparrow$ & StrictScore $\uparrow$ \\
    \midrule 
    \textbf{Plane ($\mathrm{II}=0$, $K=0$)} \\ 
    RFT	& 95.6 & 0 & 90.2 & 94.5 & 0.943 & 0.923 \\
    FlowDPO	($\beta=1$) &  85.1 & 0.1 & 47.6 & 59.1 & 0.56 & 0.534 \\
    FlowDPO	($\beta=10$) & 35.1 & 15.4 & 31.2 & 78.5 & 0.889 & 0.549 \\
    FlowDPO	($\beta=100$) & 30.3 & 19.1 & 27.8 & 94.3 & 0.918 & 0.611 \\
    FlowDPO	($\beta=500$) & 34.1 & 15.9 & 30.7 & 79.1 & 0.9 & 0.549 \\
    \methodweighted	(linear,$\beta=1$)	& \textbf{97.4} & 0 & \textbf{93.4} & \textbf{96} & \textbf{0.959} & \textbf{0.947} \\
    \methodweighted	(linear,$\beta=10$)	& 97.3 & 0 & 93.3 & 95.9 & 0.959 & 0.946 \\
    \methodweighted	(linear,$\beta=100$) & 97.2 & 0 & 93 & 95.7 & 0.956 & 0.943 \\
    \methodweighted	(linear,$\beta=500$) & 97.2 & 0 & 92.8 & 95.5 & 0.955 & 0.942 \\
    \midrule 
    \textbf{Cylinder ($\mathrm{II} \neq 0$, $K=0$)} \\
    RFT	& 97.7 & 0 & 95.7 & 97.9 & 0.979 & 0.968 \\
    FlowDPO	($\beta=1$)	& 90.8 & 0 & 20.4 & 28.1 & 0.225 & 0.243 \\
    FlowDPO	($\beta=10$)	& 74.1 & 0.2 & 25.7 & 47.4 & 0.347 & 0.365 \\
    FlowDPO	($\beta=100$)	& 32.6 & 16.4 & 31.4 & 93.4 & 0.966 & 0.624 \\
    FlowDPO	($\beta=500$)	& 25.8 & 22.8 & 25 & 96.9 & 0.968 & 0.609 \\
    \methodweighted	(linear,	$\beta=1$)	& 97.8 & 0 & \textbf{96.5} & \textbf{98.7} & \textbf{0.987} & \textbf{0.976} \\
    \methodweighted	(linear,	$\beta=10$)	& 97.8 & 0 & 96.1 & 98.2 & 0.982 & 0.971 \\
    \methodweighted	(linear,	$\beta=100$)	& 97.7 & 0 & 96.1 & 98.4 & 0.984 & 0.972 \\
    \methodweighted	(linear,	$\beta=500$)	& \textbf{97.9} & 0 & 96.2 & 98.3 & 0.983 & 0.972 \\
    \midrule
    \textbf{Sphere ($K > 0$)} \\ 
    RFT	& 96.7 & 0 & 96.6 & 99.9 & 0.999 & 0.982 \\
    FlowDPO	($\beta=1$)	& \textbf{100} & 0 & 39.3 & 39.3 & 0.393 & 0.393 \\
    FlowDPO	($\beta=10$)	& 66.2 & 2.5 & 65.1 & 84.5 & 0.983 & 0.748 \\
    FlowDPO	($\beta=100$)	& 27.6 & 21.4 & 27.5 & 99.6 & 0.997 & 0.635 \\
    FlowDPO	($\beta=500$)	& 26.6 & 23.2 & 26.5 & 99.6 & 0.996 & 0.631 \\
    \methodweighted	(linear,	$\beta=1$)	& 97.2 & 0 & \textbf{97.2} & 99.9 & \textbf{1} & 0.985 \\
    \methodweighted	(linear,	$\beta=10$)	& 96.9 & 0 & 96.9 & 99.9 & 0.999 & 0.984 \\
    \methodweighted	(linear,	$\beta=100$)	& 97.1 & 0 & 97.1 & \textbf{100} & \textbf{1} & 0.985 \\
    \methodweighted	(linear,	$\beta=500$)	& 97.1 & 0 & 97.1 & \textbf{100} & \textbf{1} & \textbf{0.986} \\
    \midrule 
    \textbf{Saddle ($K < 0$)} \\ 
    RFT	& \textbf{97.9} & 0 & 94 & 96 & 0.96 & 0.95 \\
    FlowDPO	($\beta=1$)	& 96.1 & 0 & 59 & 62.6 & 0.613 & 0.608 \\
    FlowDPO	($\beta=10$)	& 71 & 2.3 & 66.4 & 88.7 & 0.936 & 0.775 \\
    FlowDPO	($\beta=100$)	& 28.8 & 19.9 & 27.3 & 96 & 0.946 & 0.616 \\
    FlowDPO	($\beta=500$)	& 27.3 & 21.9 & 25.9 & 96.2 & 0.948 & 0.611 \\
    \methodweighted	(linear,	$\beta=1$)	& 97.4 & 0 & 94.8 & \textbf{97.4} & \textbf{0.974} & \textbf{0.961} \\
    \methodweighted	(linear,	$\beta=10$)	& 97.2 & 0 & 94.4 & 97.3 & 0.972 & 0.959 \\
    \methodweighted	(linear,	$\beta=100$)	& 97.6 & 0 & 94.9 & 97.2 & 0.972 & 0.96 \\
    \methodweighted	(linear,	$\beta=500$)	& 97.8 & 0 & \textbf{95} & 97.2 & 0.971 & \textbf{0.961} \\
    \bottomrule
    \end{tabular}
    }
\end{table*}

\subsection{Real-Image Experimental Protocol on SD3.5-M}
\label{app:real_image_setup}

\paragraph{Prompt suites.}
For the real-image experiment, we use the OCR suitein which reward hacking is easier to observe. The prompt is the same as the OCR prompt of DiffusionNFT~\citep{zheng2025diffusionnft}.

\paragraph{Backbone model and training setup.}
All methods are initialized from the Stable Diffusion 3.5-M checkpoint. To ensure a controlled comparison, we keep the architecture, tokenizer, text encoder, image resolution, inference sampler, and sampling budget fixed across methods. All real-image experiments are run at resolution $512 \times 512$ using the DPM2 sampler with 40 sampling steps. Unless otherwise stated, all preference fine-tuning runs use AdamW, a global batch size of 128, and 1200 update steps. Training is performed on 64 $\times$ H100 GPUs with 80GB memory per GPU. For each real-image method, we train for more than 24 hours on 64 GPUs, i.e. more than $24 \times 64 = 1536$ GPU-hours per full training sweep. The checkpoints reported in the main results correspond to approximately 711 GPU-hours for RFT, 509 GPU-hours for FlowDPO, 252 GPU-hours for FlowDPO + RFT, 364 GPU-hours for FlowDPO + KL, and 690 GPU-hours for \methodweighted.

\paragraph{Compared methods.}
We compare the following methods in the real-image setting:
\begin{itemize}[leftmargin=*]
\item \textbf{Base FM model.} The pretrained flow-matching model before preference fine-tuning.
\item \textbf{RFT.} Rejection sampling fine-tuning on preferred samples only.
\item \textbf{FlowDPO.} Standard pairwise preference optimization for flow models.
\item \textbf{FlowDPO + KL.} FlowDPO with an additional regularization term that keeps the updated vector field close to the pretrained reference model.
\item \textbf{FlowDPO + RFT.} A hybrid baseline that combines pairwise preference optimization with a winner-side reconstruction anchor.
\item \textbf{\methodweightedname.} Our proposed method with temperature schedules $\tau(t) = t / \beta$ used in the toy experiments.
\end{itemize}

\paragraph{Preference data construction.}
We construct preference pairs using the same prompt pool for all methods. Specifically, we use 1000 prompts taken directly from DrawBench~\citep{saharia2022photorealistic}. For each prompt, we sample 64 candidate images from the pretrained model and rank them with the OCR-based training reward described below. We keep exactly one winner / loser pair per prompt, yielding an offline preference dataset for pairwise fine-tuning. This controlled offline setup matches the practical preference optimization in industrial usage and ensures that the comparison across methods is driven by the optimization objective rather than by differences in prompt pools or candidate generation.

\subsubsection{Automatic Reward-Model Evaluation}
\label{app:real_image_reward}

\paragraph{Training reward and held-out evaluators.}
To distinguish genuine alignment from reward hacking, we separate the \emph{training reward} used to construct the preference data from the \emph{held-out evaluators} used only at test time. The training reward is an OCR-based prompt-satisfaction metric implemented with PaddleOCR. For held-out evaluation, GenEval follows the same implementation used in Flow-GRPO~\citep{liu2025flow} and DiffusionNFT~\citep{zheng2025diffusionnft}, HPSv3.0 uses the authors' open-source codebase~\citep{ma2025hpsv3widespectrumhumanpreference}, and UniReward follows the DiffusionNFT evaluation pipeline~\citep{zheng2025diffusionnft}.

\subsubsection{Human Evaluation Protocol}
\label{app:real_image_human}

We report pairwise human preferences in \figref{fig:human_evaluation}. The current study uses the first 30 prompts from the test dataset, each comparison is labeled by a single annotator, and a tie option is allowed. Because each comparison is labeled once, no inter-rater aggregation is required.

The annotator was one of the paper authors. No external crowd workers or paid participants were recruited, and no compensation was involved. The task consisted only of side-by-side judgments over generated images, and no personal or sensitive data were collected.

For each prompt, the annotation interface displayed the prompt text together with the candidate images under comparison and asked the annotator to make two separate judgments: one for text accuracy and one for visual quality. The exact instructions were as follows: ``Given the prompt and the candidate generated images, judge text accuracy and visual quality separately. For \emph{text accuracy}, choose the image that better matches the prompt, especially the requested visible text content. For \emph{visual quality}, choose the image with better overall perceptual quality, including realism, coherence, and freedom from obvious artifacts. If the compared images are indistinguishable for a criterion, choose \emph{tie}.'' An example screenshot of the interface is shown in \figref{fig:human_eval_interface}.

Because the annotator was a paper author and the task only involved viewing model outputs and recording preference judgments, we regarded the study as minimal risk. To the best of our understanding of local requirements, this internal author-only evaluation did not require separate IRB or equivalent review.

\begin{figure*}[t]
    \centering
    \includegraphics[width=\textwidth]{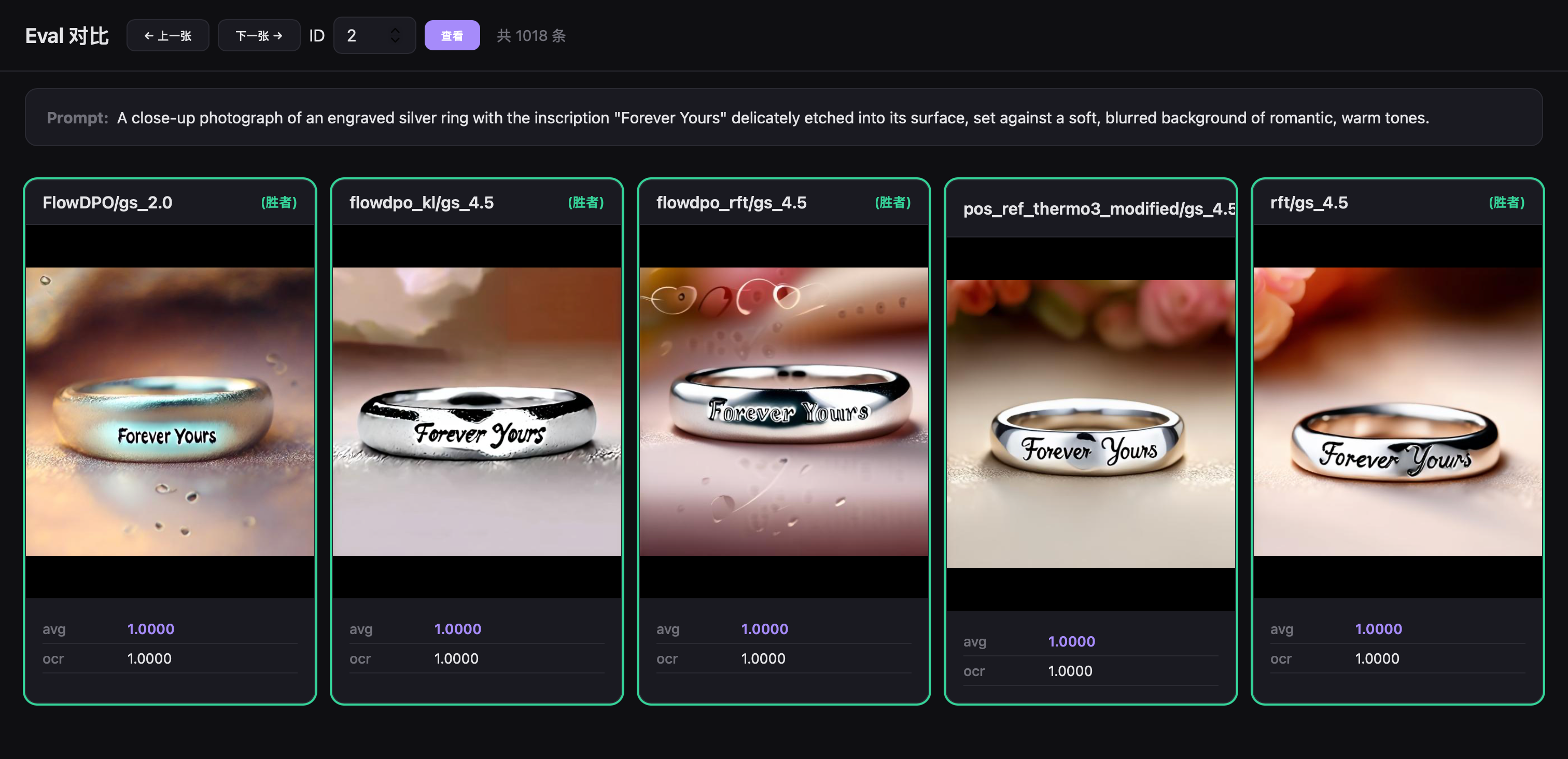}
    \caption{
    Example screenshot of the human-evaluation interface used in our study.
    The interface displays the prompt together with the candidate images under comparison for the current example.
    }
    \label{fig:human_eval_interface}
\end{figure*}

\clearpage
\subsection{Qualitative Results of Real-Image Experiments on SD3.5-M~\cite{esser2024scaling}}
\label{app:real_image_qualitative}

We provide qualitative results of the real-image experiments in \figref{fig:real_image_qualitative_flowdpo}, \figref{fig:real_image_qualitative_thermodpo}, \figref{fig:real_image_qualitative_flowdpo_kl}, \figref{fig:real_image_qualitative_flowdpo_rft}, \figref{fig:real_image_qualitative_original}, and \figref{fig:real_image_qualitative_rft}. The prompts are seleted from the top 20 prompts in the OCR test set. The prompts from left to right from top to bottom are in \figref{fig:prompts_figure1}. Clearly, the FlowDPO variants exhibit noticeable quality degradation, while \methodweighted and RFT maintain visual quality closer to the original model.
\begin{figure*}[ht]
    \centering
    \caption{Prompts in Appendix~\ref{app:real_image_qualitative} (from left to right from top to bottom).}
    \label{fig:prompts_figure1}
    
    \begin{tcolorbox}[
        width=0.96\textwidth,
        colback=white,
        colframe=black!70,
        boxrule=0.8pt,
        arc=3mm,
        left=6mm,
        right=6mm,
        top=4mm,
        bottom=4mm
    ]
    \small
    \begin{enumerate}[leftmargin=8mm,itemsep=0.35em,topsep=0pt]
        \item A weathered cave explorer's journal page, with the phrase "Lost City Near" prominently written in faded ink, surrounded by sketches of ancient ruins and cryptic symbols, under a dim, mystical light.
        \item A high-altitude mountain summit with a wooden signpost clearly marked "Elevation 8000 Feet", surrounded by rocky terrain and a backdrop of distant, snow-capped peaks under a clear blue sky.
        \item A hiking trail with a wooden signpost clearly displaying "Private Property No Entry", surrounded by dense, green foliage and a winding dirt path leading into the forest.
        \item A realistic photo of a tech campus courtyard at night, featuring a glowing "AI Training Zone" hologram floating in the center, surrounded by futuristic buildings and greenery, with soft ambient lighting enhancing the futuristic atmosphere.
        \item A close-up photograph of an engraved silver ring with the inscription "Forever Yours" delicately etched into its surface, set against a soft, blurred background of romantic, warm tones.
        \item A realistic photograph of a fast food drive-thru menu board at dusk, featuring a bold and colorful advertisement that reads "Try Our New Burger" with an appetizing image of the burger below, set against the backdrop of a busy suburban street.
        \item A realistic photograph of a wrist tattoo in cursive script reading "Fearless", with the skin slightly tanned and a subtle shadow under the text, set against a neutral background.
        \item A dark, decrepit haunted house with a menacing door knocker that reads "Abandon All Hope" in eerie, gothic lettering, set against a moonlit night with twisted, shadowy trees in the background.
        \item A detailed ski resort trail map with a prominent marker labeled "Black Diamond Run", set against a snowy backdrop with pine trees and skiers in the distance, capturing the thrill and challenge of the advanced slope.
        \item A medieval knight's castle with a grand drawbridge, the wooden sign above it boldly declaring "Trespassers Will Be Jousted", surrounded by a moat with water lilies and a cloudy sky.
        \item A vintage postcard with a faded, nostalgic look, featuring elegant cursive text that reads "Wish You Were Here" against a backdrop of a serene, old-world seaside town with pastel buildings and a gentle, sunny sky.
    \end{enumerate}
    \end{tcolorbox}
\end{figure*}

\begin{figure}[ht]
    \centering
    \includegraphics[width=0.9\textwidth]{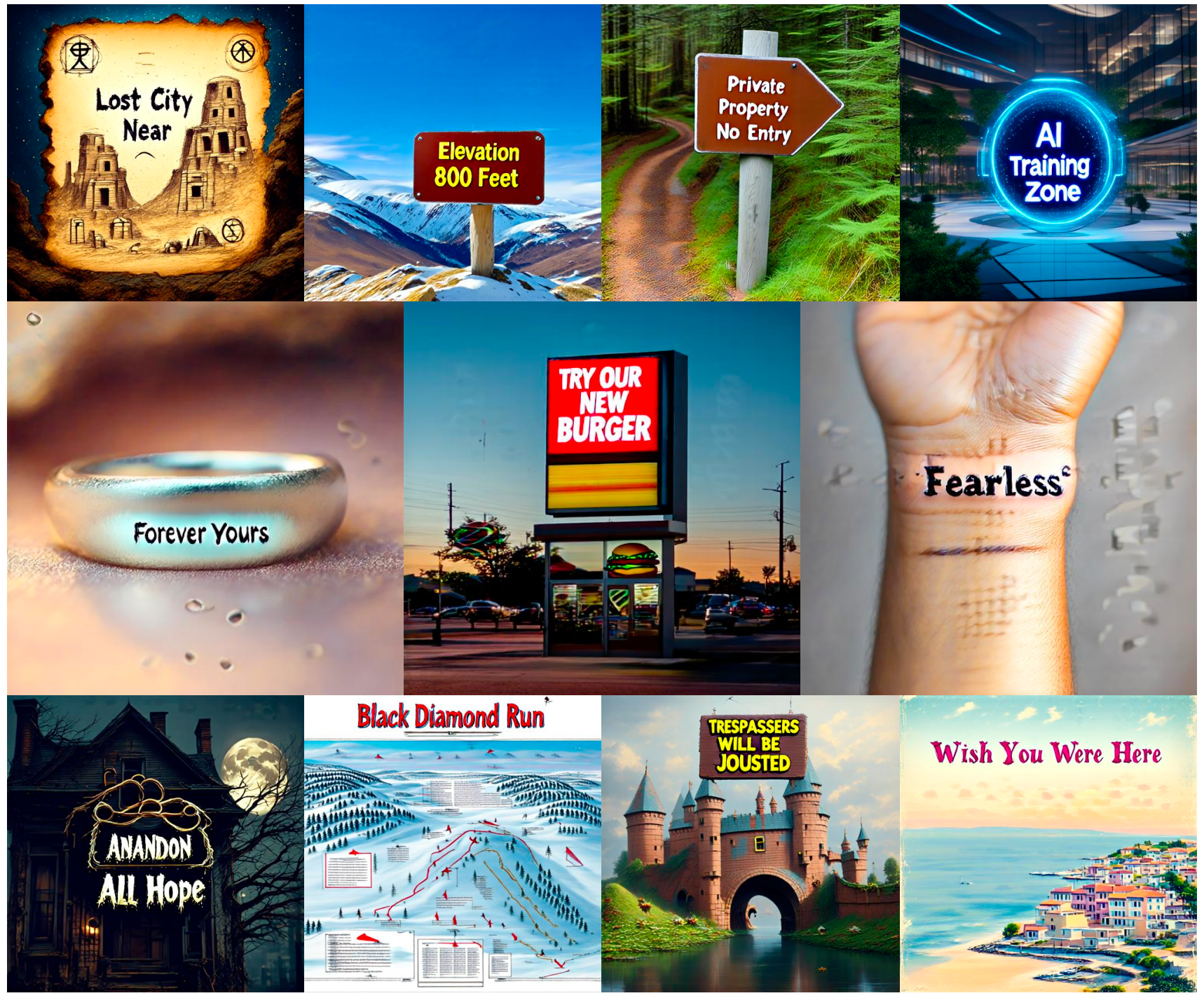}
    \caption{Qualitative results of FlowDPO ($\beta = 100$) on SD3.5-M.}
    \label{fig:real_image_qualitative_flowdpo}
\end{figure}

\begin{figure}[ht]
    \centering
    \includegraphics[width=0.9\textwidth]{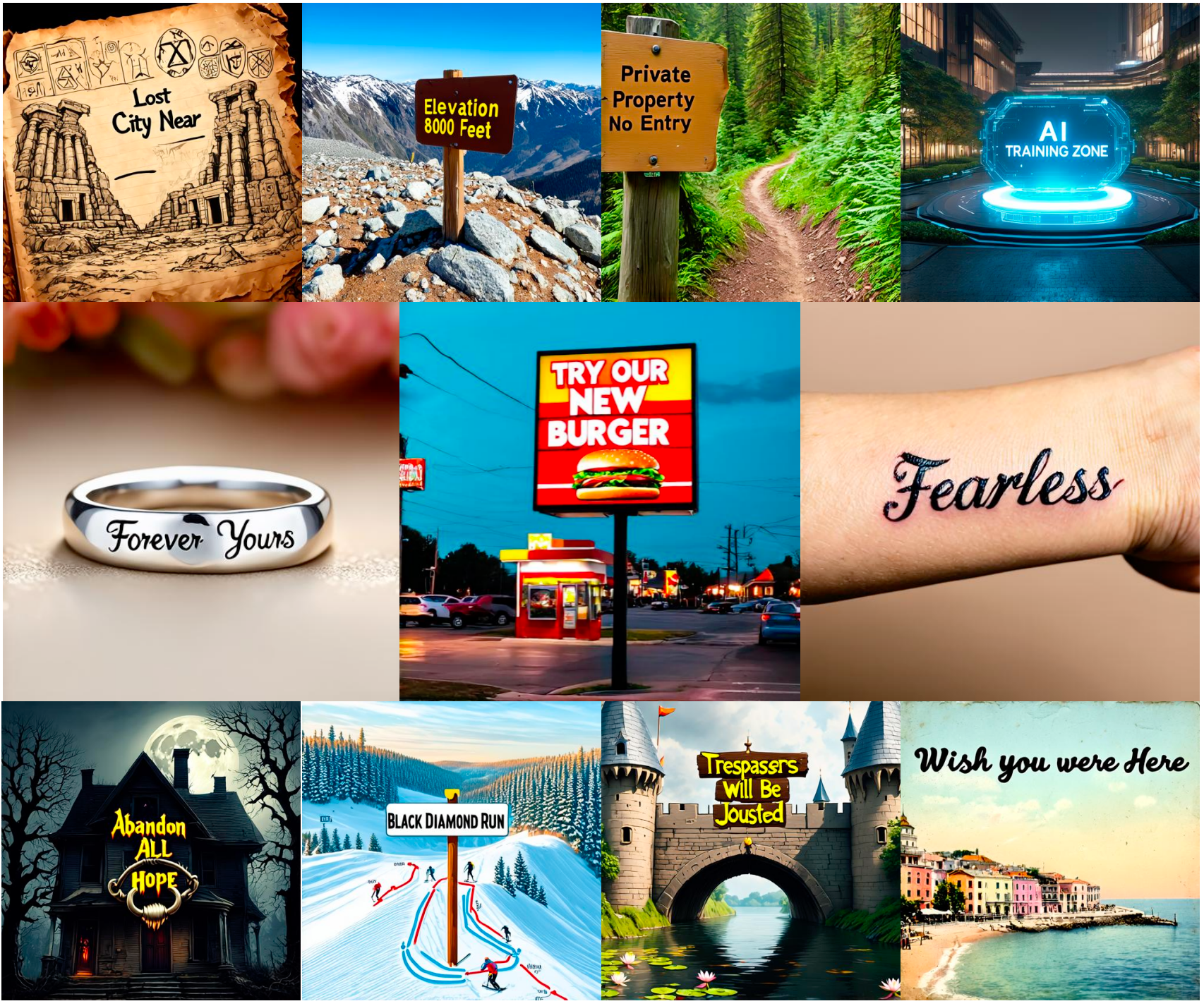}
    \caption{Qualitative results of \methodweighted ($\beta = 100$) on SD3.5-M.}
    \label{fig:real_image_qualitative_thermodpo}
\end{figure}

\begin{figure}[ht]
    \centering
    \includegraphics[width=0.9\textwidth]{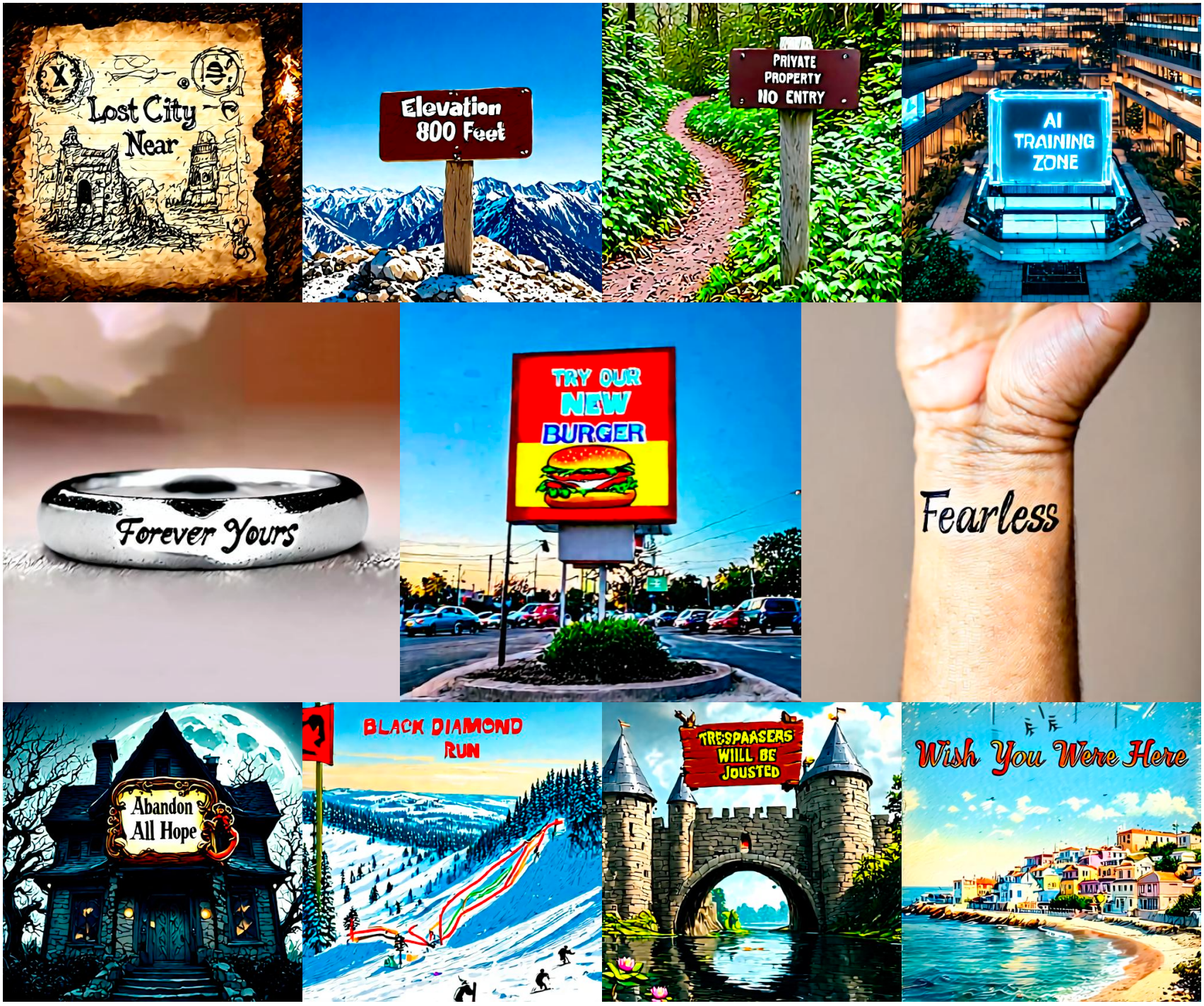}
    \caption{Qualitative results of FlowDPO + KL ($\beta = 100$) on SD3.5-M.}
    \label{fig:real_image_qualitative_flowdpo_kl}
\end{figure}

\begin{figure}[ht]
    \centering
    \includegraphics[width=0.9\textwidth]{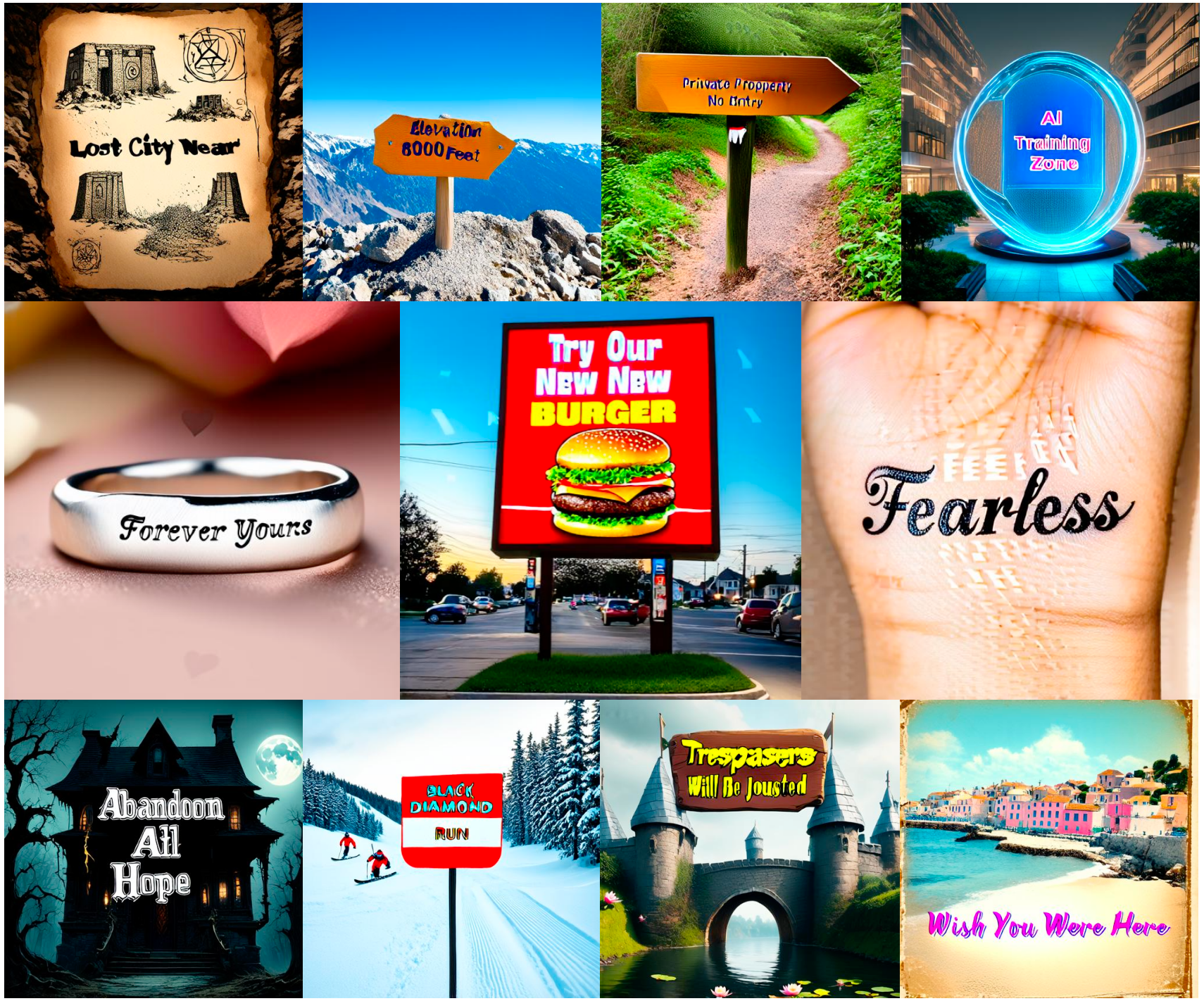}
    \caption{Qualitative results of FlowDPO + RFT ($\beta = 100$) on SD3.5-M.}
    \label{fig:real_image_qualitative_flowdpo_rft}
\end{figure}

\begin{figure}[ht]
    \centering
    \includegraphics[width=0.9\textwidth]{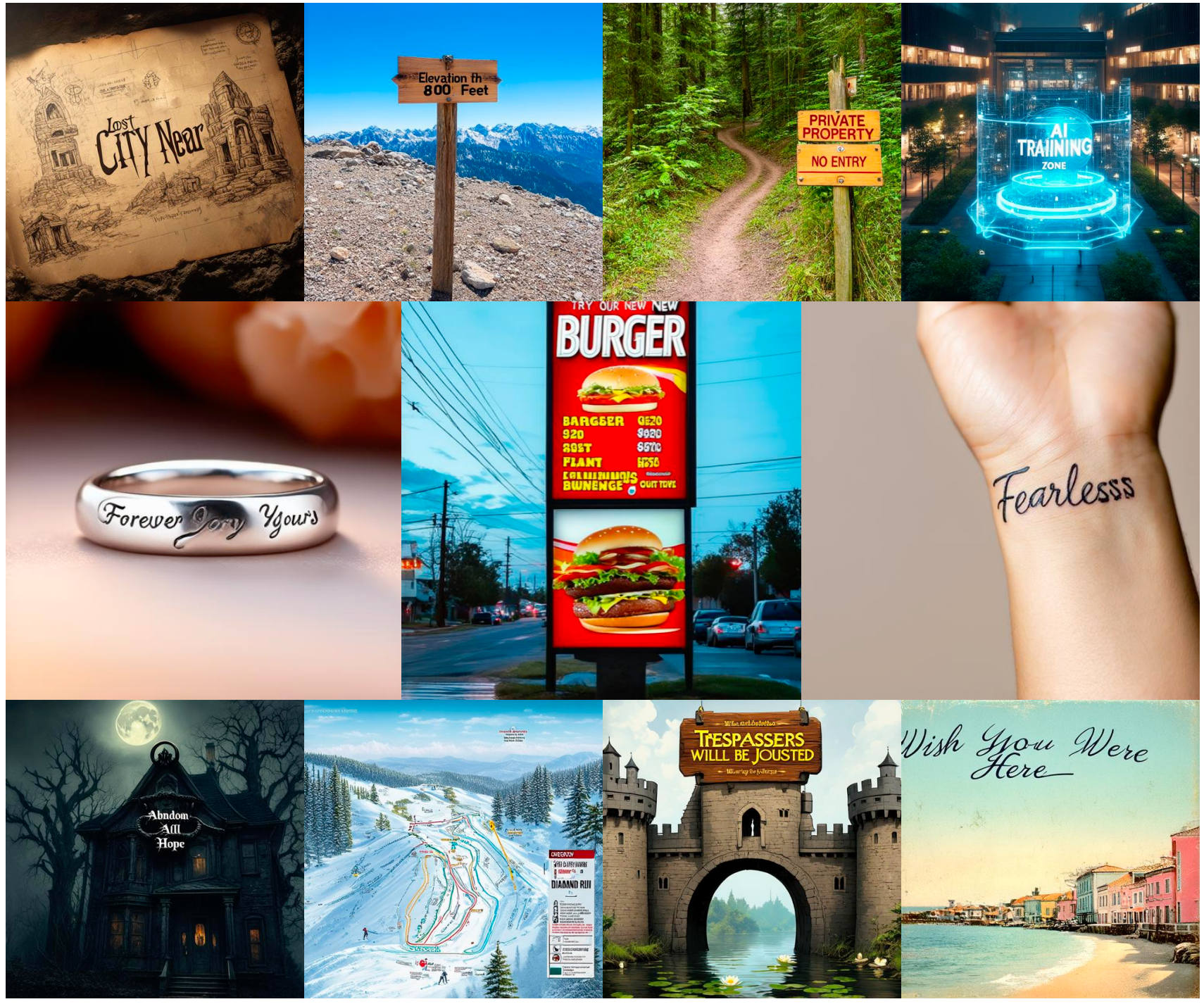}
    \caption{Qualitative results of Original Stable Diffusion 3.5-M.}
    \label{fig:real_image_qualitative_original}
\end{figure}

\begin{figure}[ht]
    \centering
    \includegraphics[width=0.9\textwidth]{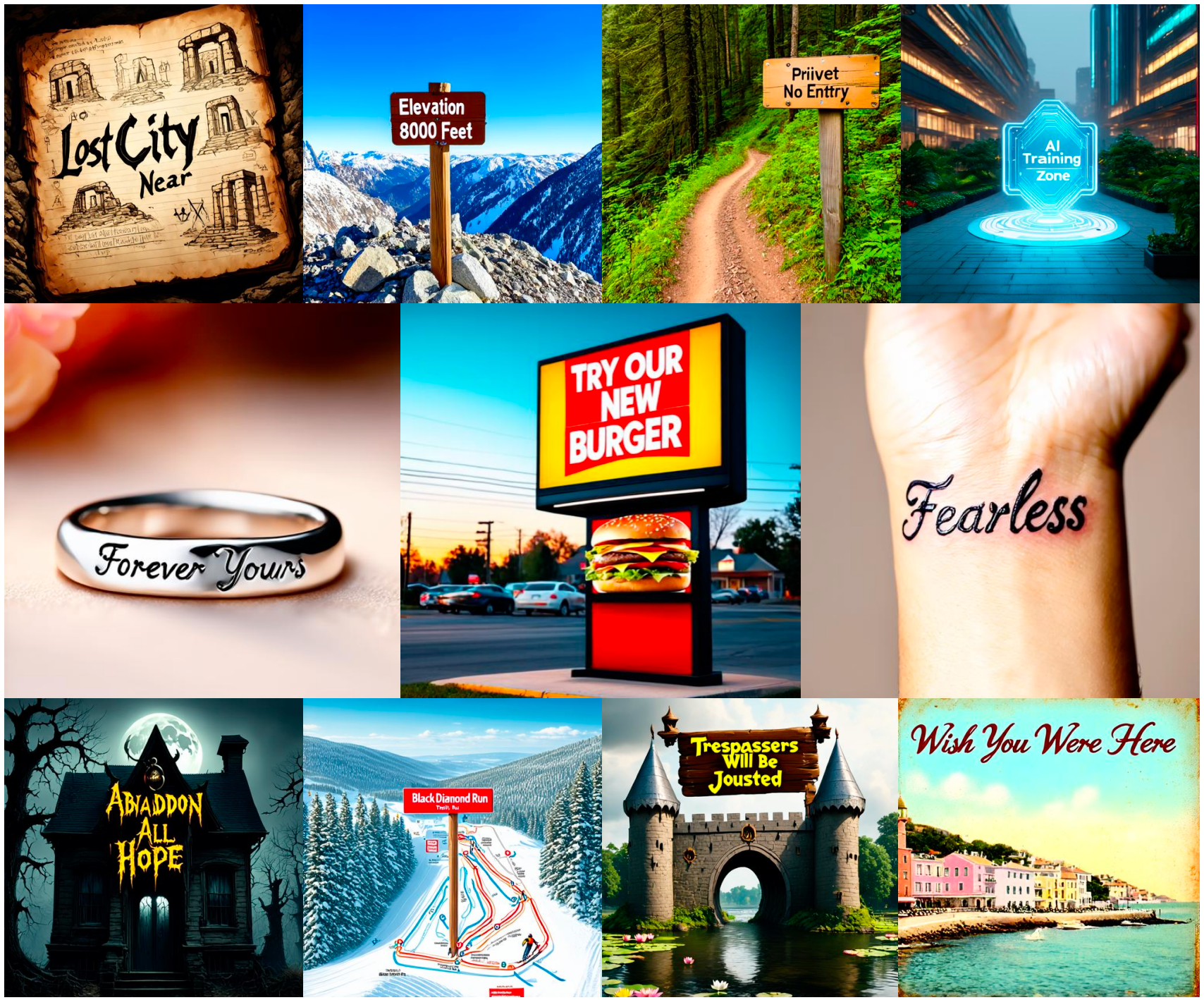}
    \caption{Qualitative results of RFT.}
    \label{fig:real_image_qualitative_rft}
\end{figure}

\subsection{Prompts of Fig.~\ref{fig:comparison}}
\label{app:prompts_figure_abs}
The prompts in Fig.~\ref{fig:comparison} are shown in \figref{fig:prompts_figure_abs}.

\begin{figure*}[ht]
    \centering
    \caption{Prompts in Fig.~\ref{fig:comparison} (from left to right from top to bottom).}
    \label{fig:prompts_figure_abs}
    
    \begin{tcolorbox}[
        width=0.96\textwidth,
        colback=white,
        colframe=black!70,
        boxrule=0.8pt,
        arc=3mm,
        left=6mm,
        right=6mm,
        top=4mm,
        bottom=4mm
    ]
    \small
    \begin{enumerate}[leftmargin=8mm,itemsep=0.35em,topsep=0pt]
        \item A weathered cave explorer's journal page, with the phrase "Lost City Near" prominently written in faded ink, surrounded by sketches of ancient ruins and cryptic symbols, under a dim, mystical light.
        \item A high-altitude mountain summit with a wooden signpost clearly marked "Elevation 8000 Feet", surrounded by rocky terrain and a backdrop of distant, snow-capped peaks under a clear blue sky.
        \item A hiking trail with a wooden signpost clearly displaying "Private Property No Entry", surrounded by dense, green foliage and a winding dirt path leading into the forest.
        \item A close-up photograph of an engraved silver ring with the inscription "Forever Yours" delicately etched into its surface, set against a soft, blurred background of romantic, warm tones.
        \item A realistic photograph of a fast food drive-thru menu board at dusk, featuring a bold and colorful advertisement that reads "Try Our New Burger" with an appetizing image of the burger below, set against the backdrop of a busy suburban street.
        \item A realistic photograph of a wrist tattoo in cursive script reading "Fearless", with the skin slightly tanned and a subtle shadow under the text, set against a neutral background.
    \end{enumerate}
    \end{tcolorbox}
\end{figure*}

\subsection{Real-Image Experimental Results on FLUX.2-klein-base-4B~\cite{flux-2-2025}}
\label{app:real_image_setup_flux}


\subsubsection{Experimental Results}
\label{app:real_image_setup_flux_results}
On FLUX.2-klein-base-4B at CFG $=3.5$, \tabref{tab:main_results_flux} shows \methodweighted improving GenEval ($0.7766$ vs. $0.7347$), UniReward ($0.6663$ vs. $0.6358$), and HPSv3.0 ($9.5420$ vs. $9.0675$) over the pretrained model while also improving OCR ($0.7241$ vs. $0.5475$). FlowDPO reaches higher OCR ($0.8349$) but lower values on the other three metrics. This second model reproduces the alignment--quality trade-off observed on SD3.5-M; it does not establish model-independent behavior.

\begin{table}[htbp]
    \centering
    \small
    \setlength{\tabcolsep}{4.5pt}
    \renewcommand{\arraystretch}{1.12}
    \caption{
        \textbf{Quantitative comparison on \texttt{FLUX.2-klein-base-4B}.}
        All compared methods are trained using the OCR preference pair dataset, while evaluation is conducted across GenEval, OCR, UniRwd and HPSv3.0.
        For each metric, we report the absolute score at both CFG settings.
        Best results are highlighted in \textbf{bold}.}
    \label{tab:main_results_flux}
    \resizebox{1\textwidth}{!}{
        \begin{tabular}{l c c c c c }
            \toprule
            \textbf{Model} & \textbf{CFG}
                           & \textbf{GenEval~\cite{ghosh2023geneval} $\uparrow$}
                           & \textbf{OCR $\uparrow$}
                           & \textbf{UniRwd~\cite{unifiedreward} $\uparrow$}
                           & \textbf{HPSv3.0~\cite{ma2025hpsv3widespectrumhumanpreference} $\uparrow$}                                         \\
            \midrule

            \multirow{2}{*}{FLUX.2-klein-base-4B (Baseline)}
                           & $1.0$
                           & \makecell[c]{0.2951}
                           & \makecell[c]{$\ \ $0.1360$\ \ \ $}
                           & \makecell[c]{0.4327}
                           & \makecell[c]{0.6506}                                              \\
                           & $4$
                           & \makecell[c]{0.7347}
                           & \makecell[c]{$\ \ $0.5475$\ \ \ $}
                           & \makecell[c]{0.6358}
                           & \makecell[c]{9.0675}                                                   \\
            \midrule

            \multirow{2}{*}{FlowDPO~\cite{liu2025improving} ($\beta=100$)}
                           & $1.0$
                           & \makecell[c]{0.5426}
                           & \makecell[c]{\textbf{0.8888}}
                           & \makecell[c]{0.5703}
                           & \makecell[c]{6.6977}                                               \\
                           & $4$
                           & \makecell[c]{0.6649}
                           & \makecell[c]{0.8349}
                           & \makecell[c]{0.6165}
                           & \makecell[c]{7.8769}                                             \\
            \cmidrule(l){2-6}

            \multirow{2}{*}{\methodweighted ($t$, $\beta=100$)}
                           & $1.0$
                           & \makecell[c]{0.6857}
                           & \makecell[c]{0.5162}
                           & \makecell[c]{0.6230}
                           & \makecell[c]{8.3607}                                           \\
                           & $4$
                           & \makecell[c]{\textbf{0.7766}}
                           & \makecell[c]{0.7241}
                           & \makecell[c]{\textbf{0.6663}}
                           & \makecell[c]{\textbf{9.5420}}                       \\
            \bottomrule
        \end{tabular}
    }
\end{table}

\end{document}